\documentclass[11pt]{article}
\usepackage[bookmarksnumbered = true, colorlinks = true, linktocpage = true, bookmarksopen = true,linkcolor = black, urlcolor = black]{hyperref}
\usepackage[intlimits]{amsmath}
\usepackage{subcaption}
\usepackage{amssymb}
\usepackage{comment}
\usepackage{amsthm}
\usepackage{rotating}
\usepackage{xcolor}
\usepackage[numbers,sort]{natbib}  
\usepackage{mathtools}
\usepackage[]{algorithm}
\usepackage{algorithmic}
\usepackage{tikz}
\usepackage{authblk}
\usepackage{booktabs}
\usetikzlibrary{fit,shapes.misc}
\usepackage[utf8]{inputenc}
\usepackage[margin =1in]{geometry}
\usepackage{url}
\usepackage[numbers,sort]{natbib}
\usepackage[normalem]{ulem}

\usepackage[titletoc]{appendix}
\usepackage[bottom,flushmargin,hang,multiple]{footmisc}
\usepackage{setspace}
\newtheorem{theorem}{Theorem}[section]
\newtheorem{lemma}[theorem]{Lemma}
\newtheorem{proposition}{Proposition}[section]

\newtheorem{assumption}[theorem]{Assumption}
\newtheorem{remark}[theorem]{Remark}

\usepackage{color, colortbl}
\definecolor{Lightgray}{rgb}{0.75,0.75,0.75}
\usepackage{indentfirst}

\renewenvironment{proof}[1][Proof]{\noindent\textbf{#1.} }{\hfill\qed\vspace{\baselineskip}}
\newcommand{\cS}{\mathcal{S}}
\newcommand{\A}{\mathcal{A}}
\newcommand{\Ptrans}{\mathsf{P}}        

\theoremstyle{plain}

\title{Implicit Q-Learning and SARSA:\\ Liberating Policy Control from Step-Size Calibration}
\author{Hwanwoo Kim{$^\dagger$}, Eric Laber}
\affil{Department of Statistical Science, Duke University}
\date{\today}

\begin{document}
\onehalfspacing

\maketitle
\def\thefootnote{$\dagger$}\footnotetext{hwanwoo.kim@duke.edu}
\begin{abstract}
Q-learning and SARSA are foundational reinforcement learning algorithms whose practical success depends critically on step-size calibration. Step-sizes that are too large can cause numerical instability, while step-sizes that are too small can lead to slow progress. We propose implicit variants of Q-learning and SARSA that reformulate their iterative updates as fixed-point equations. This yields an adaptive step-size adjustment that scales inversely with feature norms, providing automatic regularization without manual tuning. Our non-asymptotic analyses demonstrate that implicit methods maintain stability over significantly broader step-size ranges. Under favorable conditions, it permits arbitrarily large step-sizes while achieving comparable convergence rates. Empirical validation across benchmark environments spanning discrete and continuous state spaces shows that implicit Q-learning and SARSA exhibit substantially reduced sensitivity to step-size selection, achieving stable performance with step-sizes that would cause standard methods to fail.
\end{abstract}

\section{Introduction}
Reinforcement learning (RL) has emerged as a powerful framework for sequential decision-making, enabling agents to learn optimal behaviors through interaction with their environment \citep{sutton1998reinforcement}. A key component of many successful RL algorithms is the action-value function, which quantifies how much total reward an agent can expect to accumulate after taking a particular action from a given state. Among the most fundamental and widely studied action-value based approaches in RL are 
Q-learning \citep{watkins1989learning, watkins1992q} and State-Action-Reward-State-Action (SARSA) \citep{rummery1994line}, which have demonstrated remarkable success across diverse applications, including game playing \citep{mnih2015human}, robotics \citep{kober2013reinforcement}, and autonomous systems 
\citep{sallab2017deep}. A persistent challenge with such methods is their sensitivity to step-size selection, which can critically affect numerical stability and convergence. While tabular methods with decreasing step-sizes under a suitable behavioral policy enjoy well-established convergence properties \citep{watkins1992q, jaakkola1993convergence, tsitsiklis1994asynchronous,singh2000convergence}, the introduction of 
function approximation fundamentally alters the convergence behavior and amplifies step-size sensitivity. While function approximation introduces additional complications, i.e., Q-learning and SARSA with linear function approximation can diverge or fail to converge \citep{baird1995residual, gordon1996chattering}, the practical challenge of step-size selection remains critical. Even when convergence is theoretically guaranteed, existing results require stringent conditions on step-sizes 
\citep{zoufinite, zou2019finite, zhang2024constant}. Overly large step-sizes induce oscillations and potential 
divergence, while conservative step-sizes result in prohibitively slow learning. Such sensitivity makes 
it difficult to identify a step-size schedule that achieves both stability and fast convergence \citep{dabney2012adaptive, sutton1998reinforcement, szepesvari2022algorithms, chen2019performance, zou2019finite}. We seek a principled approach to mitigate step-size sensitivity without compromising computational efficiency or theoretical guarantees.

\subsection{Related Works}

\paragraph{Convergence Analysis.} For tabular settings, asymptotic convergence of Q-learning was established under appropriate step-size conditions \citep{watkins1992q, jaakkola1993convergence, tsitsiklis1994asynchronous}, and convergence of tabular SARSA was similarly established under suitable assumptions on the behavioral policy \citep{singh2000convergence}. However, with linear function approximation, the convergence properties become substantially more complex. \citet{baird1995residual} provided an illustrative counter-example demonstrating the divergence of Q-learning with linear function approximation. Recent work has shown that linear Q-learning with an $\epsilon$-softmax behavioral policy and adaptive temperature converges to a bounded set \citep{meyn2024projected}. For SARSA with linear function approximation, convergence depends critically on the policy improvement operator. \citet{gordon1996chattering, gordon2000reinforcement} demonstrated that SARSA with an $\epsilon$-greedy policy can exhibit chattering and fail to converge. \citet{melo2008analysis} showed that linear SARSA converges to a unique fixed point when the policy improvement operator is Lipschitz continuous with a sufficiently small constant. \citet{zou2019finite} later established finite-time error bounds for projected linear SARSA through novel techniques for handling time-inhomogeneous Markov chains.

\paragraph{Implicit Updates.} Implicit update schemes have emerged as a powerful approach to enhance the stability of stochastic approximation algorithms. Works on implicit stochastic gradient descent (SGD) \citep{toulis2015implicit, toulis2015scalable, toulis2017asymptotic, chee2023plus} demonstrated that evaluating a gradient at a future iterate rather than the current one provides adaptive shrinkage. In the context of temporal difference (TD) learning, \citet{kim2025stabilizing} formalized implicit TD methods and analyzed their asymptotic behavior. They also provided finite-time analyses demonstrating robustness to step-size selection and extended this framework to two-timescale methods by analyzing an implicit variant of TD learning with gradient correction. Furthermore, \citet{kim2025implicit} proposed an average-reward implicit TD method with finite-time error guarantees under relaxed step-size constraints. These advances demonstrate that implicit updates offer a principled mechanism for step-size adaptation, motivating their application to Q-learning and SARSA.

\subsection{Contributions}
\noindent Our main contributions are as follows.
\begin{itemize}
    \item \textbf{Implicit Q-learning and SARSA algorithms:} We propose implicit variants of Q-learning and SARSA that formulate algorithm updates as fixed-point equations. These admit closed-form iterative updates with adaptive step-sizes that scale inversely with feature norms, providing automatic step-size adjustment.

    \item \textbf{Finite-time analysis of Q-learning and implicit Q-learning:} We establish non-asymptotic error bounds for projected linear Q-learning and its implicit variant under Markovian data. We consider both constant and decreasing step-sizes of the form $\beta_t = \beta_0/t^s$ with $s \in (0,1)$. Our analysis reveals a sharp contrast between the two methods. Standard Q-learning exhibits critical sensitivity to step-size selection: large initial step-sizes cause instability and require careful calibration. In contrast, implicit Q-learning demonstrates remarkable robustness. Under suitable conditions, it admits virtually any step-size and achieves stable convergence without meticulous tuning.
    


    \item \textbf{Finite-time analysis of SARSA and implicit SARSA:} We establish non-asymptotic error bounds for projected linear SARSA with decreasing step-sizes $\beta_t = \beta_0/t^s$, $s \in (0,1)$, and for its implicit variant with both constant and decreasing step-sizes, under time-varying Lipschitz behavioral policies. Compared to \citet{zou2019finite}, our bounds for standard SARSA improve by a logarithmic factor. More significantly, implicit SARSA exhibits remarkable step-size robustness: under suitable conditions, it admits virtually any step-size, including constant step-sizes that cause standard SARSA to diverge, thereby eliminating the meticulous calibration required by linear SARSA.
    
    \item \textbf{Empirical validation:} We demonstrate the effectiveness of implicit Q-learning and SARSA across diverse benchmark environments spanning discrete and continuous state spaces. Empirical results confirm that implicit methods maintain stable performance across step-size ranges where standard methods fail, validating our theoretical guarantees.
\end{itemize}

\section{Background}\label{sec:background}
\paragraph{Markov Decision Processes.}
Consider a discounted MDP $(\cS,\A,\Ptrans,r,\gamma)$ composed of a continuous state space $\cS\subset\mathbb{R}^n$, a finite action space $\A$, a time-homogeneous transition kernel $\Ptrans: \cS \times \cS \times \A \to [0,1]$, a bounded reward $r:\cS\times\A\to[0, b]$ for some $b > 0$ and a discount factor $\gamma\in(0,1)$. Suppose we are given a stationary policy $\pi$, i.e., policy does not change over time. A policy $\pi$ is deterministic if it is a mapping from the state space $\cS$ to the action space $\A$, and stochastic if it is a mapping from the state space $\cS$ to the simplex $\Delta(\A)$, the space of probability distributions over $\A$. At each time step $t \ge 0$, an agent is at state $\boldsymbol{S}_t^{\pi}$ and takes an action $A_t^{\pi}$ according to $\pi$. The agent then receives a reward $R_t^{\pi} := r(\boldsymbol{S}_t^{\pi},A_t^{\pi})$ and transitions to the next state $\boldsymbol{S}_{t+1}^{\pi}$ with a probability $\Ptrans(\boldsymbol{S}_t^{\pi}, \boldsymbol{S}_{t+1}^{\pi}, A_t^{\pi})$.

Given a policy $\pi$, the corresponding value function is defined as 
\[
V^{\pi}(\boldsymbol{s})=\mathbb{E}^{\pi}\left(\sum_{t=0}^\infty\gamma^t R_t^{\pi}~\middle|~\boldsymbol{S}_0^{\pi}=\boldsymbol{s}\right)
\]
and the corresponding action-value function is defined as
\[
Q^{\pi}(\boldsymbol{s},a)=\mathbb{E}^{\pi}\left(\sum_{t=0}^\infty\gamma^t R_t^{\pi}~\middle|~\boldsymbol{S}_0^{\pi}=\boldsymbol{s}, A^{\pi}_0 = a\right) = r(\boldsymbol{s},a)+\gamma\int_{\cS} V^{\pi}(\boldsymbol{y})\Ptrans(\mathrm{d}\boldsymbol{y}\mid \boldsymbol{s},a).
\]
Note that the expectation $\mathbb{E}^{\pi}$ is over trajectories $\{\boldsymbol{S}_t^{\pi}, A_t^{\pi}, R_t^{\pi}\}_{t \ge 0}$ in which $\pi$ is followed at each time step. The optimal value function is defined as $V^{\star}(\boldsymbol{s}) := \sup_{\pi} V^\pi(\boldsymbol{s})$ for all $\boldsymbol{s} \in \cS$ and the optimal action-value function is defined as $Q^*(\boldsymbol{s},a) :=\sup_{\pi}Q^{\pi}(\boldsymbol{s},a)$ for all $(\boldsymbol{s}, a) \in \cS \times \A$. An optimal (deterministic) policy $\pi^{\star}$ can be obtained from the optimal action-value function $Q^{\star}$ by selecting $ \pi^{\star}(\boldsymbol{s}) \in \arg\max_{a \in \A} Q^{\star}(\boldsymbol{s}, a)$ \citep{puterman2014markov}. This fundamental relationship has served as a backbone of model-free policy learning \citep{sutton1998reinforcement, clifton2020q}.

\paragraph{Linear Function Approximation.}
When the state space $\cS$ and the action space $\A$ are finite and relatively low-dimensional, storing the action-value function $Q^\pi$ in tabular form for all possible state-action pairs is manageable. However, when either space becomes high-dimensional or continuous, this tabular approach is no longer feasible. A computationally tractable alternative is to approximate the action-value function using a basis function representation. To this end, we introduce feature basis functions $\{\phi_i: \cS\times\A\to\mathbb{R}\}_{i=1}^M$ satisfying $\|\boldsymbol{\phi}(\boldsymbol{s},a)\|_2\leq 1$ for all $(\boldsymbol{s},a) \in \cS \times \A$. We then approximate the action-value function $Q^\pi$ via a linear combination of these basis functions: 
\begin{align*} &Q^\pi(\boldsymbol{s},a) \approx \boldsymbol{\phi}(\boldsymbol{s},a)^{\top} \boldsymbol{\theta}= \sum_{i=1}^M \phi_i(\boldsymbol{s},a)\theta_i~, \quad \text{where} \\
&\boldsymbol{\phi}(\boldsymbol{s},a) := [\phi_1(\boldsymbol{s},a), \ldots, \phi_M(\boldsymbol{s},a)]^\top\in\mathbb{R}^{M}, \quad \boldsymbol{\theta} := [\theta_1, \ldots, \theta_M]^\top\in\mathbb{R}^{M}. \end{align*} 
Here, $\boldsymbol{\theta}\in \mathbb{R}^M$ is the weight parameter vector characterizing the linear function approximation. The problem of learning an optimal action-value function (and thus an optimal policy) therefore reduces to learning an optimal parameter vector $\boldsymbol{\theta}^{\star}\in \mathbb{R}^M$.

\paragraph{Q-learning.}
Given a behavioral policy $\pi$ used to generate data $\{\boldsymbol{S}_t^{\pi}, A_t^{\pi}\}_{t \in \mathbb{N}}$, Q-learning is an iterative algorithm that approximates the optimal action-value function $Q^{\star}$ from data. In the tabular setting, where both state space $\mathcal{S}$ and action space $\mathcal{A}$ are finite, Q-learning maintains a look-up table $\widehat{Q}_t \in \mathbb{R}^{|\cS| \times |\A|}$ which stores the estimated action-value for all state-action pairs. Given initial state $\boldsymbol{S}^\pi_0$ and initial estimate $\widehat{Q}_0$, at iteration $t \in \mathbb{N}\cup\{0\}$, the update rule yields $(t+1)^{\text{th}}$ estimate given by 
\begin{align*} 
&\widehat{Q}_{t+1}(\boldsymbol{S}_t^{\pi},A_t^{\pi}) = \widehat{Q}_t(\boldsymbol{S}_t^{\pi},A_t^{\pi}) + \beta_t \delta_{t}^Q\\
&\delta_{t}^Q:= R_t^{\pi} + \gamma \max_{a \in \mathcal{A}} \widehat{Q}_t(\boldsymbol{S}_{t+1}^{\pi},a) - \widehat{Q}_t(\boldsymbol{S}_t^{\pi},A_t^{\pi})
\end{align*} 
where $(\beta_t)_{t \ge 0}$ is a sequence of step-sizes. Since the update rule does not rely on the action $A_{t+1}^\pi$ obtained from the behavioral policy $\pi$, Q-learning is classified as an off-policy learning algorithm. The tabular setting can be viewed as a special case of linear function approximation with one-hot encoding basis functions $\{\boldsymbol{e}{(\boldsymbol{s},a)}\}_{\boldsymbol{s} \in \mathcal{S}, a \in \mathcal{A}}$, where $\boldsymbol{e}{(\boldsymbol{s},a)} \in \mathbb{R}^{|\mathcal{S}| \times |\mathcal{A}|}$ has a single entry equal to one at the position corresponding to state-action pair $(\boldsymbol{s},a)$ and zeros elsewhere. With the representation $\widehat{Q}_t(\boldsymbol{s},a) = \boldsymbol{e}{(\boldsymbol{s},a)}^\top\widehat{\boldsymbol{\theta}}_{t}$ for all $(\boldsymbol{s},a) \in \mathcal{S} \times \mathcal{A}$ and $\boldsymbol{e}^\pi_t:=\boldsymbol{e}{(\boldsymbol{S}_t^{\pi},A_t^{\pi})}$, 
the tabular Q-learning update can be equivalently expressed as 
\begin{align*} 
\widehat{\boldsymbol{\theta}}_{t+1} = \widehat{\boldsymbol{\theta}}_{t} + \beta_t \left\{R_t^{\pi} + \gamma \max_{a \in \mathcal{A}} \boldsymbol{e}{(\boldsymbol{S}_{t+1}^{\pi},a)}^\top\widehat{\boldsymbol{\theta}}_{t} -{\boldsymbol{e}^\pi_t}^\top\widehat{\boldsymbol{\theta}}_{t}\right\}\boldsymbol{e}^\pi_t. 
\end{align*} 
Generalizing to a broader class of basis functions $\boldsymbol{\phi}(\boldsymbol{s},a) = [\phi_1(\boldsymbol{s},a), \ldots, \phi_M(\boldsymbol{s},a)]^\top\in\mathbb{R}^{M}$, with $\boldsymbol{\phi}^\pi_t := \boldsymbol{\phi}(\boldsymbol{S}_t^{\pi},A_t^{\pi})$, 
the linear Q-learning update rule at each time step $t \in \mathbb{N}$ is given by
\begin{align*} 
\widehat{\boldsymbol{\theta}}_{t+1} = \widehat{\boldsymbol{\theta}}_{t} + \beta_t \left\{R_t^{\pi} + \gamma \max_{a \in \mathcal{A}}\boldsymbol{\phi}(\boldsymbol{S}_{t+1}^{\pi},a)^\top\widehat{\boldsymbol{\theta}}_{t} - {\boldsymbol{\phi}^\pi_t}^\top\widehat{\boldsymbol{\theta}}_{t} \right\}\boldsymbol{\phi}^\pi_t.
\end{align*}

\paragraph{State-Action-Reward-State-Action (SARSA).} SARSA is an on-policy algorithm that learns an optimal policy through data obtained from a sequence of evolving behavioral policies. Similar to Q-learning, in the tabular setting, SARSA maintains a look-up table $\widehat{Q}_t \in \mathbb{R}^{|\mathcal{S}| \times |\mathcal{A}|}$ which stores the estimated action-value function for all state-action pairs. Given initial state $\boldsymbol{S}_0$, initial estimate $\widehat{Q}_0$ and initial policy $\pi_0 := \mathcal{I}(\widehat{Q}_0)$ where $\mathcal{I}$ denotes a policy improvement operator (e.g., $\epsilon$-greedy, softmax), SARSA proceeds as follows. At iteration $t \in \mathbb{N}\cup\{0\}$, the algorithm proceeds as follows. Starting from state $\boldsymbol{S}^{\pi_{t-2}}_t$ (generated under policy $\pi_{t-2}$), it selects action $A^{\pi_{t-1}}_t$ according to policy $\pi_{t-1}$. The algorithm then observes reward $R_t = r(\boldsymbol{S}^{\pi_{t-2}}_t, A^{\pi_{t-1}}_t)$ and transitions to the next state $\boldsymbol{S}^{\pi_{t-1}}_{t+1}$. Finally, it updates the policy to $\pi_{t} = \mathcal{I}(\widehat{Q}_{t})$ and selects the next action $A^{\pi_t}_{t+1}$ under this new policy. By convention, we use $\pi_{-1} := \pi_0$ and $\boldsymbol{S}^{\pi_{-2}}_0:= \boldsymbol{S}_0$. The Q-function estimate is updated via
\begin{align*}
&\widehat{Q}_{t+1}(\boldsymbol{S}_t^{\pi_{t-2}},A_t^{\pi_{t-1}}) = \widehat{Q}_t(\boldsymbol{S}_t^{\pi_{t-2}},A_t^{\pi_{t-1}}) + \beta_t\delta_t^{S}, \\
&\delta_t^{S}:= R_t + \gamma \widehat{Q}_t(\boldsymbol{S}_{t+1}^{\pi_{t-1}},A_{t+1}^{\pi_{t}}) - \widehat{Q}_t(\boldsymbol{S}_t^{\pi_{t-2}},A_t^{\pi_{t-1}})
\end{align*} 
where $(\beta_t)_{t \ge 0}$ is a sequence of step-sizes. Crucially, the data $(\boldsymbol{S}_t^{\pi_{t-2}}, A_t^{\pi_{t-1}}, R_t, \boldsymbol{S}_{t+1}^{\pi_{t-1}}, A_{t+1}^{\pi_t})$ are generated by the evolving policy sequence, which is iteratively updated. The key difference from Q-learning is that the update uses the action $A_{t+1}^{\pi_{t}}$ selected under the behavioral policy $\pi_{t}$ rather than the greedy action $\arg\max_{a \in \mathcal{A}} \widehat{Q}_t(\boldsymbol{S}_{t+1}^{\pi_{t-1}},a)$. This makes SARSA an on-policy algorithm, as the Q-function estimate corresponds to the policy currently generating the data.

\noindent As with Q-learning, the tabular setting can be viewed as a special case of linear function approximation with one-hot encoding basis functions $\{\boldsymbol{e}{(\boldsymbol{s},a)} \in \mathbb{R}^{|\cS|\times |\A|}\}_{\boldsymbol{s} \in \mathcal{S}, a \in \mathcal{A}}$. Let $\boldsymbol{e}_t:=\boldsymbol{e}{(\boldsymbol{S}_t^{\pi_{t-2}},A_t^{\pi_{t-1}})}$, then tabular SARSA update admits the following update rule
\begin{align*}
\widehat{\boldsymbol{\theta}}_{t+1} = \widehat{\boldsymbol{\theta}}_{t} + \beta_t \left(R_t + \gamma \boldsymbol{e}_{t+1}^\top\widehat{\boldsymbol{\theta}}_{t} - \boldsymbol{e}_t^\top\widehat{\boldsymbol{\theta}}_{t}\right)\boldsymbol{e}_t .
\end{align*}
Generalizing to a broader class of basis functions $\boldsymbol{\phi}(\boldsymbol{s},a) = [\phi_1(\boldsymbol{s},a), \ldots, \phi_M(\boldsymbol{s},a)]^\top\in\mathbb{R}^{M}$ with $\boldsymbol{\phi}_t := \boldsymbol{\phi}(\boldsymbol{S}_t^{\pi_{t-2}},A_t^{\pi_{t-1}})$, 
the linear SARSA update rule at each time step $t \in \mathbb{N}$ is given by
\begin{align*}
    \widehat{\boldsymbol{\theta}}_{t+1} =  \widehat{\boldsymbol{\theta}}_{t} + \beta_t  \left(R_t + \gamma \boldsymbol{\phi}_{t+1}^\top\widehat{\boldsymbol{\theta}}_{t} - \boldsymbol{\phi}_t^\top\widehat{\boldsymbol{\theta}}_{t}\right)\boldsymbol{\phi}_t
\end{align*}
where the data $(\boldsymbol{S}_t^{\pi_{t-2}}, A_t^{\pi_{t-1}}, R_t, \boldsymbol{S}_{t+1}^{\pi_{t-1}}, A_{t+1}^{\pi_t})$ are generated by a sequence of time-varying policies $\pi_{k} := \mathcal{I}\{\boldsymbol{\phi}(\boldsymbol{s},a)^{\top} \widehat{\boldsymbol{\theta}}_{k}\}$ for $k \in \{t-2, t-1, t\}$. After each update, the policy evolves to $\pi_{t+1} = \mathcal{I}\{\boldsymbol{\phi}(\boldsymbol{s},a)^{\top} \widehat{\boldsymbol{\theta}}_{t+1}\}$.

\paragraph{Projection.} In practice, Q-learning and SARSA are often implemented with a projection step to ensure numerical stability and theoretical tractability \citep{bhandari2018finite, zou2019finite, zhang2023convergence}. We incorporate such a projection to ensure the iterates $(\widehat{\boldsymbol{\theta}}_n)_{n\in\mathbb{N}}$ remain in a Euclidean ball of radius $r > 0$. Specifically, the projected Q-learning update becomes
\begin{align*}
\widehat{\boldsymbol{\theta}}_{t+1} = \Pi_r\left[\widehat{\boldsymbol{\theta}}_{t} + \beta_t \left\{R_t^{\pi} + \gamma \max_{a \in \mathcal{A}}\boldsymbol{\phi}(\boldsymbol{S}_{t+1}^{\pi},a)^\top\widehat{\boldsymbol{\theta}}_{t} - \boldsymbol{\phi}_t^\top\widehat{\boldsymbol{\theta}}_{t} \right\}\boldsymbol{\phi}_t\right]
\end{align*}
where the projection operator is defined as
\begin{align*}\label{EUCLIDEAN_PROJ}
\Pi_r(\boldsymbol{\theta}) := \begin{cases}
r\boldsymbol{\theta}/\|\boldsymbol{\theta}\|_2 & \text{if } \|\boldsymbol{\theta}\|_2 > r \\
\boldsymbol{\theta} & \text{otherwise}.
\end{cases}
\end{align*}
Similarly, the projected linear SARSA updates according to
\begin{align*}
\widehat{\boldsymbol{\theta}}_{t+1} = \Pi_r\left[\widehat{\boldsymbol{\theta}}_{t} + \beta_t  \left\{R_t + \gamma \boldsymbol{\phi}_{t+1}^\top\widehat{\boldsymbol{\theta}}_{t} - \boldsymbol{\phi}_t^\top\widehat{\boldsymbol{\theta}}_{t}\right\}\boldsymbol{\phi}_t\right].
\end{align*}
Projection has been widely adopted in both theoretical analyses and practical implementations of reinforcement learning and optimization algorithms, as it provides a computationally efficient mechanism to ensure boundedness while preserving essential convergence properties \citep{nemirovski2009robust, bubeck2015convex, bhandari2018finite, zou2019finite, xu2019two, xu2020finite, wu2020finite, kim2025implicit, kim2025stabilizing}. In this work, we focus on these projected variants of linear Q-learning and SARSA.

\paragraph{Sensitivity to step-size.}
Step-size selection remains a fundamental challenge in reinforcement learning, with its importance well-documented in the literature \citep{sutton1998reinforcement,bertsekas2004improved,szepesvari2022algorithms,powell2022reinforcement}. Figure \ref{fig:Q_SARSA_MOTIVATION} illustrates this sensitivity using two standard benchmark environments. Both Q-learning and SARSA maintain stable performance below certain thresholds, approximately $\beta \approx 1.0$ for Cliff Walking and $\beta \approx 0.8$ for Taxi, beyond which performance degrades sharply. However, conservative step-size choices incur substantial learning inefficiency. The bottom panels demonstrate that at fixed iteration budgets, algorithms employing small step-sizes exhibit markedly slower convergence, achieving significantly inferior cumulative rewards compared to larger step-sizes. This presents a fundamental trade-off: conservative step-sizes guarantee algorithmic stability at the expense of learning efficiency, while aggressive step-sizes enable rapid progress but introduce the risk of divergence. Consequently, practitioners must engage in careful hyperparameter tuning, a process frequently neglected in practice \citep{powell2022reinforcement}. We address this limitation by proposing implicit variants of Q-learning and SARSA that exhibit substantially improved robustness to step-size selection, permitting larger step-sizes without performance degradation while maintaining convergence guarantees.

\begin{figure}[htp]
\centering
\includegraphics[height=.295\textwidth]{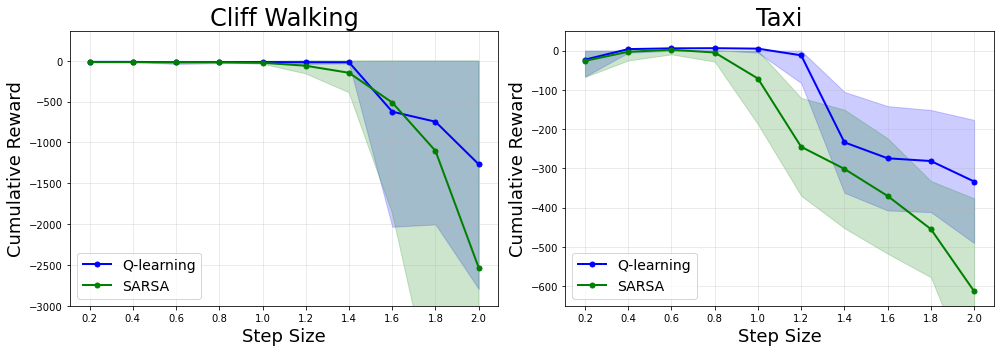}
\includegraphics[height=.295\textwidth]{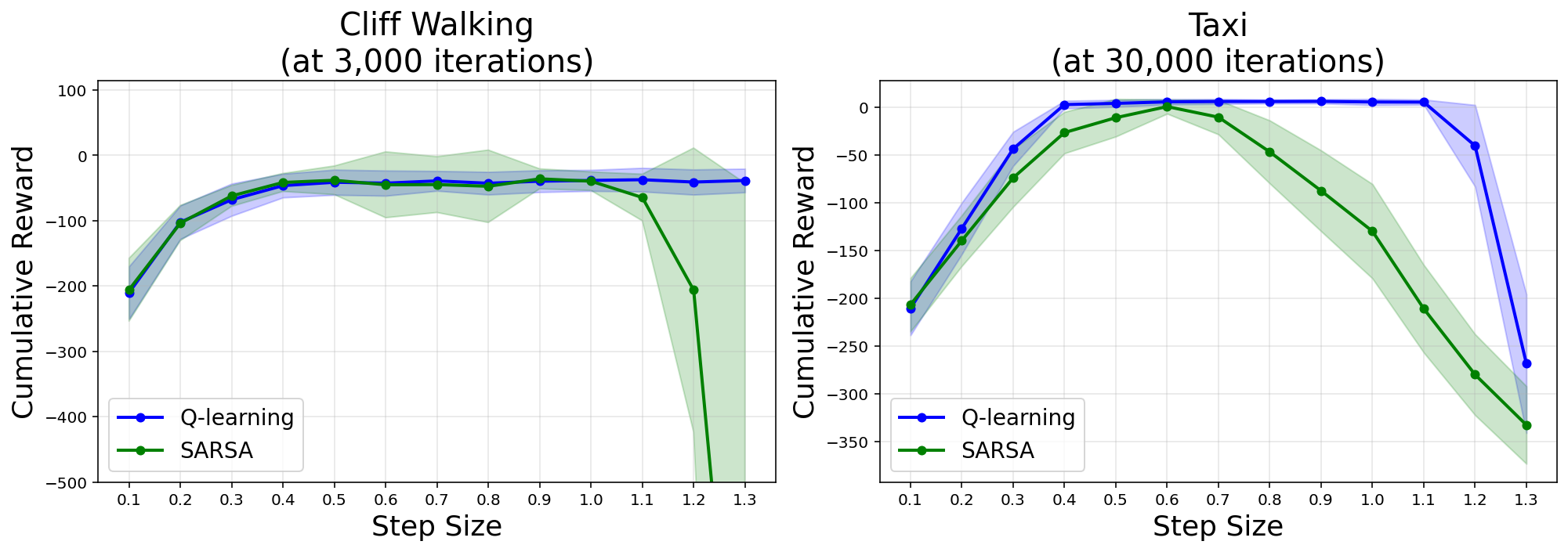}
\caption{Sensitivity of Q-learning and SARSA to step-size selection on Cliff Walking and Taxi environments. Top: Average cumulative reward versus step-size $\beta$ shows both algorithms maintain stability below certain thresholds ($\beta \approx 1.0$ for Cliff Walking, $\beta \approx 0.8$ for Taxi) but degrade sharply beyond them. Bottom: Cumulative reward at fixed iterations (3,000 for Cliff Walking, 30,000 for Taxi) reveals that conservative step-sizes, while stable, yield substantially inferior performance. Shaded regions represent standard errors over 50 runs.}
\label{fig:Q_SARSA_MOTIVATION}
\end{figure}

\section{Policy Control via Implicit Procedures}\label{sec:method}
\noindent Implicit update schemes serve as a powerful approach to enhance the numerical stability of stochastic approximation algorithms. The core principle involves reformulating the standard recursive update as a fixed-point equation where the future iterate appears on both the left- and right-hand sides of the algorithm update rule. In the context of SGD with an objective function $f$ that one wishes to minimize, the implicit SGD iterates admit the following update rule 
$$
\boldsymbol{\theta}_{t+1}^{\text{im}} := \boldsymbol{\theta}_t^{\text{im}} + \beta_t \nabla f(\boldsymbol{\theta}_{t+1}^{\text{im}}; \xi_t)
$$ 
for a step-size $\beta_t$ with a stochastic component $\xi_t$. It evaluates the gradient at the future iterate rather than the current one, as proposed and analyzed by \citep{toulis2015implicit, toulis2016towards, Toulis2014-mq, toulis2017asymptotic, chee2023plus}. This seemingly subtle modification yields significant stability improvements by providing adaptive shrinkage that automatically adjusts the effective step-size based on gradient information. Similarly, implicit TD methods reformulate standard TD updates by evaluating the current state value at the future iterate, resulting in adaptive step-sizes that scale inversely with feature norms \citep{kim2025stabilizing, kim2025implicit}. These implicit procedures have demonstrated substantial practical benefits in mitigating numerical instabilities inherent in TD learning with function approximation, motivating their adoption in our formulation of Q-learning and SARSA.

\paragraph{Implicit Q-learning.}
Inspired by developments in the aforementioned implicit procedures, we propose an implicit variant of Q-learning that enhances numerical stability. Recall, in the standard linear Q-learning, $t^{\text{th}}$ update direction is
$$
\left\{R_t^{\pi} + \gamma \max_{a' \in \mathcal{A}}\boldsymbol{\phi}(\boldsymbol{S}_{t+1}^{\pi},a')^{\top}\widehat{\boldsymbol{\theta}}_{t} - \boldsymbol{\phi}_t^{\top}\widehat{\boldsymbol{\theta}}_{t}\right\}\boldsymbol{\phi}_t.
$$
An implicit version instead uses the future iterate $\widehat{\boldsymbol{\theta}}_{t+1}$ in place of $\widehat{\boldsymbol{\theta}}_{t}$ when computing the current state-action value, i.e., 
$$
\left\{R_t^{\pi} + \gamma \max_{a' \in \mathcal{A}}\boldsymbol{\phi}(\boldsymbol{S}_{t+1}^{\pi},a')^{\top}\widehat{\boldsymbol{\theta}}_{t} - \boldsymbol{\phi}_t^{\top}{\color{red}\widehat{\boldsymbol{\theta}}_{t+1}}\right\}\boldsymbol{\phi}_t.
$$
This reformulation transforms the recursive update into a fixed-point equation, where the future iterate appears on both sides of the update rule. Combining terms involving future iterate from both sides of the Q-learning update, the implicit update can be written as
\begin{align*}
\left(I + \beta_t\boldsymbol{\phi}_t\boldsymbol{\phi}_t^{\top}\right)\widehat{\boldsymbol{\theta}}_{t+1} = \widehat{\boldsymbol{\theta}}_{t} + \beta_t\left\{R_t^{\pi} + \gamma\max_{a' \in \mathcal{A}}\boldsymbol{\phi}(\boldsymbol{S}_{t+1}^{\pi},a')^{\top}\widehat{\boldsymbol{\theta}}_{t}\right\}\boldsymbol{\phi}_t.
\end{align*}
Applying the Sherman-Morrison-Woodbury formula yields an explicit update with an adaptive step-size that scales inversely with the squared feature norm, providing automatic shrinkage and improved stability over standard Q-learning, as formally stated in the proposition below.

\begin{proposition}[Implicit Q-learning update]
\label{prop:implicit_q_learning}
The implicit linear Q-learning algorithm admits the following update rule:
\begin{align*}
\widehat{\boldsymbol{\theta}}_{t+1} = \widehat{\boldsymbol{\theta}}_{t} + \frac{\beta_t\left\{R_t^{\pi} + \gamma \max_{a' \in \mathcal{A}} \boldsymbol{\phi}(\boldsymbol{S}_{t+1}^{\pi},a')^\top\widehat{\boldsymbol{\theta}}_{t} - \boldsymbol{\phi}_t^\top\widehat{\boldsymbol{\theta}}_{t}\right\} }{1 + \beta_t \|\boldsymbol{\phi}_t\|_2^2} \boldsymbol{\phi}_t.
\end{align*}
\end{proposition}

\begin{algorithm}[htp]
\caption{(Projected) Implicit Q-learning}
\begin{algorithmic}[1]\label{ALG:Q}
\STATE \textbf{Input:} step-sizes $\{\beta_t\}_{t \in \mathbb{N}}$, initial estimate $\widehat{\boldsymbol{\theta}}_0$, initial state $\boldsymbol{S}_0^{\pi}$, basis functions $\{\phi_i\}_{i=1}^M$, behavioral policy $\pi$, projection radius $r>0$ 
\STATE \textbf{Notation:} $\boldsymbol{\phi}_t \equiv \boldsymbol{\phi}(\boldsymbol{S}_t^{\pi}, A_t^{\pi})$; ~$\widetilde{\beta}_t \equiv \beta_t/(1 + \beta_t\|\boldsymbol{\phi}_t\|_2^2)$
\FOR{$t=0,1,2,\dots$}
  \STATE Take $A_t^{\pi} \sim \pi(\cdot|\boldsymbol{S}_t^{\pi})$, observe $R_t^{\pi} = r(\boldsymbol{S}_t^{\pi}, A_t^{\pi})$
  \STATE Transition to $\boldsymbol{S}_{t+1}^{\pi} \sim \Ptrans(\cdot|\boldsymbol{S}_t^{\pi}, A_t^{\pi})$
  \STATE $\delta^Q_t \gets R_t^{\pi} + \gamma\max_{a' \in \mathcal{A}}\boldsymbol{\phi}(\boldsymbol{S}_{t+1}^{\pi}, a')^{\top}\widehat{\boldsymbol{\theta}}_t - \boldsymbol{\phi}_t^{\top}\widehat{\boldsymbol{\theta}}_t$
  \STATE
    $\widehat{\boldsymbol{\theta}}_{t+1} \gets \widehat{\boldsymbol{\theta}}_t + \widetilde{\beta}_t \delta^Q_t\boldsymbol{\phi}_t$
  \STATE (If projection applied):
  $\widehat{\boldsymbol{\theta}}_{t+1} \gets \Pi_r \widehat{\boldsymbol{\theta}}_{t+1}$
\ENDFOR
\end{algorithmic}
\end{algorithm}

\paragraph{Implicit SARSA.}
Following the same principle of implicit updates, we propose an implicit variant of SARSA that evaluates the current state-action value at the next iterate. In standard SARSA, the $t^{\text{th}}$ update direction is 
$$
\left(R_t + \gamma \boldsymbol{\phi}_{t+1}^{\top}\widehat{\boldsymbol{\theta}}_{t} - \boldsymbol{\phi}_t^{\top}\widehat{\boldsymbol{\theta}}_{t}\right)\boldsymbol{\phi}_t,
$$
while implicit SARSA moves along 
$$
\left(R_t + \gamma \boldsymbol{\phi}_{t+1}^{\top}\widehat{\boldsymbol{\theta}}_{t} - \boldsymbol{\phi}_t^{\top}{\color{red}\widehat{\boldsymbol{\theta}}_{t+1}}\right)\boldsymbol{\phi}_t.
$$
This modification similarly transforms the SARSA update into a fixed-point equation: 
\begin{align*}
\left(I + \beta_t\boldsymbol{\phi}_t\boldsymbol{\phi}_t^{\top}\right)\widehat{\boldsymbol{\theta}}_{t+1} = \widehat{\boldsymbol{\theta}}_{t} + \beta_t\left(R_t + \gamma\boldsymbol{\phi}_{t+1}^{\top}\widehat{\boldsymbol{\theta}}_{t}\right)\boldsymbol{\phi}_t.
\end{align*}
Both implicit Q-learning and SARSA share the structure of incorporating the future iterate, differing only in action selection: Q-learning uses greedy maximization while SARSA follows the policy $\pi_{t} = \mathcal{I}\{\boldsymbol{\phi}(\boldsymbol{s},a)^{\top} \widehat{\boldsymbol{\theta}}_{t}\}$. This implicit formulation enhances stability through adaptive, feature-dependent step-size adjustment, which can be seen in the proposition below.

\begin{proposition}[Implicit SARSA update]
\label{prop:implicit_sarsa}
The implicit linear SARSA algorithm admits the following update rule:
\begin{align*}
\widehat{\boldsymbol{\theta}}_{t+1} = \widehat{\boldsymbol{\theta}}_{t} + \frac{\beta_t\left(R_t + \gamma \boldsymbol{\phi}_{t+1}^\top\widehat{\boldsymbol{\theta}}_{t} - \boldsymbol{\phi}_t^\top\widehat{\boldsymbol{\theta}}_{t}\right)}{1 + \beta_t \|\boldsymbol{\phi}_t\|_2^2}  \boldsymbol{\phi}_t.
\end{align*}
\end{proposition}

\begin{algorithm}[htp]
\caption{(Projected) Implicit SARSA}
\begin{algorithmic}[1]\label{ALG:SARSA}
\STATE \textbf{Input:} step-sizes $\{\beta_t\}_{t \in \mathbb{N}}$, initial estimate $\widehat{\boldsymbol{\theta}}_{0}$, initial state $\boldsymbol{S}_0$, basis functions $\{\phi_i\}_{i=1}^M$, policy improvement operator $\mathcal{I}$, projection radius $r>0$
\STATE \textbf{Notation:} $\boldsymbol{\phi}_0 \equiv \boldsymbol{\phi}(\boldsymbol{S}_0, A_0)$; $\boldsymbol{\phi}_1 \equiv \boldsymbol{\phi}(\boldsymbol{S}_1^{\pi_{0}}, A_1^{\pi_0})$; $\boldsymbol{\phi}_t \equiv \boldsymbol{\phi}(\boldsymbol{S}_t^{\pi_{t-2}}, A_t^{\pi_{t-1}})$; $\boldsymbol{\phi}_{t+1} \equiv \boldsymbol{\phi}(\boldsymbol{S}_{t+1}^{\pi_{t-1}}, A_{t+1}^{\pi_t})$ $(t \ge 2)$;~ $\widetilde{\beta}_t \equiv \beta_t/(1 + \beta_t\|\boldsymbol{\phi}_t\|_2^2)$
\STATE Initialize $\pi_0 = \mathcal{I}\{\boldsymbol{\phi}(\boldsymbol{s},a)^{\top}\widehat{\boldsymbol{\theta}}_0\}$
\STATE Take $A_0 \sim \pi_0(\cdot|\boldsymbol{S}_0)$, observe $R_0 = r(\boldsymbol{S}_0, A_0)$
\STATE Transition to $\boldsymbol{S}_1^{\pi_{0}} \sim \Ptrans(\cdot|\boldsymbol{S}_0, A_0)$
\STATE Take $A_1^{\pi_0} \sim \pi_0(\cdot|\boldsymbol{S}_1^{\pi_{-1}})$
\STATE $\widehat{\boldsymbol{\theta}}_1 \gets \widehat{\boldsymbol{\theta}}_0 + \widetilde{\beta}_0(R_0 + \gamma\boldsymbol{\phi}_1^{\top}\widehat{\boldsymbol{\theta}}_0 - \boldsymbol{\phi}_0^{\top}\widehat{\boldsymbol{\theta}}_0)\boldsymbol{\phi}_0$
\STATE (If projection applied): $\widehat{\boldsymbol{\theta}}_1 \gets \Pi_r \widehat{\boldsymbol{\theta}}_1$
\FOR{$t=1,2,\dots$}
  \STATE $\pi_t = \mathcal{I}\{\boldsymbol{\phi}(\boldsymbol{s},a)^{\top}\widehat{\boldsymbol{\theta}}_t\}$ 
  \STATE Observe $R_t = r(\boldsymbol{S}_t^{\pi_{t-2}}, A_t^{\pi_{t-1}})$
  \STATE Transition to $\boldsymbol{S}_{t+1}^{\pi_{t-1}} \sim \Ptrans(\cdot|\boldsymbol{S}_t^{\pi_{t-2}}, A_t^{\pi_{t-1}})$
  \STATE Take $A_{t+1}^{\pi_t} \sim \pi_t(\cdot|\boldsymbol{S}_{t+1}^{\pi_{t-1}})$
  \STATE $\widehat{\boldsymbol{\theta}}_{t+1} \gets \widehat{\boldsymbol{\theta}}_t + \widetilde{\beta}_t(R_t + \gamma\boldsymbol{\phi}_{t+1}^{\top}\widehat{\boldsymbol{\theta}}_t - \boldsymbol{\phi}_t^{\top}\widehat{\boldsymbol{\theta}}_t)\boldsymbol{\phi}_t$
    \STATE (If projection applied): $\widehat{\boldsymbol{\theta}}_{t+1} \gets \Pi_r \widehat{\boldsymbol{\theta}}_{t+1}$
\ENDFOR
\end{algorithmic}
\end{algorithm}

\section{Non-Asymptotic Analysis of Implicit Q-learning and SARSA}
\noindent In this section, after incorporating a projection step, we provide finite-time error bounds for both the standard algorithms and their implicit counterparts under constant and decreasing step-size sequences. Our theoretical analysis reveals a unifying principle: implicit variants admit substantially greater flexibility in step-size selection while maintaining rigorous convergence guarantees, thereby resolving the numerical instability issues of standard algorithms described in Section \ref{sec:background}.

\subsection{Finite-time Analysis of Q-learning}\label{subsec:fin_IMP_Q}
\noindent To facilitate the theoretical analysis of linear Q-learning, we first assume uniform geometric ergodicity of the underlying Markov chain $\{\boldsymbol{S}_t^{\pi}\}_{t\geq 0}$ induced by a policy $\pi$.
\begin{assumption}[Uniform geometric ergodicity]
\label{assume:ergodicity}
Given a behavioral policy $\pi$, the Markov chain $\{\boldsymbol{S}_t^{\pi}\}_{t\geq 0}$ is uniformly geometrically ergodic with a stationary distribution $\mu^\pi$, i.e., there exist constants $m > 0$ and $\rho \in (0,1)$ such that
\begin{align*}
\sup_{\boldsymbol{s}\in\mathcal{S}} d_{TV}\{\mathbb{P}(\boldsymbol{S}_t^{\pi} \in \cdot\mid\boldsymbol{S}_0^{\pi} = \boldsymbol{s}), \mu^\pi\} \leq m\rho^t, \quad \forall t \geq 0,
\end{align*}
where $d_{TV}(P,Q)$ denotes the total variation distance between the probability distributions $P$ and $Q$.
\end{assumption}

\begin{assumption}[Near optimality of behavioral policy]
\label{assume:optimality_behavioral}
For all $\boldsymbol{\theta} \in \mathbb{R}^M$,  $\boldsymbol{\Sigma}^{\mu^\pi} \succ \gamma^2 \boldsymbol{\Sigma}^*(\boldsymbol{\theta})$, where 
\begin{align*}
    \boldsymbol{\Sigma}^{\mu^\pi} &:= \mathbb{E}^{\mu^\pi}\{\boldsymbol{\phi}(\boldsymbol{S},A)\boldsymbol{\phi}(\boldsymbol{S},A)^\top\}\\ \boldsymbol{\Sigma}^*(\boldsymbol{\theta}) &:= \mathbb{E}^{\mu^\pi}\{\boldsymbol{\phi}(\boldsymbol{S}, A_{\boldsymbol{S}}^{\boldsymbol{\theta}})\boldsymbol{\phi}(\boldsymbol{S}, A_{\boldsymbol{S}}^{\boldsymbol{\theta}})^\top\}
\end{align*}
with $A_{\boldsymbol{S}}^{\boldsymbol{\theta}} := \arg\max_{a\in\mathcal{A}} \boldsymbol{\phi}(\boldsymbol{S},a)^\top\boldsymbol{\theta}$.
\end{assumption}

\noindent Here, $\mathbb{E}^{\mu^\pi}$ denotes the expectation where $\boldsymbol{S}$ follows the stationary distribution $\mu^\pi$, $A$ is sampled according to the behavioral policy $\pi(\cdot | \boldsymbol{S})$, and $\boldsymbol{S}'$ denotes the subsequent state. Let $w_q > 0$ be the minimum smallest eigenvalue of $\boldsymbol{\Sigma}^{\mu^\pi} - \gamma^2\boldsymbol{\Sigma}^*(\boldsymbol{\theta})$ over all $\boldsymbol{\theta} \in \mathbb{R}^M$. Under Assumptions \ref{assume:ergodicity}--\ref{assume:optimality_behavioral}, \citet{melo2008analysis} established asymptotic convergence of Q-learning iterates to the limit point $\boldsymbol{\theta}^\star$, and \citet{chen2019performance} provided finite-time error bounds under a restrictive step-size condition. The limit point $\boldsymbol{\theta}^\star$ of Q-learning satisfies
$$
\mathbb{E}^{\mu^\pi}\left[\phi(\boldsymbol{S}, A)\left\{r(\boldsymbol{S}, A) + \gamma \max_{a \in \mathcal{A}} \phi(\boldsymbol{S}', a)^{\top} \boldsymbol{\theta}^\star - \phi(\boldsymbol{S}, A)^{\top}\boldsymbol{\theta}^\star\right\}\right] = 0.
$$
We present finite-time error bounds for projected linear Q-learning and its implicit variant under constant step-sizes, demonstrating implicit Q-learning's superior robustness to step-size selection.

\begin{theorem}[Finite-time bounds for Q-learning under constant step-size]
\label{thm:constant_qlearning}
Let $r > 0$ satisfy $\|\boldsymbol{\theta}^{\star}\|_2 \leq r$ and $\beta \in (0, 1/w_q)$ be a constant step-size. Under Assumptions \ref{assume:ergodicity} and \ref{assume:optimality_behavioral}, projected linear Q-learning satisfies
\begin{equation*}
\mathbb{E}^{\mu^\pi}\left\|\widehat{\boldsymbol{\theta}}_{t+1} - \boldsymbol{\theta}^{\star}\right\|_2^2 \le    (1 - \beta w_q)^{t+1}\left\|\widehat{\boldsymbol{\theta}}_0- \boldsymbol{\theta}^{\star}\right\|_2^2 + \mathcal{O}\left( \tau_\beta \beta\right ),
\end{equation*}
where $\tau_\beta = \min\{n \in \mathbb{N} : m\rho^n \leq \beta\}$. In comparison, projected implicit linear Q-learning with a constant step-size $\beta$ satisfying $\frac{\beta w_q}{1+\beta} \in (0, 1)$ yields
\begin{equation*}
\mathbb{E}^{\mu^\pi}\left\|\widehat{\boldsymbol{\theta}}^{\text{im}}_{t+1} - \boldsymbol{\theta}^{\star}\right\|_2^2 \le  \left(1 - \frac{\beta w_q}{1+\beta}\right)^{t+1}\left\|\widehat{\boldsymbol{\theta}}_0- \boldsymbol{\theta}^{\star}\right\|_2^2 + \mathcal{O}\left( \tau_\beta \beta + \tau_\beta\beta^2\right).
\end{equation*}
\end{theorem}
\noindent We next present finite-time error bounds for projected linear Q-learning and its implicit variant under decreasing step-sizes.
\begin{theorem}[Finite-time bounds for Q-learning under decreasing step-size]
\label{thm:decreasing_qlearning}
Let $r > 0$ satisfy $\|\boldsymbol{\theta}^{\star}\|_2 \leq r$ and $\{\beta_t\}_{t\ge 0}$ be a decreasing step-size sequence of the form $\beta_t = \frac{\beta_0}{(t+1)^s}$, where $s \in (0,1)$ and $\beta_0 \in (0, 1/w_q)$. Under Assumptions \ref{assume:ergodicity} and \ref{assume:optimality_behavioral}, projected linear Q-learning satisfies
\begin{equation*}
\mathbb{E}^{\mu^\pi} \left\|\widehat{\boldsymbol{\theta}}_{t+1} - \boldsymbol{\theta}^{\star}\right\|_2^2 \leq \exp\left[\frac{-\beta_0w_q}{(1-s)} \{(t+1)^{1-s} - 1\}\right] \left\|\widehat{\boldsymbol{\theta}}_{0} - \boldsymbol{\theta}^{\star}\right\|_2^2 + \mathcal{O}\left(\tau_{\beta_t} \beta_{t-\tau_{\beta_t}} \right), 
\end{equation*}
where $\tau_{\beta_t} = \min\{n \in \mathbb{N} : m\rho^n \leq \beta_t\}$. In comparison, projected implicit linear Q-learning with step-sizes $\beta_t = \frac{\beta_0}{(t+1)^s}$, $s \in (0,1)$, satisfying $\frac{\beta_0 w_q}{1+\beta_0} < 1$ yields
\begin{equation*}
\mathbb{E}^{\mu^\pi} \left\|\widehat{\boldsymbol{\theta}}^{\text{im}}_{t+1} - \boldsymbol{\theta}^{\star}\right\|_2^2 \leq \exp\left[\frac{-\beta_0w_q}{(1+\beta_0)(1-s)} \{(t+1)^{1-s} - 1\}\right] \left\|\widehat{\boldsymbol{\theta}}_{0}- \boldsymbol{\theta}^{\star}\right\|_2^2 + \mathcal{O}\left(\tau_{\beta_t} \beta_{t-\tau_{\beta_t}} \right).
\end{equation*}
\end{theorem}

\begin{remark}
[Comparison on step-size condition] For a constant step-size, standard Q-learning requires $\beta \in (0, 1/w_q)$, which becomes increasingly restrictive as the convergence rate $w_q$ increases. In comparison, implicit Q-learning requires a substantially weaker requirement that $\frac{\beta w_q}{1+\beta} \in (0,1)$. This condition is satisfied for all $\beta > 0$ when $w_q < 1$, thus implicit Q-learning tolerates arbitrarily large step-sizes in this regime. For decreasing step-sizes, implicit Q-learning similarly relaxes the step-size condition, enabling significantly larger initial step-sizes, facilitating faster initial learning. 
\end{remark}

\begin{remark}[Comparison on convergence rate and error terms]
In terms of convergence rates, implicit Q-learning achieves comparable performance to standard Q-learning. For constant step-sizes, both methods achieve $\mathcal{O}(\tau_\beta\beta)$ rates, with implicit Q-learning having an additional $\mathcal{O}(\tau_\beta\beta^2)$ term that becomes negligible for small $\beta$. For decreasing step-sizes, both methods exhibit exponential decay with bias terms of order $\mathcal{O}(\tau_{\beta_t}\beta_{t-\tau_{\beta_t}})$. The primary advantage of implicit Q-learning lies not in faster convergence, but in the substantial relaxation of step-size constraints, which enables more aggressive learning rates while maintaining theoretical guarantees.
\end{remark}

\subsection{Finite-time Analysis of SARSA}\label{subsec:fin_IMP_SARSA}
\noindent We now establish finite-time error bounds for the implicit SARSA algorithm under both constant and decreasing step-sizes. The analysis is more intricate than that of Q-learning as the time-varying behavioral policy induces a non-stationary Markov chain, precluding direct application of Assumption \ref{assume:ergodicity} and \ref{assume:optimality_behavioral}. Instead, for any $\boldsymbol{\theta} \in \mathbb{R}^M$, we assume a uniform geometric ergodicity of the Markov chain $\{{\boldsymbol{S}_t^{\pi_{\boldsymbol{\theta}}}}\}_{t \ge 0}$ induced by the policy $\pi_{\boldsymbol{\theta}}$.

\begin{assumption}[Uniform geometric ergodicity]
\label{assume:ergodicity_SARSA}
For any $\boldsymbol{\theta} \in \mathbb{R}^M$, a Markov chain $\{{\boldsymbol{S}_t^{\pi_{\boldsymbol{\theta}}}}\}_{t \ge 0}$ induced from the policy $\pi_{\boldsymbol{\theta}}$ is uniformly geometrically ergodic with a stationary distribution $\mu^{\boldsymbol{\theta}}$, i.e., there exist constants $m > 0$ and $\rho \in (0,1)$ such that
\begin{align*}
\sup_{\boldsymbol{s}\in\mathcal{S}} d_{TV}\left\{\mathbb{P}(\boldsymbol{S}_t^{\pi_{\boldsymbol{\theta}}} \in \cdot\mid\boldsymbol{S}_0^{\pi_{\boldsymbol{\theta}}} = \boldsymbol{s}), \mu^{\boldsymbol{\theta}}\right\} \leq m\rho^t, \quad \forall t \geq 0,
\end{align*}
where $d_{TV}(P,Q)$ denotes the total variation distance between the probability distributions $P$ and $Q$.
\end{assumption}

\noindent Let $\boldsymbol{\theta}^{\star}$ be the limit point of Algorithm \ref{ALG:SARSA} and define 
$$
\boldsymbol{A}^{\boldsymbol{\theta}^{\star}} := \mathbb{E}^{\mu^{\boldsymbol{\theta}^{\star}}}\left[\boldsymbol{\phi}(\boldsymbol{S},A)\left\{\gamma\boldsymbol{\phi}(\boldsymbol{S}',A')^\top - \boldsymbol{\phi}(\boldsymbol{S},A)^\top\right\}\right]
$$ 
where $\mathbb{E}^{\mu^{\boldsymbol{\theta}^{\star}}}$ denotes the expectation under the following distribution: the current state $\boldsymbol{S}$ is drawn from the stationary distribution $\mu^{\boldsymbol{\theta}^{\star}}$, the current action $A$ is sampled according to the behavioral policy $\pi_{\boldsymbol{\theta}^{\star}}(\cdot | \boldsymbol{S})$, the future state $\boldsymbol{S}'$ is drawn with the transition probability $\Ptrans(\cdot | \boldsymbol{S}, A)$, and the future action $A'$ is chosen following the behavioral policy $\pi_{\boldsymbol{\theta}^{\star}}(\cdot | \boldsymbol{S}')$. Following the convergence analysis of SARSA in \citep{melo2008analysis, zou2019finite}, we restrict our attention to the class of Lipschitz behavioral policies.

\begin{assumption}[Lipschitz policy and solvability]
\label{assume:lipshitz_SARSA}
For any $\boldsymbol{\theta}_1, \boldsymbol{\theta}_2 \in \mathbb{R}^M$ and $(\boldsymbol{s},a) \in \mathcal{S} \times \mathcal{A}$, a behavioral policy $\pi_{\boldsymbol{\theta}}(a|\boldsymbol{s})$ satisfies
$
|\pi_{\boldsymbol{\theta}_1}(a|\boldsymbol{s}) - \pi_{\boldsymbol{\theta}_2}(a|\boldsymbol{s})| \leq C\|\boldsymbol{\theta}_1 - \boldsymbol{\theta}_2\|_2
$
for some $C > 0$. Furthermore, the Lipschitz constant $C > 0$ is not so large that the matrix $\boldsymbol{A}^{\boldsymbol{\theta}^{\star}} + C\lambda \boldsymbol{I}$ is negative definite, where $\lambda := (b+2r)|\mathcal{A}|\left(2 + \lceil{\log_{\rho} \frac{1}{m}}\rceil + \frac{1}{1-\rho}\right)$.
\end{assumption}

\noindent Under the aforementioned assumptions, \citet{melo2008analysis} established asymptotic convergence of SARSA iterates to the limit point $\boldsymbol{\theta}^\star$, and \citet{zou2019finite} obtained finite-time error bounds under a restrictive step-size condition. The limit point $\boldsymbol{\theta}^\star$ of SARSA satisfies
$$
\mathbb{E}^{\mu^{\boldsymbol{\theta}^{\star}}}\left[\phi(\boldsymbol{S}, A)\left\{r(\boldsymbol{S}, A) + \gamma \phi(\boldsymbol{S}', A')^{\top} \boldsymbol{\theta}^\star - \phi(\boldsymbol{S}, A)^{\top}\boldsymbol{\theta}^\star\right\}\right] = 0.
$$
Let $w_s > 0$ be the smallest eigenvalue of $-\frac{1}{2}\{(\boldsymbol{A}^{\boldsymbol{\theta}^{\star}} + C\lambda \boldsymbol{I}) + (\boldsymbol{A}^{\boldsymbol{\theta}^{\star}} + C\lambda \boldsymbol{I})^\top\}$. We now present error bounds for projected linear SARSA and its implicit variant under constant step-sizes.
\begin{theorem}[Finite-time bound for implicit SARSA under constant step-size]
\label{thm:constant_SARSA}
Let $r > 0$ satisfy $\|\boldsymbol{\theta}^{\star}\|_2 \leq r$, and let $\beta$ be a constant step-size satisfying $\frac{2\beta w_s}{1+\beta} \in (0,1)$. Under Assumptions \ref{assume:ergodicity_SARSA} and \ref{assume:lipshitz_SARSA}, projected implicit SARSA yields
\begin{equation*}
\mathbb{E}\left\|\widehat{\boldsymbol{\theta}}^{\text{im}}_{t+1} - \boldsymbol{\theta}^{\star}\right\|_2^2 \le  \left(1 - \frac{2\beta w_s}{1+\beta}\right)^{t+1}\left\|\widehat{\boldsymbol{\theta}}^{\text{im}}_0- \boldsymbol{\theta}^{\star}\right\|_2^2 + \mathcal{O}\left\{(1+\beta)\beta(\tau_\beta^2 + \tau_\beta + 1)\right\}.
\end{equation*}
\end{theorem}
\noindent We next present finite-time error bounds for projected linear SARSA and its implicit variant under decreasing step-sizes.

\begin{theorem}[Finite-time bounds for SARSA under decreasing step-size]
\label{thm:decreasing_SARSA}
Let $r > 0$ satisfy $\|\boldsymbol{\theta}^{\star}\|_2 \leq r$ and $\{\beta_t\}_{t \ge 0}$ be a decreasing step-size sequence of the form $\beta_t = \frac{\beta_0}{(t+1)^s}$, where $s \in (0,1)$ and $\beta_0 \in (0, 1/2w_s)$. Under Assumptions \ref{assume:ergodicity_SARSA} and \ref{assume:lipshitz_SARSA}, projected linear SARSA yields
\begin{equation*}
\mathbb{E} \left\|\widehat{\boldsymbol{\theta}}_{t+1} - \boldsymbol{\theta}^{\star}\right\|_2^2 \leq \exp\left[\frac{-2\beta_0 w_s}{(1-s)} \{(t+1)^{1-s} - 1\}\right] \left\|\widehat{\boldsymbol{\theta}}_{0} - \boldsymbol{\theta}^{\star}\right\|_2^2 + \mathcal{O}\left\{\beta_{t-\tau_{\beta_t}}\left(\tau_{\beta_t}^2 + \tau_{\beta_t} \right) \right\}, 
\end{equation*}
where $\tau_{\beta_t} = \min\{n \in \mathbb{N} : m\rho^n \leq \beta_t\}$. In comparison, projected implicit linear SARSA with step-size $\beta_t = \frac{\beta_0}{(t+1)^s}, s \in (0,1)$ satisfying $\frac{2\beta_0 w_s}{1+\beta_0} \in (0, 1)$ yields 
\begin{equation*}
\mathbb{E} \left\|\widehat{\boldsymbol{\theta}}^{\text{im}}_{t+1} - \boldsymbol{\theta}^{\star}\right\|_2^2 \leq \exp\left[\frac{-2\beta_0 w_s}{(1+\beta_0)(1-s)} \{(t+1)^{1-s} - 1\}\right] \left\|\widehat{\boldsymbol{\theta}}_{0}- \boldsymbol{\theta}^{\star}\right\|_2^2 + \mathcal{O}\left\{\beta_{t-\tau_{\beta_t}}\left(\tau_{\beta_t}^2 + \tau_{\beta_t} \right) \right\} .
\end{equation*}
\end{theorem}

\begin{remark}
[Comparison on step-size conditions] For both constant and decreasing step-sizes, the condition $\frac{2\beta w_s}{1+\beta} < 1$ for implicit SARSA is substantially less restrictive than the requirement $\beta < \frac{1}{2w_s}$ for standard SARSA. Notably, when $w_s \le 1/2$, implicit SARSA permits arbitrary step-sizes with a finite-time error bound guarantee.
\end{remark}

\begin{remark}[Comparison on convergence rate and error terms] For a constant step-size, standard SARSA attains an asymptotic error of $\mathcal{O}\{\beta(\tau_\beta^2 + \tau_\beta + 1)\}$ \citep{zou2019finite}, while implicit SARSA attains $\mathcal{O}\{(1+\beta)\beta(\tau_\beta^2 + \tau_\beta + 1)\}$. For moderate $\beta$, both methods yield comparable error bounds; for large $\beta$, implicit SARSA exhibits increased asymptotic error due to $(1+\beta)$ factor, though this is offset by its ability to maintain numerical stability at step-sizes where standard SARSA diverges. For decreasing step-sizes $\beta_t = \beta_0/(t+1)^s$ with $s \in (0,1)$, both methods exhibit exponential decay with bias terms of order $\mathcal{O}\{\beta_{t-\tau_{\beta_t}}(\tau_{\beta_t}^2 + \tau_{\beta_t})\}$. Under geometric ergodicity assumptions, our analysis for standard SARSA achieves $\mathcal{O}(\log^2 T/T)$ convergence, improving upon the $\mathcal{O}(\log^3 T/T)$ bound in \citet{zou2019finite}, while implicit SARSA attains the same rate under the substantially relaxed step-size condition. The convergence rates are thus essentially equivalent, with the primary advantage of implicit method being substantially improved numerical stability over a wider range of step-sizes.
\end{remark}

\section{Numerical experiments}
\noindent In this section, we evaluate implicit Q-learning and implicit SARSA algorithms on four benchmark environments from the open-source \texttt{gymnasium} library in Python \citep{towers2024gymnasium}, which is commonly used to evaluate reinforcement learning methods. The test suite includes two discrete state environments (Cliff Walking, Taxi) and two continuous state environments (Mountain Car, Acrobot) that require function approximation. Throughout all experiments, Q-learning employs an $\epsilon$-greedy behavioral policy: with probability $\epsilon$, a random action is selected; otherwise, the greedy action $\arg\max_{a} \widehat{Q}(s,a)$ is chosen. SARSA uses an $\epsilon$-softmax policy: $$\pi(a|s) = \frac{\epsilon}{|\mathcal{A}|} + (1-\epsilon)\frac{\exp\left\{\widehat{Q}(s,a)/\iota\right\}}{\sum_{a'\in\mathcal{A}} \exp\left\{\widehat{Q}(s,a')/\iota\right\}},$$ where $\iota = 0.05$ is the temperature parameter controlling the softness of the action distribution.

\subsection{Discrete State Environments}
\noindent We first evaluate performance on two discrete state environments with finite state-action spaces represented using one-hot feature vectors. In Cliff Walking, an agent navigates a grid from a designated start state to a goal state while avoiding a cliff region; falling off the cliff incurs a large negative reward and resets the agent to the start. In Taxi, an agent navigates a grid world to pick up a passenger from one location and deliver them to a specified destination, receiving penalties for illegal pickup or dropoff attempts and a time penalty for each step taken. We use discount factor $\gamma = 0.99$ and projection radius $r = 5000$, train for 400 episodes (maximum 10000 steps each), and average results over 50 independent runs. Following \citep{mnih2015human}, the exploration parameter decays linearly from $\epsilon_0 = 0.1$ to $\epsilon_{400} = 0.01$ over the 400 episodes. 

Figure \ref{fig:control_discrete} demonstrates that our proposed implicit variants resolve the stability-efficiency trade-off. While Q-learning (blue) and SARSA (green) exhibit sharp performance degradation beyond their thresholds, implicit Q-learning (orange) and implicit SARSA (red) maintain stable performance across the entire range of step-sizes tested, remaining robust even at $\beta = 2.0$ where standard methods fail catastrophically. The bottom panels show that at fixed iteration counts, implicit variants achieve strong performance at large step-sizes where standard methods collapse, enabling practitioners to use aggressive $\beta$ values that accelerate learning. This relaxes the restrictive conditions on step-size required to avoid numerical instability, enabling faster convergence.

\begin{figure}[htp]
\centering
\includegraphics[height=.295\textwidth]{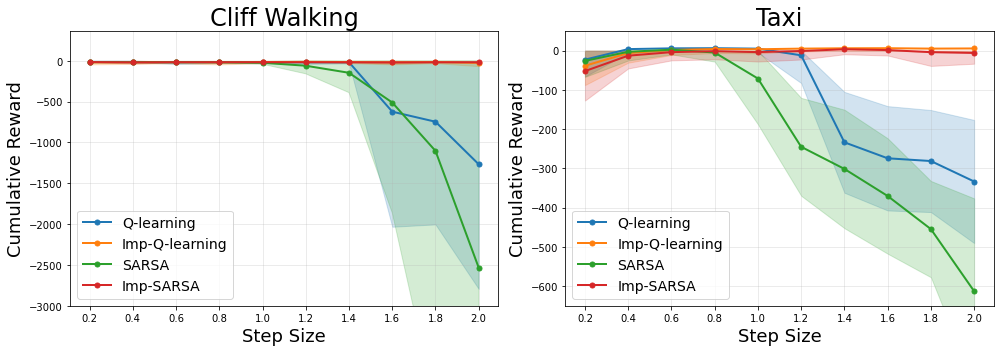}
\includegraphics[height=.295\textwidth]{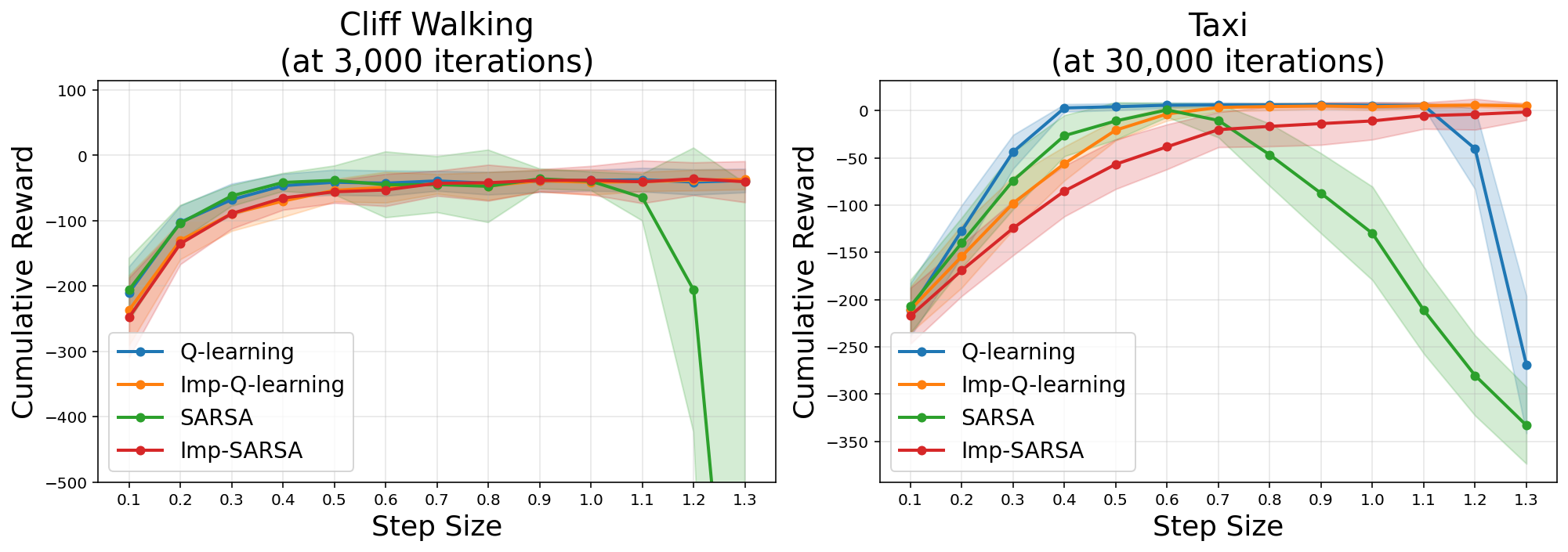}
\caption{Performance comparison on Cliff Walking and Taxi environments. Top: Average cumulative reward versus step-size $\beta$ shows that while standard Q-learning (blue) and SARSA (green) degrade sharply beyond thresholds ($\beta \approx 1.0$ for Cliff Walking, $\beta \approx 0.8$ for Taxi), implicit Q-learning (orange) and implicit SARSA (red) maintain stable performance across the entire range up to $\beta = 2.0$. Bottom: Cumulative reward at fixed iterations (3,000 for Cliff Walking, 30,000 for Taxi) demonstrates that implicit variants achieve strong performance at large step-sizes where standard methods fail. Shaded regions represent standard errors over 50 runs.}
\label{fig:control_discrete}
\end{figure}

\subsection{Continuous State Environments}
\noindent We next evaluate the implicit algorithms on the Acrobot and Mountain Car environments, which require function approximation due to continuous state spaces. In Acrobot, a two-link underactuated robot must swing up to a target height while controlling only the joint between the two links. In Mountain Car, an agent must drive an underpowered car up a steep hill by building momentum, with the goal of reaching the target position in minimal time. We use radial basis features with four kernels at different length scales $\{5.0, 2.0, 1.0, 0.5\}$, each producing 100 components, for a total of 400 features per state. We employ a decreasing step-size schedule of the form $\beta_t = \beta_0/(t+1)^{2/3}$, where $t$ is the iteration index within an episode, with $\beta_0 \in \{1, 10\}$ for Acrobot, $\beta_0 \in \{1, 5\}$ for Mountain Car. We use discount factor $\gamma = 0.99$, exploration parameter $\epsilon = 0.1$, and projection radius $r = 1000$. For each environment, we conduct 20 independent runs of 30 episodes.

Figure \ref{fig:control_continuous} illustrates the learning dynamics across episodes for both algorithm classes in continuous state environments. When initialized with small step-sizes ($\beta_0 = 1$), all approaches exhibit comparable learning behavior and convergence patterns. The performance gap emerges distinctly with larger initial step-sizes: in Acrobot with $\beta_0 = 10$, standard Q-learning and SARSA show minimal improvement over 30 episodes, maintaining high episode lengths around 6.2 (log scale), while implicit variants demonstrate consistent learning, reducing episode lengths to approximately 5.25-5.5. The Mountain Car domain reveals an even more pronounced contrast at $\beta_0 = 5$: standard methods plateau early with episode lengths remaining above 5.2, whereas implicit algorithms achieve a steady reduction to approximately 4.9-5.0, indicating substantially more efficient policies. These observations validate that implicit formulations provide robustness advantages under decreasing step-size schedules, particularly when aggressive initial values are employed.

\begin{figure}[htp]
\centering
\includegraphics[height=.295\textwidth]{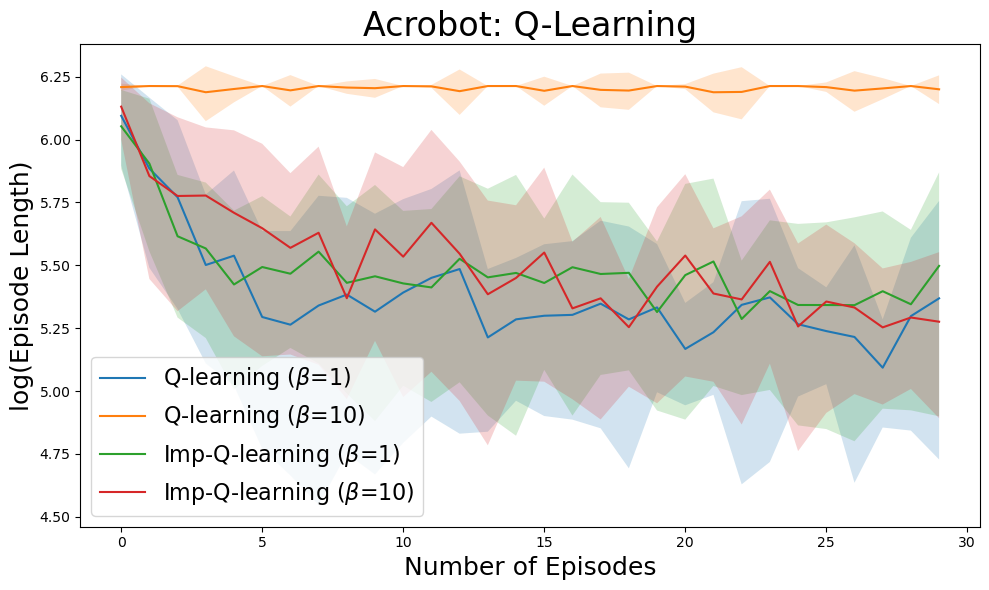}
\includegraphics[height=.295\textwidth]{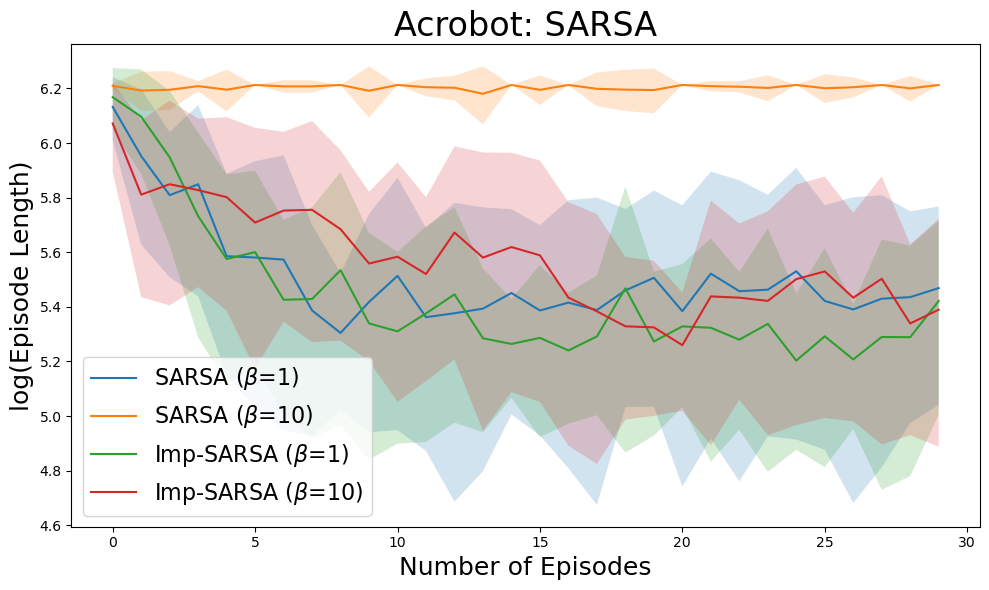}
\includegraphics[height=.295\textwidth]{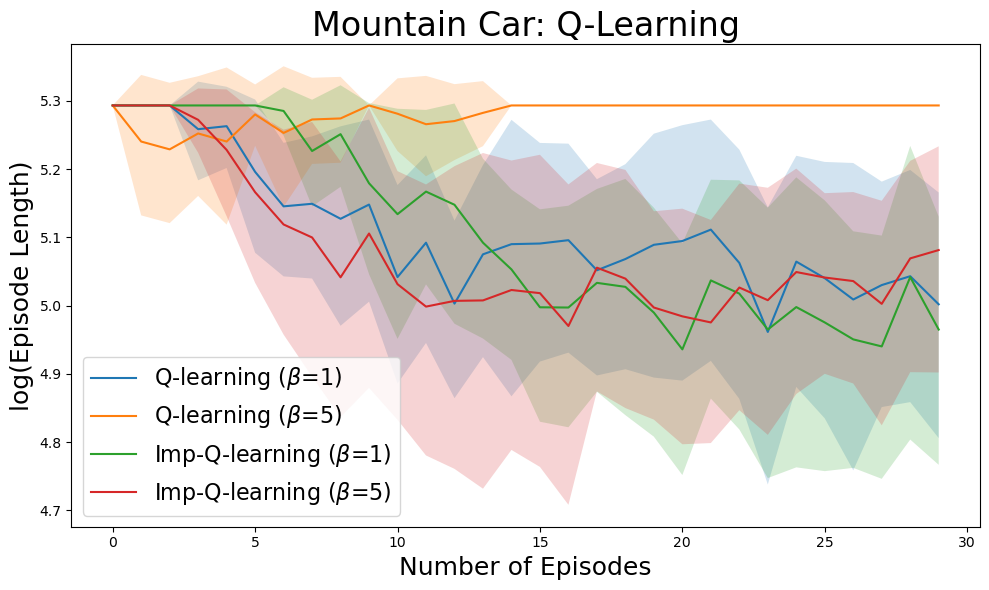}
\includegraphics[height=.295\textwidth]{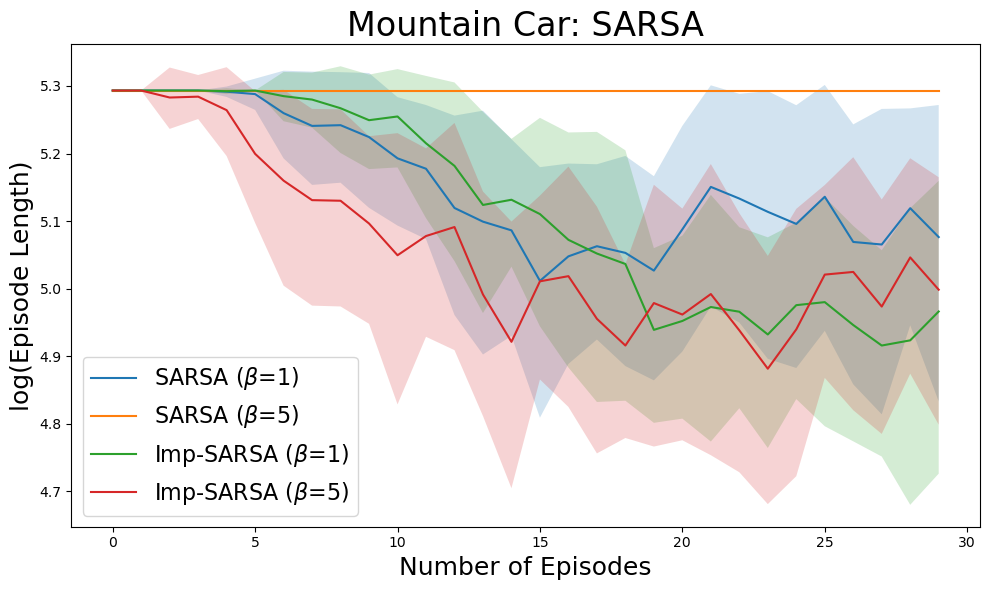}
\caption{Performance on continuous state environments with function approximation. Learning curves (log episode length vs episode number) for Q-learning and SARSA on Acrobot and Mountain Car with decreasing step-sizes $\beta_t = \beta_0/(t+1)^{2/3}$. Lower values indicate better performance. At a small initial step-size ($\beta_0=1$), all methods perform similarly. At a large initial step-size ($\beta_0=10$ for Acrobot, $\beta_0=5$ for Mountain Car), implicit methods (orange/red) demonstrate superior stability compared to standard methods (blue/green). Shaded regions show standard errors over 20 runs.}
\label{fig:control_continuous}
\end{figure}

\section{Conclusion}
\noindent This work addresses the persistent challenge of step-size sensitivity in Q-learning and SARSA. We reformulate these algorithms as fixed-point equations that evaluate state-action values at future iterates, inducing adaptive regularization scaling inversely with feature norms. Our theoretical analysis establishes that implicit methods tolerate arbitrarily large step-sizes under favorable conditions, a sharp contrast to standard algorithms' restrictive requirements. Empirical evaluation across four benchmark environments demonstrates that implicit variants maintain stable learning where standard methods fail. Extensions to nonlinear function approximation and average-reward formulations represent promising future directions where adaptive step-size control becomes increasingly critical for practical deployment.
\newpage
\bibliographystyle{plainnat}
\bibliography{references}
\newpage
\appendix
\section{Proofs of Proposition \ref{prop:implicit_q_learning} and \ref{prop:implicit_sarsa}}
\noindent \textbf{Proposition 3.1} (Implicit Q-learning update).
The implicit linear Q-learning algorithm admits the following update rule:
\begin{align*}
\widehat{\boldsymbol{\theta}}_{t+1} = \widehat{\boldsymbol{\theta}}_{t} + \frac{\beta_t}{1 + \beta_t \|\boldsymbol{\phi}(\boldsymbol{S}_t^{\pi},A_t^{\pi})\|_2^2} \left\{R_t^{\pi} + \gamma \max_{a' \in \mathcal{A}} \boldsymbol{\phi}(\boldsymbol{S}_{t+1}^{\pi},a')^\top\widehat{\boldsymbol{\theta}}_{t} - \boldsymbol{\phi}(\boldsymbol{S}_t^{\pi},A_t^{\pi})^\top\widehat{\boldsymbol{\theta}}_{t}\right\} \boldsymbol{\phi}(\boldsymbol{S}_t^{\pi},A_t^{\pi}).
\end{align*}

\begin{proof}
The implicit linear Q-learning update is defined as
\begin{align}\label{IMP_Q_ORIG}
\widehat{\boldsymbol{\theta}}_{t+1} = \widehat{\boldsymbol{\theta}}_{t} + \beta_t \left\{R_t^{\pi} + \gamma \max_{a' \in \mathcal{A}} \boldsymbol{\phi}(\boldsymbol{S}_{t+1}^{\pi},a')^\top\widehat{\boldsymbol{\theta}}_{t} - \boldsymbol{\phi}(\boldsymbol{S}_t^{\pi},A_t^{\pi})^\top\widehat{\boldsymbol{\theta}}_{t+1}\right\} \boldsymbol{\phi}(\boldsymbol{S}_t^{\pi},A_t^{\pi}).
\end{align}
\eqref{IMP_Q_ORIG} is equivalent to the fixed-point equation
\begin{align}
\left(I + \beta_t \boldsymbol{\phi}_t \boldsymbol{\phi}_t^\top\right) \widehat{\boldsymbol{\theta}}_{t+1} = \widehat{\boldsymbol{\theta}}_{t} + \beta_t \left\{R_t^{\pi} + \gamma \max_{a' \in \mathcal{A}} \boldsymbol{\phi}(\boldsymbol{S}_{t+1}^{\pi},a')^\top\widehat{\boldsymbol{\theta}}_{t}\right\} \boldsymbol{\phi}_t
\end{align}
where we denote $\boldsymbol{\phi}_t := \boldsymbol{\phi}(\boldsymbol{S}_t^{\pi},A_t^{\pi})$ for notational convenience. Using the Sherman-Morrison-Woodbury formula, we have
\begin{align*}
\left(I + \beta_t \boldsymbol{\phi}_t \boldsymbol{\phi}_t^\top\right)^{-1} = I - \frac{\beta_t \boldsymbol{\phi}_t \boldsymbol{\phi}_t^\top}{1 + \beta_t \|\boldsymbol{\phi}_t\|_2^2}.
\end{align*}
Solving for $\widehat{\boldsymbol{\theta}}_{t+1}$, we obtain
\begin{align*}
\widehat{\boldsymbol{\theta}}_{t+1} &= \left(I - \frac{\beta_t \boldsymbol{\phi}_t \boldsymbol{\phi}_t^\top}{1 + \beta_t \|\boldsymbol{\phi}_t\|_2^2}\right) \left[\widehat{\boldsymbol{\theta}}_{t} + \beta_t \left\{R_t^{\pi} + \gamma \max_{a' \in \mathcal{A}} \boldsymbol{\phi}(\boldsymbol{S}_{t+1}^{\pi},a')^\top\widehat{\boldsymbol{\theta}}_{t}\right\} \boldsymbol{\phi}_t\right] \\
&= \widehat{\boldsymbol{\theta}}_{t} + \beta_t \left\{R_t^{\pi} + \gamma \max_{a' \in \mathcal{A}} \boldsymbol{\phi}(\boldsymbol{S}_{t+1}^{\pi},a')^\top\widehat{\boldsymbol{\theta}}_{t}\right\} \boldsymbol{\phi}_t - \frac{\beta_t \boldsymbol{\phi}_t^\top\widehat{\boldsymbol{\theta}}_{t}}{1 + \beta_t \|\boldsymbol{\phi}_t\|_2^2} \boldsymbol{\phi}_t \\
&\quad - \frac{\beta_t^2 \|\boldsymbol{\phi}_t\|_2^2}{1 + \beta_t \|\boldsymbol{\phi}_t\|_2^2} \left\{R_t^{\pi} + \gamma \max_{a' \in \mathcal{A}} \boldsymbol{\phi}(\boldsymbol{S}_{t+1}^{\pi},a')^\top\widehat{\boldsymbol{\theta}}_{t}\right\} \boldsymbol{\phi}_t \\
&= \widehat{\boldsymbol{\theta}}_{t} + \beta_t \left(1 - \frac{\beta_t \|\boldsymbol{\phi}_t\|_2^2}{1 + \beta_t \|\boldsymbol{\phi}_t\|_2^2}\right) \left\{R_t^{\pi} + \gamma \max_{a' \in \mathcal{A}} \boldsymbol{\phi}(\boldsymbol{S}_{t+1}^{\pi},a')^\top\widehat{\boldsymbol{\theta}}_{t}\right\} \boldsymbol{\phi}_t - \frac{\beta_t}{1 + \beta_t \|\boldsymbol{\phi}_t\|_2^2} (\boldsymbol{\phi}_t^\top\widehat{\boldsymbol{\theta}}_{t}) \boldsymbol{\phi}_t \\
&= \widehat{\boldsymbol{\theta}}_{t} + \frac{\beta_t}{1 + \beta_t \|\boldsymbol{\phi}_t\|_2^2} \left\{R_t^{\pi} + \gamma \max_{a' \in \mathcal{A}} \boldsymbol{\phi}(\boldsymbol{S}_{t+1}^{\pi},a')^\top\widehat{\boldsymbol{\theta}}_{t} - \boldsymbol{\phi}_t^\top\widehat{\boldsymbol{\theta}}_{t}\right\} \boldsymbol{\phi}_t,
\end{align*}
which completes the proof.
\end{proof}

\noindent \textbf{Proposition 3.2} (Implicit SARSA update).
The implicit linear SARSA algorithm admits the following update rule:
\begin{align*}
\widehat{\boldsymbol{\theta}}_{t+1} = \widehat{\boldsymbol{\theta}}_{t} + \frac{\beta_t}{1 + \beta_t \|\boldsymbol{\phi}(\boldsymbol{S}_t^{\pi_{t-2}},A_t^{\pi_{t-1}})\|_2^2} \left\{R_t + \gamma \boldsymbol{\phi}(\boldsymbol{S}_{t+1}^{\pi_{t-1}},A_{t+1}^{\pi_t})^\top\widehat{\boldsymbol{\theta}}_{t} - \boldsymbol{\phi}(\boldsymbol{S}_t^{\pi_{t-2}},A_t^{\pi_{t-1}})^\top\widehat{\boldsymbol{\theta}}_{t}\right\} \boldsymbol{\phi}(\boldsymbol{S}_t^{\pi_{t-2}},A_t^{\pi_{t-1}}).
\end{align*}

\begin{proof}
The implicit linear SARSA update is defined as
\begin{align}\label{IMP_SARSA_ORIG}
\widehat{\boldsymbol{\theta}}_{t+1} = \widehat{\boldsymbol{\theta}}_{t} + \beta_t \left\{R_t + \gamma \boldsymbol{\phi}(\boldsymbol{S}_{t+1}^{\pi_{t-1}},A_{t+1}^{\pi_t})^\top\widehat{\boldsymbol{\theta}}_{t} - \boldsymbol{\phi}(\boldsymbol{S}_t^{\pi_{t-2}},A_t^{\pi_{t-1}})^\top\widehat{\boldsymbol{\theta}}_{t+1}\right\} \boldsymbol{\phi}(\boldsymbol{S}_t^{\pi_{t-2}},A_t^{\pi_{t-1}}).
\end{align}
Let $\boldsymbol{\phi}_t := \boldsymbol{\phi}(\boldsymbol{S}_t^{\pi_{t-2}},A_t^{\pi_{t-1}})$ and $\boldsymbol{\phi}_{t+1} := \boldsymbol{\phi}(\boldsymbol{S}_{t+1}^{\pi_{t-1}},A_{t+1}^{\pi_t})$. Then \eqref{IMP_SARSA_ORIG} yields the following fixed-point equation
\begin{align*}
\left(I + \beta_t \boldsymbol{\phi}_t \boldsymbol{\phi}_t^\top\right) \widehat{\boldsymbol{\theta}}_{t+1} = \widehat{\boldsymbol{\theta}}_{t} + \beta_t \left(R_t + \gamma \boldsymbol{\phi}_{t+1}^\top\widehat{\boldsymbol{\theta}}_{t}\right) \boldsymbol{\phi}_t.
\end{align*}
Solving for $\widehat{\boldsymbol{\theta}}_{t+1}$ with the Sherman-Morrison-Woodbury formula, we have
\begin{align*}
\widehat{\boldsymbol{\theta}}_{t+1} &= \left(I - \frac{\beta_t \boldsymbol{\phi}_t \boldsymbol{\phi}_t^\top}{1 + \beta_t \|\boldsymbol{\phi}_t\|_2^2}\right) \left\{\widehat{\boldsymbol{\theta}}_{t} + \beta_t \left(R_t + \gamma \boldsymbol{\phi}_{t+1}^\top\widehat{\boldsymbol{\theta}}_{t}\right) \boldsymbol{\phi}_t\right\} \\
&= \widehat{\boldsymbol{\theta}}_{t} + \beta_t \left(R_t + \gamma \boldsymbol{\phi}_{t+1}^\top\widehat{\boldsymbol{\theta}}_{t}\right) \boldsymbol{\phi}_t - \frac{\beta_t (\boldsymbol{\phi}_t^\top\widehat{\boldsymbol{\theta}}_{t}) \boldsymbol{\phi}_t}{1 + \beta_t \|\boldsymbol{\phi}_t\|_2^2} - \frac{\beta_t^2 \|\boldsymbol{\phi}_t\|_2^2}{1 + \beta_t \|\boldsymbol{\phi}_t\|_2^2} \left(R_t + \gamma \boldsymbol{\phi}_{t+1}^\top\widehat{\boldsymbol{\theta}}_{t}\right) \boldsymbol{\phi}_t \\
&= \widehat{\boldsymbol{\theta}}_{t} + \beta_t \left(1 - \frac{\beta_t \|\boldsymbol{\phi}_t\|_2^2}{1 + \beta_t \|\boldsymbol{\phi}_t\|_2^2}\right) \left(R_t + \gamma \boldsymbol{\phi}_{t+1}^\top\widehat{\boldsymbol{\theta}}_{t}\right) \boldsymbol{\phi}_t - \frac{\beta_t}{1 + \beta_t \|\boldsymbol{\phi}_t\|_2^2} (\boldsymbol{\phi}_t^\top\widehat{\boldsymbol{\theta}}_{t}) \boldsymbol{\phi}_t \\
&= \widehat{\boldsymbol{\theta}}_{t} + \frac{\beta_t}{1 + \beta_t \|\boldsymbol{\phi}_t\|_2^2} \left(R_t + \gamma \boldsymbol{\phi}_{t+1}^\top\widehat{\boldsymbol{\theta}}_{t} - \boldsymbol{\phi}_t^\top\widehat{\boldsymbol{\theta}}_{t}\right)\boldsymbol{\phi}_t,
\end{align*}
which completes the proof.
\end{proof}

\section{Proofs of Q-Learning Results}
\subsection{Preliminary Lemmas}
\noindent For any $\boldsymbol{\theta} \in \mathbb{R}^M$, let us define
\begin{align*}
f_t(\boldsymbol{\theta}) &= \boldsymbol{\phi}(\boldsymbol{S}_t^{\pi}, A_t^{\pi})\left\{R_t^{\pi} + \gamma \max_{a \in \mathcal{A}} \boldsymbol{\phi}(\boldsymbol{S}_{t+1}^{\pi}, a)^{\top} \boldsymbol{\theta} - \boldsymbol{\phi}(\boldsymbol{S}_t^{\pi}, A_t^{\pi})^{\top}\boldsymbol{\theta}\right\} \\
\bar{f}(\boldsymbol{\theta}) &= \mathbb{E}^{\mu^\pi}\left[\boldsymbol{\phi}(\boldsymbol{S}^{\pi}, A^{\pi})\left\{r(\boldsymbol{S}^{\pi}, A^{\pi}) + \gamma \max_{a \in \mathcal{A}} \boldsymbol{\phi}(\boldsymbol{S}', a)^{\top} \boldsymbol{\theta} - \boldsymbol{\phi}(\boldsymbol{S}^{\pi}, A^{\pi})^{\top}\boldsymbol{\theta}\right\}\right],
\end{align*}
where $\mathbb{E}^{\mu^\pi}$ denotes the expectation where $\boldsymbol{S}^{\pi}$ follows the stationary distribution $\mu^\pi$, $A^{\pi}$ is sampled according to the behavioral policy $\pi(\cdot | \boldsymbol{S}^{\pi})$, and $\boldsymbol{S}' \sim \Ptrans(\cdot\mid \boldsymbol{S}^{\pi}, A^{\pi})$. Let us also define $\Lambda_t(\boldsymbol{\theta}) := \langle\boldsymbol{\theta} - \boldsymbol{\theta}^{\star}, f_t(\boldsymbol{\theta}) - \bar{f}(\boldsymbol{\theta})\rangle$, where $\boldsymbol{\theta}^{\star}$ is the solution of $\bar{f}(\boldsymbol{\theta}) = 0$.

\begin{lemma}[Lemma 11 of \citep{zoufinite}]\label{lem:B1}
For any $\boldsymbol{\theta}\in \mathbb{R}^M$ such that $\|\boldsymbol{\theta}\|_2 \leq r$, 
$$
\|f_t(\boldsymbol{\theta})\|_2 \leq G:= b + 2r ~~\text{for all}~~ t \ge 0.
$$ 
\end{lemma}

\begin{lemma}[Lemma 12 of \citep{zoufinite}]\label{lem:B2}
For any $\boldsymbol{\theta} \in \mathbb{R}^M$, $\langle\boldsymbol{\theta} - \boldsymbol{\theta}^{\star}, \bar{f}(\boldsymbol{\theta}) - \bar{f}(\boldsymbol{\theta}^{\star})\rangle \leq -\frac{w_q\|\boldsymbol{\theta} - \boldsymbol{\theta}^{\star}\|_2^2}{2}$,
where $w_q > 0$ denotes the minimum of the smallest eigenvalue of $\boldsymbol{\Sigma}^\pi - \gamma^2\boldsymbol{\Sigma}^\star(\boldsymbol{\theta})$ over all $\boldsymbol{\theta} \in \mathbb{R}^M$. 
\end{lemma}


\begin{lemma}[Lemma 13 of \citep{zoufinite}]\label{lem:B3}
For any $\boldsymbol{\theta}\in \mathbb{R}^M$ such that $\|\boldsymbol{\theta}\|_2 \leq r$, $\Lambda_t(\boldsymbol{\theta}) \leq 2G^2$ for all $t \ge 0$.
\end{lemma}

\begin{lemma}[Lemma 14 of \citep{zoufinite}]\label{lem:B4}
For any $\boldsymbol{\theta}_1, \boldsymbol{\theta}_2 \in \mathbb{R}^M$ such that $\|\boldsymbol{\theta}_1\|_2 \le r,  \|\boldsymbol{\theta}_2\|_2 \le r$, 
$$
|\Lambda_t(\boldsymbol{\theta}_1) - \Lambda_t(\boldsymbol{\theta}_2)| \leq 6G\|\boldsymbol{\theta}_1 - \boldsymbol{\theta}_2\|_2.
$$
\end{lemma}

\begin{lemma}\label{lem:B5}
Suppose the step-size sequence $\{\beta_i\}_{i \in \mathbb{N}}$ is constant, i.e., $\beta_i = \beta, ~\forall i \in \mathbb{N}$. Let $\tau_\beta := \min\{n \in \mathbb{N} : m \rho^n \leq \beta\}$, then for all $i \in \mathbb{N}$
\begin{equation}
\mathbb{E}^{\mu^\pi}\left\{\Lambda_i(\widehat{\boldsymbol{\theta}}_i)\right\} \leq 4G^2 m \rho^{\tau_\beta} + 6\beta G^2 \tau_\beta \leq 10\beta G^2 \tau_\beta
\end{equation}
for projected Q-learning iterates $(\widehat{\boldsymbol{\theta}}_i)_{i \in \mathbb{N}}$.
\end{lemma}
\begin{proof}
For all $i \geq 0$, 
\begin{align*}
\left\|\widehat{\boldsymbol{\theta}}_{i+1} - \widehat{\boldsymbol{\theta}}_i\right\|_2 &= \left\|\Pi_{r}\{\widehat{\boldsymbol{\theta}}_i + \beta f_i(\widehat{\boldsymbol{\theta}}_i)\} - \Pi_{r}\widehat{\boldsymbol{\theta}}_i\right\|_2  \\
&\leq \left\|\widehat{\boldsymbol{\theta}}_i + \beta f_i(\widehat{\boldsymbol{\theta}}_i) - \widehat{\boldsymbol{\theta}}_i\right\|_2 && \text{(non-expansiveness of $\Pi_r$)}\\
&\leq \beta G && \text{(by Lemma \ref{lem:B1})}.
\end{align*}
When $i > \tau_{\beta}$, 
\begin{align*}
\left|\Lambda_i(\widehat{\boldsymbol{\theta}}_i) - \Lambda_i(\widehat{\boldsymbol{\theta}}_{i-\tau_{\beta}})\right| &\leq 6G \left\|\widehat{\boldsymbol{\theta}}_i - \widehat{\boldsymbol{\theta}}_{i-\tau_{\beta}}\right\|_2 && \text{(by Lemma \ref{lem:B4})}\\
&\leq 6G\left(\left\|\widehat{\boldsymbol{\theta}}_i - \widehat{\boldsymbol{\theta}}_{i-1}\right\|_2 + \cdots + \left\|\widehat{\boldsymbol{\theta}}_{i-\tau_\beta+1} - \widehat{\boldsymbol{\theta}}_{i-\tau_\beta}\right\|_2\right) \\
&\leq 6G(\beta G + \cdots + \beta G) \\
&\leq 6\beta G^2 \tau_\beta.
\end{align*}
Therefore,
\begin{equation}\label{eq:lambda_recursion_tau}
\Lambda_i(\widehat{\boldsymbol{\theta}}_i) \leq \Lambda_i(\widehat{\boldsymbol{\theta}}_{i-\tau_\beta}) + 6\beta G^2 \tau_\beta. 
\end{equation}
By the same logic, when $i \leq \tau_\beta$,  we have 
\begin{equation*}\label{eq:lambda_recursion_tau0}
\Lambda_i(\widehat{\boldsymbol{\theta}}_i) \leq \Lambda_i(\widehat{\boldsymbol{\theta}}_{0}) + 6\beta G^2 \tau_\beta.
\end{equation*}
Note that $\mathbb{E}^{\mu^\pi}\Lambda_i(\boldsymbol{\theta}) = 0$ for any fixed $\boldsymbol{\theta} \in \mathbb{R}^M$. Therefore, when $i \leq \tau_\beta$, $\mathbb{E}^{\mu^\pi}\Lambda_i(\widehat{\boldsymbol{\theta}}_{0}) = 0$ yields $\mathbb{E}^{\mu^\pi}\Lambda_i(\widehat{\boldsymbol{\theta}}_i) \leq 6\beta G^2 \tau_\beta$. Now, consider the other case when $i > \tau_\beta$, let $O_i := \{\boldsymbol{S}_i^{\pi}, A_i^{\pi}, \boldsymbol{S}_{i+1}^{\pi}\}$ and write $\Lambda(\widehat{\boldsymbol{\theta}}_{i-\tau_\beta}, O_i) = \Lambda_i(\widehat{\boldsymbol{\theta}}_{i-\tau_\beta})$ to highlight its dependence on $O_i$. Consider random variables $\widehat{\boldsymbol{\theta}}_{i-\tau_\beta}'$ and $O_i'$ drawn independently from the marginal distributions of $\widehat{\boldsymbol{\theta}}_{i-\tau_\beta}$ and $O_i$, respectively. Then we obtain
\begin{equation*}
\mathbb{E}^{\mu^\pi}\left\{\Lambda(\widehat{\boldsymbol{\theta}}_{i-\tau_\beta}', O_i')\right\} = \mathbb{E}\left[\mathbb{E}\{\Lambda(\widehat{\boldsymbol{\theta}}_{i-\tau_\beta}', O_i') \mid \widehat{\boldsymbol{\theta}}_{i-\tau_\beta}'\}\right] = 0.
\end{equation*}
Furthermore, since 
\begin{itemize}
    \item $\widehat{\boldsymbol{\theta}}_{i-\tau_\beta} \to \boldsymbol{S}_{i-\tau_\beta} \to \boldsymbol{S}_i^{\pi} \to O_i$ forms a Markov chain
    \item 
$|\Lambda(\boldsymbol{\theta}, O_i)| \leq 2G^2 ~~ \text{for any} ~~\|\boldsymbol{\theta}\|_2 \leq r$
\end{itemize}
Lemma 10 of \citep{bhandari2018finite} yields
\begin{equation}\label{eq:lemma10_result}
\left|\mathbb{E}^{\mu^\pi}\{\Lambda(\widehat{\boldsymbol{\theta}}_{i-\tau_\beta}, O_i)\}\right| \leq 4G^2 m \rho^{\tau_\beta}. 
\end{equation}
Combining \eqref{eq:lambda_recursion_tau} and \eqref{eq:lemma10_result}, we get
\begin{equation*}
\mathbb{E}^{\mu^\pi}\Lambda_i(\widehat{\boldsymbol{\theta}}_i) \leq 4G^2 m \rho^{\tau_\beta} + 6\beta G^2 \tau_\beta \leq 4\beta G^2 + 6\beta G^2 \tau_\beta \leq 10\beta G^2 \tau_\beta.
\end{equation*}
\end{proof}

\begin{lemma}\label{lem:B6}
Suppose the step-size sequence $\{\beta_i\}_{i \in \mathbb{N}}$ is decreasing, i.e., $\beta_{i+1} < \beta_i, ~\forall i \in \mathbb{N}$. Fix $i \le t$, and let $\tau_{\beta_t} := \min\{n \in \mathbb{N} : m \rho^n \leq \beta_t\}$, then
\begin{align*}
\mathbb{E}^{\mu^\pi}\left\{\Lambda_i(\widehat{\boldsymbol{\theta}}_i)\right\} &\leq 6G^2 \tau_{\beta_t} \beta_0 \quad \text{for} \quad i \leq \tau_{\beta_t} \\
\mathbb{E}^{\mu^\pi}\left\{\Lambda_i(\widehat{\boldsymbol{\theta}}_i)\right\} &\leq 10 G^2 \tau_{\beta_t} \beta_{i-\tau_{\beta_t}} \quad \text{for} \quad i > \tau_{\beta_t}
\end{align*}
for projected Q-learning iterates $(\widehat{\boldsymbol{\theta}}_i)_{i \in \mathbb{N}}$.
\end{lemma}

\begin{proof}
For all $i \geq 0$,
\begin{align*}
\left\|\widehat{\boldsymbol{\theta}}_{i+1} - \widehat{\boldsymbol{\theta}}_i\right\|_2 &= \left\|\Pi_{r}\{\widehat{\boldsymbol{\theta}}_i + \beta_i f_i(\widehat{\boldsymbol{\theta}}_i)\} - \Pi_{r}\{\widehat{\boldsymbol{\theta}}_i\}\right\|_2 \\
&\leq \left\|\widehat{\boldsymbol{\theta}}_i + \beta_i f_i(\widehat{\boldsymbol{\theta}}_i) - \widehat{\boldsymbol{\theta}}_i\right\|_2 && \text{(non-expansiveness of $\Pi_r$)}\\
&\leq \beta_i G. && \text{(by Lemma \ref{lem:B1})}
\end{align*}
When $i > \tau_{\beta_t}$,
\begin{align*}
|\Lambda_i(\widehat{\boldsymbol{\theta}}_i) - \Lambda_i(\widehat{\boldsymbol{\theta}}_{i-\tau_{\beta_t}})| &\leq 6G \left\|\widehat{\boldsymbol{\theta}}_i - \widehat{\boldsymbol{\theta}}_{i-\tau_{\beta_t}}\right\|_2 && \text{(by Lemma \ref{lem:B4})} \\ &\leq 6G\left(\left\|\widehat{\boldsymbol{\theta}}_i - \widehat{\boldsymbol{\theta}}_{i-1}\right\|_2 + \cdots + \left\|\widehat{\boldsymbol{\theta}}_{i-\tau_{\beta_t}+1} - \widehat{\boldsymbol{\theta}}_{i-\tau_{\beta_t}}\right\|_2\right) \\
&\leq 6G(\beta_{i-1} G + \cdots + \beta_{i-\tau_{\beta_t}} G) \\
&\leq 6 G^2 \sum_{j=i-\tau_{\beta_t}}^{i-1} \beta_j.
\end{align*}
Therefore,
\begin{equation}\label{eq:lambda_recursion_decreasing}
\Lambda_i(\widehat{\boldsymbol{\theta}}_i) \leq \Lambda_i(\widehat{\boldsymbol{\theta}}_{i-\tau_{\beta_t}}) + 6 G^2 \sum_{j=i-\tau_{\beta_t}}^{i-1} \beta_j \le  \Lambda_i(\widehat{\boldsymbol{\theta}}_{i-\tau_{\beta_t}}) + 6 G^2 \tau_{\beta_t}\beta_{i-\tau_{\beta_t}} .
\end{equation}
When $i \leq \tau_{\beta_t}$, by the same logic, 
$$
\Lambda_i(\widehat{\boldsymbol{\theta}}_i) \leq \Lambda_i(\widehat{\boldsymbol{\theta}}_{0}) + 6 G^2 \sum_{j=0}^{i-1} \beta_j \le \Lambda_i(\widehat{\boldsymbol{\theta}}_{0}) + 6 G^2 \tau_{\beta_t} \beta_0.
$$
Since $\mathbb{E}^{\mu^\pi}\Lambda_i(\boldsymbol{\theta}) = 0$ for any fixed $\boldsymbol{\theta} \in \mathbb{R}^M$, we know $\mathbb{E}^{\mu^\pi}\Lambda_i(\widehat{\boldsymbol{\theta}}_{0}) = 0$. For $i \leq \tau_{\beta_t}$, we thus obtain $\mathbb{E}^{\mu^\pi}\Lambda_i(\widehat{\boldsymbol{\theta}}_i) \leq 6G^2 \tau_{\beta_t} \beta_0 $. When $i > \tau_{\beta_t}$, using the same argument as in Lemma \ref{lem:B5}, we know $\mathbb{E}^{\mu^\pi}\Lambda_i(\widehat{\boldsymbol{\theta}}_{i-\tau_{\beta_t}}) \le 4G^2 \beta_t$. Combined with \eqref{eq:lambda_recursion_decreasing}, we have
\begin{align*}
\mathbb{E}^{\mu^\pi} \Lambda_i(\widehat{\boldsymbol{\theta}}_i) &\leq 4 G^2 \beta_t + 6 G^2  \tau_{\beta_t} \beta_{i - \tau_{\beta_t}} \le 10G^2\tau_{\beta_t}\beta_{i - \tau_{\beta_t}}.
\end{align*}
\end{proof}

\begin{lemma}\label{lem:B7}
Suppose the step-size sequence $\{\beta_i\}_{i \in \mathbb{N}}$ is constant, i.e., $\beta_i = \beta, ~\forall i \in \mathbb{N}$. Let $\tau_\beta := \min\{n \in \mathbb{N} : m \rho^n \leq \beta\}$, then for all $i \in \mathbb{N}$
\begin{equation}
\mathbb{E}^{\mu^\pi}\left\{\Lambda_i(\widehat{\boldsymbol{\theta}}^{\text{im}}_i)\right\} \leq 4G^2 m \rho^{\tau_\beta} + 6\beta G^2 \tau_\beta \leq 10\beta G^2 \tau_\beta
\end{equation}
for projected implicit Q-learning iterates $(\widehat{\boldsymbol{\theta}}^{\text{im}}_i)_{i \in \mathbb{N}}$.
\end{lemma}
\begin{proof}
For all $i \geq 0$, let $\tilde{\beta}_i := \beta/(1+\beta\|\boldsymbol{\phi}(\boldsymbol{S}_i^{\pi},A_i^{\pi})\|_2)$. Observe that
\begin{align*}
\left\|\widehat{\boldsymbol{\theta}}^{\text{im}}_{i+1} - \widehat{\boldsymbol{\theta}}^{\text{im}}_i\right\|_2 &= \left\|\Pi_{r}\{\widehat{\boldsymbol{\theta}}^{\text{im}}_i + \tilde\beta_i f_i(\widehat{\boldsymbol{\theta}}^{\text{im}}_i)\} - \Pi_{r}\widehat{\boldsymbol{\theta}}^{\text{im}}_i\right\|_2  \\
&\leq \left\|\widehat{\boldsymbol{\theta}}^{\text{im}}_i + \tilde\beta_i f_i(\widehat{\boldsymbol{\theta}}^{\text{im}}_i) - \widehat{\boldsymbol{\theta}}^{\text{im}}_i\right\|_2 && \text{(non-expansiveness of $\Pi_r$)}\\
&\leq \tilde\beta_i G && \text{(by Lemma \ref{lem:B1})} \\
&\leq \beta G && \text{since $\tilde\beta_i \le \beta$}.
\end{align*}
The remainder of the proof follows the same arguments as those used in the proof of Lemma \ref{lem:B5}.
\end{proof}

\begin{lemma}\label{lem:B8}
Suppose the step-size sequence $\{\beta_i\}_{i \in \mathbb{N}}$ is decreasing, i.e., $\beta_{i+1} < \beta_i, ~\forall i \in \mathbb{N}$. Fix $i \le t$, and let $\tau_{\beta_t} := \min\{n \in \mathbb{N} : m \rho^n \leq \beta_t\}$, then
\begin{align*}
\mathbb{E}^{\mu^\pi}\left\{\Lambda_i(\widehat{\boldsymbol{\theta}}^{\text{im}}_i)\right\} &\leq 6G^2 \tau_{\beta_t} \beta_0 \quad \text{for} \quad i \leq \tau_{\beta_t} \\
\mathbb{E}^{\mu^\pi}\left\{\Lambda_i(\widehat{\boldsymbol{\theta}}^{\text{im}}_i)\right\} &\leq 10 G^2 \tau_{\beta_t} \beta_{i-\tau_{\beta_t}} \quad \text{for} \quad i > \tau_{\beta_t}
\end{align*}
for projected implicit Q-learning iterates $(\widehat{\boldsymbol{\theta}}^{\text{im}}_i)_{i \in \mathbb{N}}$.
\end{lemma}

\begin{proof}
For all $i \geq 0$, let $\tilde{\beta}_i := \beta_i/(1+\beta_i\|\boldsymbol{\phi}(\boldsymbol{S}_i^{\pi},A_i^{\pi})\|_2)$. Observe that
\begin{align*}
\left\|\widehat{\boldsymbol{\theta}}^{\text{im}}_{i+1} - \widehat{\boldsymbol{\theta}}^{\text{im}}_i\right\|_2 &= \left\|\Pi_{r}\{\widehat{\boldsymbol{\theta}}^{\text{im}}_i + \tilde\beta_i f_i(\widehat{\boldsymbol{\theta}}^{\text{im}}_i)\} - \Pi_{r}\widehat{\boldsymbol{\theta}}^{\text{im}}_i\right\|_2  \\
&\leq \left\|\widehat{\boldsymbol{\theta}}^{\text{im}}_i + \tilde\beta_i f_i(\widehat{\boldsymbol{\theta}}^{\text{im}}_i) - \widehat{\boldsymbol{\theta}}^{\text{im}}_i\right\|_2 && \text{(non-expansiveness of $\Pi_r$)}\\
&\leq \tilde\beta_i G && \text{(by Lemma \ref{lem:B1})} \\
&\leq \beta_i G && \text{since $\tilde\beta_i \le \beta_i$}.
\end{align*}
The remainder of the proof follows the same arguments as those used in the proof of Lemma \ref{lem:B6}.
\end{proof}

\subsection{Proof of Theorem \ref{thm:constant_qlearning}}
\noindent \textbf{Standard Q-learning:} we first obtain the finite-time error bound of projected Q-learning iterates with a constant step-size sequence, i.e., $\beta_i = \beta, ~\forall i \in \mathbb{N}$. Consider
\begin{align*}
\mathbb{E}^{\mu^\pi}\left\|\widehat{\boldsymbol{\theta}}_{t+1} - \boldsymbol{\theta}^{\star}\right\|_2^2 &= \mathbb{E}^{\mu^\pi}\left\|\Pi_r\{\widehat{\boldsymbol{\theta}}_{t} + \beta f_t(\widehat{\boldsymbol{\theta}}_{t})\} - \Pi_r\boldsymbol{\theta}^{\star}\right\|_2^2 \\
&\leq \mathbb{E}^{\mu^\pi}\left\|\widehat{\boldsymbol{\theta}}_{t} + \beta f_t(\widehat{\boldsymbol{\theta}}_{t}) - \boldsymbol{\theta}^{\star}\right\|_2^2\\
&= \mathbb{E}^{\mu^\pi}\left\{\left\|\widehat{\boldsymbol{\theta}}_{t} - \boldsymbol{\theta}^{\star}\right\|_2^2 + \beta^2\left\|f_t(\widehat{\boldsymbol{\theta}}_{t})\right\|_2^2 + 2\beta\langle\widehat{\boldsymbol{\theta}}_{t} - \boldsymbol{\theta}^{\star}, f_t(\widehat{\boldsymbol{\theta}}_{t})\rangle\right\} \\
&= \mathbb{E}^{\mu^\pi}\left\{\left\|\widehat{\boldsymbol{\theta}}_{t} - \boldsymbol{\theta}^{\star}\right\|_2^2 + \beta^2\left\|f_t(\widehat{\boldsymbol{\theta}}_{t})\right\|_2^2 + 2\beta\langle\widehat{\boldsymbol{\theta}}_{t} - \boldsymbol{\theta}^{\star}, \bar{f}(\widehat{\boldsymbol{\theta}}_{t}) - \bar{f}(\boldsymbol{\theta}^{\star})\rangle + 2\beta\Lambda_t(\widehat{\boldsymbol{\theta}}_{t})\right\}
\end{align*}
where the inequality is thanks to the non-expansiveness of $\Pi_r$, and the last equality is due to $\bar{f}(\boldsymbol{\theta}^{\star}) = 0$. From Lemma \ref{lem:B1}  and Lemma \ref{lem:B2}, we obtain
\begin{align*}
\mathbb{E}^{\mu^\pi}\left\|\widehat{\boldsymbol{\theta}}_{t+1} - \boldsymbol{\theta}^{\star}\right\|_2^2 &\leq (1 - \beta w_q)\mathbb{E}^{\mu^\pi}\left\|\widehat{\boldsymbol{\theta}}_{t} - \boldsymbol{\theta}^{\star}\right\|_2^2 + \beta^2 G^2 + 2\beta\mathbb{E}^{\mu^\pi}\Lambda_t(\widehat{\boldsymbol{\theta}}_{t}) \\
&\leq (1 - \beta w_q)\mathbb{E}^{\mu^\pi}\left\|\widehat{\boldsymbol{\theta}}_{t} - \boldsymbol{\theta}^{\star}\right\|_2^2 + \beta^2 G^2 + 20\beta^2G^2\tau_\beta
\end{align*}
where the second inequality is due to Lemma \ref{lem:B5}. Recursively applying the inequality with the assumption $\beta \in (0, 1/w_q)$, we have
\begin{align*}
\mathbb{E}^{\mu^\pi}\left\|\widehat{\boldsymbol{\theta}}_{t+1} - \boldsymbol{\theta}^{\star}\right\|_2^2 &\le (1 - \beta w_q)^{t+1}\left\|\widehat{\boldsymbol{\theta}}_0- \boldsymbol{\theta}^{\star}\right\|_2^2 + \beta^2 G^2 \sum_{i=0}^\infty (1 - \beta w_q)^i + 20\beta^2G^2\tau_\beta \sum_{i=0}^\infty (1 - \beta w_q)^i  \\
&= (1 - \beta w_q)^{t+1}\left\|\widehat{\boldsymbol{\theta}}_0- \boldsymbol{\theta}^{\star}\right\|_2^2 + \mathcal{O}\left( \beta\tau_\beta\right)
\end{align*}
\textbf{Implicit Q-learning:} we next obtain the finite-time error bound of the projected implicit Q-learning iterates with a constant step-size sequence, i.e., $\beta_i = \beta, ~\forall i \in \mathbb{N}$. Recall that $\tilde{\beta}_i := \beta/(1+\beta\|\boldsymbol{\phi}(\boldsymbol{S}_i^{\pi},A_i^{\pi})\|^2_2)$. Consider
\begin{align*}
\mathbb{E}^{\mu^\pi}\left\|\widehat{\boldsymbol{\theta}}^{\text{im}}_{t+1} - \boldsymbol{\theta}^{\star}\right\|_2^2 &= \mathbb{E}^{\mu^\pi}\left\|\Pi_r\{\widehat{\boldsymbol{\theta}}^{\text{im}}_{t} + \tilde\beta_t f_t(\widehat{\boldsymbol{\theta}}_{t})\} - \Pi_r\boldsymbol{\theta}^{\star}\right\|_2^2 \\
&\leq \mathbb{E}^{\mu^\pi}\left\|\widehat{\boldsymbol{\theta}}^{\text{im}}_{t} + \tilde\beta_t f_t(\widehat{\boldsymbol{\theta}}^{\text{im}}_{t}) - \boldsymbol{\theta}^{\star}\right\|_2^2\\
&\leq \mathbb{E}^{\mu^\pi}\left\{\left\|\widehat{\boldsymbol{\theta}}^{\text{im}}_{t} - \boldsymbol{\theta}^{\star}\right\|_2^2 + \beta^2\left\|f_t(\widehat{\boldsymbol{\theta}}^{\text{im}}_{t})\right\|_2^2 + 2\tilde\beta_t\langle\widehat{\boldsymbol{\theta}}^{\text{im}}_{t} - \boldsymbol{\theta}^{\star}, f_t(\widehat{\boldsymbol{\theta}}^{\text{im}}_{t})\rangle\right\} \\
&= \mathbb{E}^{\mu^\pi}\left\{\left\|\widehat{\boldsymbol{\theta}}^{\text{im}}_{t} - \boldsymbol{\theta}^{\star}\right\|_2^2 + \beta^2\left\|f_t(\widehat{\boldsymbol{\theta}}^{\text{im}}_{t})\right\|_2^2 + 2\tilde\beta_t\langle\widehat{\boldsymbol{\theta}}^{\text{im}}_{t} - \boldsymbol{\theta}^{\star}, \bar{f}(\widehat{\boldsymbol{\theta}}^{\text{im}}_{t}) - \bar{f}(\boldsymbol{\theta}^{\star})\rangle + 2\tilde\beta_t\Lambda_t(\widehat{\boldsymbol{\theta}}^{\text{im}}_{t})\right\}
\end{align*}
where the first inequality is thanks to the non-expansiveness of $\Pi_r$, and the second inequality is due to the bound $\tilde \beta_t \le \beta$, which holds almost surely. In the last equality, we used $\bar{f}(\boldsymbol{\theta}^{\star}) = 0$. From Lemma \ref{lem:B1} and Lemma \ref{lem:B2}, we obtain
\begin{align*}
\mathbb{E}^{\mu^\pi}\left\|\widehat{\boldsymbol{\theta}}^{\text{im}}_{t+1} - \boldsymbol{\theta}^{\star}\right\|_2^2 &\leq \mathbb{E}^{\mu^\pi}\left\{(1 - \tilde \beta_t w_q)\left\|\widehat{\boldsymbol{\theta}}^{\text{im}}_{t} - \boldsymbol{\theta}^{\star}\right\|_2^2 \right\}+ \beta^2 G^2 + 2\beta\mathbb{E}^{\mu^\pi}\Lambda_t(\widehat{\boldsymbol{\theta}}^{\text{im}}_{t}) + 2\mathbb{E}^{\mu^\pi}\left\{(\tilde\beta_t-\beta)\Lambda_t(\widehat{\boldsymbol{\theta}}^{\text{im}}_{t})\right\}\\
&\leq \left(1 - \frac{\beta w_q}{1+\beta}\right)\mathbb{E}^{\mu^\pi}\left\|\widehat{\boldsymbol{\theta}}^{\text{im}}_{t} - \boldsymbol{\theta}^{\star}\right\|_2^2 + \beta^2 G^2 + 20\beta^2 G^2\tau_\beta + 2\mathbb{E}^{\mu^\pi}\left\{(\tilde\beta_t-\beta)\Lambda_t(\widehat{\boldsymbol{\theta}}^{\text{im}}_{t})\right\}\\
&\leq \left(1 - \frac{\beta w_q}{1+\beta}\right)\mathbb{E}^{\mu^\pi}\left\|\widehat{\boldsymbol{\theta}}^{\text{im}}_{t} - \boldsymbol{\theta}^{\star}\right\|_2^2 + \beta^2 G^2 + 20\beta^2 G^2\tau_\beta + 4\beta^2 G^2 .
\end{align*}
where in the second inequality, we used Lemma \ref{lem:B7} with the bound $\tilde\beta_t \ge \frac{\beta}{1+\beta}$, which holds almost surely. In the third inequality, we used the fact $|\tilde\beta_t-\beta| \le \beta^2$ with Lemma \ref{lem:B3}. Recursively applying the inequality with the assumption $\frac{\beta w_q}{1+\beta} \in (0, 1)$, we have
\begin{align*}
\mathbb{E}^{\mu^\pi}\left\|\widehat{\boldsymbol{\theta}}^{\text{im}}_{t+1} - \boldsymbol{\theta}^{\star}\right\|_2^2 &\le \left(1 - \frac{\beta w_q}{1+\beta}\right)^{t+1}\left\|\widehat{\boldsymbol{\theta}}^{\text{im}}_0- \boldsymbol{\theta}^{\star}\right\|_2^2 + 5\beta^2 G^2 \sum_{i=0}^\infty \left(1 - \frac{\beta w_q}{1+\beta}\right)^i + 20\beta^2 G^2\tau_\beta \sum_{i=0}^\infty \left(1 - \frac{\beta w_q}{1+\beta}\right)^i  \\
&= \left(1 - \frac{\beta w_q}{1+\beta}\right)^{t+1}\left\|\widehat{\boldsymbol{\theta}}^{\text{im}}_0- \boldsymbol{\theta}^{\star}\right\|_2^2 + \mathcal{O}\left( \beta\tau_\beta + \beta^2\tau_\beta\right).
\end{align*}

\subsection{Proof of Theorem \ref{thm:decreasing_qlearning}}
\noindent \textbf{Standard Q-Learning:} we first obtain the finite-time error bound of the projected Q-learning iterates with a step-size sequence of the form $\beta_i = \frac{\beta_0}{(i+1)^s}, ~s \in (0,1), ~i \in \mathbb{N}$. Consider
\begin{align*}
\mathbb{E}^{\mu^\pi}\left\|\widehat{\boldsymbol{\theta}}_{t+1} - \boldsymbol{\theta}^{\star}\right\|_2^2 &= \mathbb{E}^{\mu^\pi}\left\|\Pi_r\{\widehat{\boldsymbol{\theta}}_{t} + \beta_t f_t(\widehat{\boldsymbol{\theta}}_{t})\} - \Pi_r\boldsymbol{\theta}^{\star}\right\|_2^2 \\
&\leq \mathbb{E}^{\mu^\pi}\left\|\widehat{\boldsymbol{\theta}}_{t} + \beta_t f_t(\widehat{\boldsymbol{\theta}}_{t}) - \boldsymbol{\theta}^{\star}\right\|_2^2\\
&= \mathbb{E}^{\mu^\pi}\left\{\left\|\widehat{\boldsymbol{\theta}}_{t} - \boldsymbol{\theta}^{\star}\right\|_2^2 + \beta_t^2\left\|f_t(\widehat{\boldsymbol{\theta}}_{t})\right\|_2^2 + 2\beta_t\langle\widehat{\boldsymbol{\theta}}_{t} - \boldsymbol{\theta}^{\star}, f_t(\widehat{\boldsymbol{\theta}}_{t})\rangle\right\} \\
&= \mathbb{E}^{\mu^\pi}\left\{\left\|\widehat{\boldsymbol{\theta}}_{t} - \boldsymbol{\theta}^{\star}\right\|_2^2 + \beta_t^2\left\|f_t(\widehat{\boldsymbol{\theta}}_{t})\right\|_2^2 + 2\beta_t\langle\widehat{\boldsymbol{\theta}}_{t} - \boldsymbol{\theta}^{\star}, \bar{f}(\widehat{\boldsymbol{\theta}}_{t}) - \bar{f}(\boldsymbol{\theta}^{\star})\rangle + 2\beta_t\Lambda_t(\widehat{\boldsymbol{\theta}}_{t})\right\}
\end{align*}
where the inequality is thanks to the non-expansiveness of $\Pi_r$, and the last equality is due to $\bar{f}(\boldsymbol{\theta}^{\star}) = 0$. From Lemma \ref{lem:B1}  and Lemma \ref{lem:B2}, we obtain
\begin{align*}
\mathbb{E}^{\mu^\pi}\left\|\widehat{\boldsymbol{\theta}}_{t+1} - \boldsymbol{\theta}^{\star}\right\|_2^2 &\leq (1 - \beta_t w_q)\mathbb{E}^{\mu^\pi}\left\|\widehat{\boldsymbol{\theta}}_{t} - \boldsymbol{\theta}^{\star}\right\|_2^2 + \beta_t^2 G^2 + 2\beta_t\mathbb{E}^{\mu^\pi}\Lambda_t(\widehat{\boldsymbol{\theta}}_{t}).
\end{align*}
Recursively applying the inequality with the assumption $\beta_0 \in (0, 1/w_q)$, we have
\begin{align*}
\mathbb{E}^{\mu^\pi} \left\|\widehat{\boldsymbol{\theta}}_{t+1} - \boldsymbol{\theta}^{\star}\right\|_2^2 
&\leq \left\{\prod_{i=0}^{t} (1 - \beta_i w_q)\right\} \left\|\widehat{\boldsymbol{\theta}}_{0} - \boldsymbol{\theta}^{\star}\right\|_2^2 + G^2 \sum_{i=0}^{t} \left\{\prod_{k=i+1}^{t} (1 - \beta_k w_q)\right\} \beta_i^2 \\ &\quad + 2\sum_{i=0}^{t} \left\{\prod_{k=i+1}^{t} (1 - \beta_k w_q)\right\} \mathbb{E}^{\mu^\pi}\{\Lambda_i(\widehat{\boldsymbol{\theta}}_i)\} \beta_i \\
&\leq \left\{\prod_{i=0}^{t} (1 - \beta_iw_q) \right\}\left\|\widehat{\boldsymbol{\theta}}_{0} - \boldsymbol{\theta}^{\star}\right\|_2^2 + G^2 \sum_{i=0}^{t} \left\{\prod_{k=i+1}^{t} (1 - \beta_k w_q)\right\} \beta_i^2 \\
&\quad + 12G^2 \tau_{\beta_t} \beta_0\sum_{i=0}^{\tau_{\beta_t}} \left\{\prod_{k=i+1}^{t} (1 - \beta_k w_q)\right\}  \beta_i + 20G^2 \tau_{\beta_t}\sum_{i=\tau_{\beta_t}+1}^{t} \left\{\prod_{k=i+1}^{t} (1 - \beta_k w_q)\right\}  \beta_{i-\tau_{\beta_t}} \beta_i
\end{align*}
where the second inequality follows from Lemma \ref{lem:B6}. Thanks to $1-\beta_i w_q \le \exp(-\beta_i w_q)$, we have
\begin{align}
\mathbb{E}^{\mu^\pi} \left\|\widehat{\boldsymbol{\theta}}_{t+1} - \boldsymbol{\theta}^{\star}\right\|_2^2 &\leq \exp\left(-w_q \sum_{i=0}^{t} \beta_i\right) \left\|\widehat{\boldsymbol{\theta}}_{0} - \boldsymbol{\theta}^{\star}\right\|_2^2 \quad \label{Q_decr_first} \\
&\quad + G^2 \sum_{i=0}^{t} \exp\left(-w_q \sum_{k=i+1}^{t} \beta_k\right) \beta_i^2 \quad  \label{Q_decr_second}\\
&\quad + 12 G^2 \tau_{\beta_t} \beta_0 \sum_{i=0}^{\tau_{\beta_t}} \exp\left(-w_q \sum_{k=i+1}^{t} \beta_k\right) \beta_i \quad  \label{Q_decr_third} \\
&\quad + 20 G^2 \tau_{\beta_t} \sum_{i=\tau_{\beta_t}+1}^{t} \exp\left(-w_q \sum_{k=i+1}^{t} \beta_k\right) \beta_{i-\tau_{\beta_t}} \beta_i. \quad \label{Q_decr_fourth}
\end{align}
We now provide a bound for each of the term above. For the first term, we have
\begin{align}
\eqref{Q_decr_first} &= \exp\left\{-w_q\beta_0 \sum_{i=0}^{t} \frac{1}{(i+1)^s}\right\} \left\|\widehat{\boldsymbol{\theta}}_{0} - \boldsymbol{\theta}^{\star}\right\|_2^2 \nonumber\\
&\leq \exp\left[\frac{-w_q\beta_0}{(1-s)} \left\{(t+1)^{1-s} - 1\right\}\right] \left\|\widehat{\boldsymbol{\theta}}_{0} - \boldsymbol{\theta}^{\star}\right\|_2^2 \label{q_const_first_bound}
\end{align}
For the second term, we have
\begin{align}
\eqref{Q_decr_second}  &\le 2G^2 (K_b e^{-\frac{w_q}{2} \sum_{k=0}^{t} \beta_k} + \beta_t) \frac{\exp\left(w_q\beta_0/2\right)}{w_q} && \text{(by Lemma \ref{lem:exp_bound1})} \nonumber\\
&\lesssim \exp\left[\frac{-w_q\beta_0}{2(1-s)} \{(t+1)^{1-s} - 1\}\right] + \beta_t\label{q_const_second_bound}
\end{align}
For the third term, we get
\begin{align}
\eqref{Q_decr_third} 
&\leq 12 G^2 \tau_{\beta_t} \beta_0 \frac{\exp(w_q\beta_0)}{w_q} \exp\left[\frac{-w_q\beta_0}{(1-s)} \{(t+1)^{1-s} - (\tau_{\beta_t}+1)^{1-s}\}\right]&&\text{(by Lemma \ref{lem:exp_bound2})}\nonumber\\
&\lesssim \tau_{\beta_t}  \exp\left[\frac{-w_q\beta_0}{(1-s)} \{(t+1)^{1-s} - (\tau_{\beta_t}+1)^{1-s}\}\right]\label{q_const_third_bound}
\end{align}
For the last term, we yield
\begin{align}
\eqref{Q_decr_fourth}
&\leq 20 G^2 \tau_{\beta_t} \left[\exp\left\{\frac{-w_q\beta_0}{2(1-s)} \{(t+1)^{1-s} - 1\}\right\} D_{\beta_t} \mathbb{I}_{\{\tau_{\beta_t}+1 \leq 1\}} + \beta_{t-\tau_{\beta_t}}\right] \frac{2\exp(w_q\beta_0/2)}{w_q} && \text{(by Lemma \ref{lem:exp_bound2})} \nonumber\\
&\lesssim  \tau_{\beta_t} \left[\exp\left\{\frac{-w_q\beta_0}{2(1-s)} \{(t+1)^{1-s} - 1\}\right\} + \beta_{t-\tau_{\beta_t}}\right] \label{q_const_fourth_bound}
\end{align}
Combining \eqref{q_const_first_bound}-\eqref{q_const_fourth_bound} we get,
\begin{align*}
\mathbb{E}^{\mu^\pi} \left\|\widehat{\boldsymbol{\theta}}_{t+1} - \boldsymbol{\theta}^{\star}\right\|_2^2 &\leq \exp\left[\frac{-w_q\beta_0}{(1-s)} \{(t+1)^{1-s} - 1\}\right] \left\|\widehat{\boldsymbol{\theta}}_{0} - \boldsymbol{\theta}^{\star}\right\|_2^2 + \mathcal{O}\left(\tau_{\beta_t} \beta_{t-\tau_{\beta_t}} \right). 
\end{align*}

\noindent \textbf{Implicit Q-Learning:} we next obtain the finite-time error bound of the projected implicit Q-learning iterates with a step-size sequence of the form $\beta_i = \frac{\beta_0}{(i+1)^s}, ~s \in (0,1), ~i \in \mathbb{N}$. With a slight abuse of notation, let $\tilde \beta_i := \frac{\beta_i}{1+\beta_i \|\boldsymbol{\phi}(\boldsymbol{S}_i^{\pi},A_i^{\pi})\|^2_2}$. Consider

\begin{align*}
\mathbb{E}^{\mu^\pi}\left\|\widehat{\boldsymbol{\theta}}^{\text{im}}_{t+1} - \boldsymbol{\theta}^{\star}\right\|_2^2 &= \mathbb{E}^{\mu^\pi}\left\|\Pi_r\{\widehat{\boldsymbol{\theta}}^{\text{im}}_{t} + \tilde\beta_t f_t(\widehat{\boldsymbol{\theta}}^{\text{im}}_{t})\} - \Pi_r\boldsymbol{\theta}^{\star}\right\|_2^2 \\
&\leq \mathbb{E}^{\mu^\pi}\left\|\widehat{\boldsymbol{\theta}}^{\text{im}}_{t} + \tilde\beta_t f_t(\widehat{\boldsymbol{\theta}}^{\text{im}}_{t}) - \boldsymbol{\theta}^{\star}\right\|_2^2\\
&\le \mathbb{E}^{\mu^\pi}\left\{\left\|\widehat{\boldsymbol{\theta}}^{\text{im}}_{t} - \boldsymbol{\theta}^{\star}\right\|_2^2 + \beta_t^2\left\|f_t(\widehat{\boldsymbol{\theta}}^{\text{im}}_{t})\right\|_2^2 + 2\tilde\beta_t\langle\widehat{\boldsymbol{\theta}}^{\text{im}}_{t} - \boldsymbol{\theta}^{\star}, f_t(\widehat{\boldsymbol{\theta}}^{\text{im}}_{t})\rangle\right\} \\
&= \mathbb{E}^{\mu^\pi}\left\{\left\|\widehat{\boldsymbol{\theta}}^{\text{im}}_{t} - \boldsymbol{\theta}^{\star}\right\|_2^2 + \beta_t^2\left\|f_t(\widehat{\boldsymbol{\theta}}^{\text{im}}_{t})\right\|_2^2 + 2\tilde\beta_t\langle\widehat{\boldsymbol{\theta}}^{\text{im}}_{t} - \boldsymbol{\theta}^{\star}, \bar{f}(\widehat{\boldsymbol{\theta}}^{\text{im}}_{t}) - \bar{f}(\boldsymbol{\theta}^{\star})\rangle + 2\tilde\beta_t\Lambda_t(\widehat{\boldsymbol{\theta}}^{\text{im}}_{t})\right\}
\end{align*}
where the first inequality is thanks to the non-expansiveness of $\Pi_r$, and the second inequality is due to $\tilde \beta_t \le \beta_t$. In the last inequality, we used $\bar{f}(\boldsymbol{\theta}^{\star}) = 0$. From Lemma \ref{lem:B1}  and Lemma \ref{lem:B2}, we obtain
\begin{align*}
\mathbb{E}^{\mu^\pi}\left\|\widehat{\boldsymbol{\theta}}^{\text{im}}_{t+1} - \boldsymbol{\theta}^{\star}\right\|_2^2 &\leq \mathbb{E}^{\mu^\pi}\left\{(1 - \tilde \beta_t w_q)\left\|\widehat{\boldsymbol{\theta}}^{\text{im}}_{t} - \boldsymbol{\theta}^{\star}\right\|_2^2 \right\}+ \beta_t^2 G^2 + 2\beta_t\mathbb{E}^{\mu^\pi}\Lambda_t(\widehat{\boldsymbol{\theta}}^{\text{im}}_{t}) + 2\mathbb{E}^{\mu^\pi}\left\{(\tilde\beta_t-\beta_t)\Lambda_t(\widehat{\boldsymbol{\theta}}^{\text{im}}_{t})\right\}\\
&\leq \left(1 - \frac{\beta_t w_q}{1+\beta_0}\right)\mathbb{E}^{\mu^\pi}\left\|\widehat{\boldsymbol{\theta}}^{\text{im}}_{t} - \boldsymbol{\theta}^{\star}\right\|_2^2 + \beta_t^2 G^2 + 2\beta_t\mathbb{E}^{\mu^\pi}\Lambda_t(\widehat{\boldsymbol{\theta}}^{\text{im}}_{t}) + 4\beta_t^2 G^2 .
\end{align*}
where in the second inequality, we used the bound $\tilde\beta_t \ge \frac{\beta_t}{1+\beta_0}$, which holds almost surely as well as the fact $|\tilde\beta_t-\beta_t| \le \beta_t^2$ with Lemma \ref{lem:B3}. 
Recursively applying the inequality with the assumption $\beta_0 w_q/(1+\beta_0)\in (0, 1)$, we have
\begin{align*}
\mathbb{E}^{\mu^\pi} \left\|\widehat{\boldsymbol{\theta}}^{\text{im}}_{t+1} - \boldsymbol{\theta}^{\star}\right\|_2^2 
&\leq \left\{\prod_{i=0}^{t} \left(1 - \frac{\beta_i w_q}{1+\beta_0}\right)\right\} \left\|\widehat{\boldsymbol{\theta}}_{0}- \boldsymbol{\theta}^{\star}\right\|_2^2 + 5G^2 \sum_{i=0}^{t} \left\{\prod_{k=i+1}^{t} \left(1 - \frac{\beta_k w_q}{1+\beta_0}\right)\right\} \beta_i^2 \\ &\quad + 2\sum_{i=0}^{t} \left\{\prod_{k=i+1}^{t} \left(1 - \frac{\beta_k w_q}{1+\beta_0}\right)\right\} \mathbb{E}^{\mu^\pi}\{\Lambda_i(\widehat{\boldsymbol{\theta}}^{\text{im}}_i)\} \beta_i \\
&\leq \left\{\prod_{i=0}^{t} \left(1 - \frac{\beta_i w_q}{1+\beta_0}\right) \right\}\left\|\widehat{\boldsymbol{\theta}}_{0}- \boldsymbol{\theta}^{\star}\right\|_2^2 + 5G^2 \sum_{i=0}^{t} \left\{\prod_{k=i+1}^{t} \left(1 - \frac{\beta_k w_q}{1+\beta_0}\right)\right\} \beta_i^2 \\
&\quad + 12G^2 \tau_{\beta_t} \beta_0\sum_{i=0}^{\tau_{\beta_t}} \left\{\prod_{k=i+1}^{t} \left(1 - \frac{\beta_k w_q}{1+\beta_0}\right)\right\}  \beta_i + 20G^2 \tau_{\beta_t}\sum_{i=\tau_{\beta_t}+1}^{t} \left\{\prod_{k=i+1}^{t} \left(1 - \frac{\beta_k w_q}{1+\beta_0}\right)\right\}  \beta_{i-\tau_{\beta_t}} \beta_i
\end{align*}
where the second inequality follows from Lemma \ref{lem:B8}. Thanks to $1 - \frac{\beta_i w_q}{1+\beta_0} \le \exp(-\frac{\beta_i w_q}{1+\beta_0})$, we have
\begin{align}
\mathbb{E}^{\mu^\pi} \left\|\widehat{\boldsymbol{\theta}}^{\text{im}}_{t+1} - \boldsymbol{\theta}^{\star}\right\|_2^2 &\leq \exp\left(-\frac{w_q}{1+\beta_0}\sum_{i=0}^{t} \beta_i\right) \left\|\widehat{\boldsymbol{\theta}}_{0}- \boldsymbol{\theta}^{\star}\right\|_2^2 \quad \label{IMP_Q_decr_first} \\
&\quad + 5G^2 \sum_{i=0}^{t} \exp\left(-\frac{w_q}{1+\beta_0} \sum_{k=i+1}^{t} \beta_k\right) \beta_i^2 \quad  \label{IMP_Q_decr_second}\\
&\quad + 12 G^2 \tau_{\beta_t} \beta_0 \sum_{i=0}^{\tau_{\beta_t}} \exp\left(-\frac{w_q}{1+\beta_0} \sum_{k=i+1}^{t} \beta_k\right) \beta_i \quad  \label{IMP_Q_decr_third} \\
&\quad + 20 G^2 \tau_{\beta_t} \sum_{i=\tau_{\beta_t}+1}^{t} \exp\left(-\frac{w_q}{1+\beta_0} \sum_{k=i+1}^{t} \beta_k\right) \beta_{i-\tau_{\beta_t}} \beta_i. \quad \label{IMP_Q_decr_fourth}
\end{align}
We now provide a bound for each of the term above. For the first term, we have
\begin{align}
\eqref{IMP_Q_decr_first} &= \exp\left\{-\frac{w_q \beta_0}{1+\beta_0} \sum_{i=0}^{t} \frac{1}{(i+1)^s}\right\} \left\|\widehat{\boldsymbol{\theta}}_{0}- \boldsymbol{\theta}^{\star}\right\|_2^2 \nonumber\\
&\leq \exp\left[\frac{-w_q\beta_0}{(1+\beta_0)(1-s)} \left\{(t+1)^{1-s} - 1\right\}\right] \left\|\widehat{\boldsymbol{\theta}}_{0}- \boldsymbol{\theta}^{\star}\right\|_2^2 \label{IMP_q_const_first_bound}
\end{align}
For the second term, we have
\begin{align}
\eqref{IMP_Q_decr_second}  &\le 2G^2 \left\{K_b e^{-\frac{w_q}{2(1+\beta_0)} \sum_{k=0}^{t} \beta_k} + \beta_t\right\}\left\{ \frac{e^{\frac{w_q\beta_0}{2(1+\beta_0)}}(1+\beta_0)}{w_q}\right\} && \text{(by Lemma \ref{lem:exp_bound1})} \nonumber\\
&\lesssim \exp\left[\frac{-w_q\beta_0}{2(1+\beta_0)(1-s)} \{(t+1)^{1-s} - 1\}\right] + \beta_t\label{IMP_q_const_second_bound}
\end{align}
For the third term, we get
\begin{align}
\eqref{IMP_Q_decr_third} 
&\leq 12 G^2 \tau_{\beta_t} \beta_0 \frac{\exp\{w_q\beta_0/(1+\beta_0)\}(1+\beta_0)}{w_q} \exp\left[\frac{-w_q\beta_0}{(1+\beta_0)(1-s)} \{(t+1)^{1-s} - (\tau_{\beta_t}+1)^{1-s}\}\right]&&\text{(by Lemma \ref{lem:exp_bound2})}\nonumber\\
&\lesssim \tau_{\beta_t}  \exp\left[\frac{-w_q\beta_0}{(1+\beta_0)(1-s)} \{(t+1)^{1-s} - (\tau_{\beta_t}+1)^{1-s}\}\right]\label{IMP_q_const_third_bound}
\end{align}
For the last term, by Lemma \ref{lem:exp_bound2}, we yield
\begin{align}
\eqref{IMP_Q_decr_fourth}
&\leq 20 G^2 \tau_{\beta_t} \left[\exp\left\{\frac{-w_q\beta_0}{2(1+\beta_0)(1-s)} \{(t+1)^{1-s} - 1\}\right\} D_{\beta_t} \mathbb{I}_{\{\tau_{\beta_t}+1 \leq 1\}} + \beta_{t-\tau_{\beta_t}}\right] \frac{2e^{\frac{w_q\beta_0}{2(1+\beta_0)}}(1+\beta_0)}{w_q} \nonumber\\
&\lesssim  \tau_{\beta_t} \left[\exp\left\{\frac{-w_q\beta_0}{2(1+\beta_0)(1-s)} \{(t+1)^{1-s} - 1\}\right\} + \beta_{t-\tau_{\beta_t}}\right] \label{IMP_q_const_fourth_bound}
\end{align}
Combining \eqref{IMP_q_const_first_bound}-\eqref{IMP_q_const_fourth_bound} we get,
\begin{align*}
\mathbb{E}^{\mu^\pi} \left\|\widehat{\boldsymbol{\theta}}^{\text{im}}_{t+1} - \boldsymbol{\theta}^{\star}\right\|_2^2 &\leq \exp\left[\frac{-w_q\beta_0}{(1+\beta_0)(1-s)} \{(t+1)^{1-s} - 1\}\right] \left\|\widehat{\boldsymbol{\theta}}_{0}- \boldsymbol{\theta}^{\star}\right\|_2^2 + \mathcal{O}\left(\tau_{\beta_t} \beta_{t-\tau_{\beta_t}} \right).
\end{align*}

\section{Proofs of SARSA}
\subsection{Preliminary Lemmas}
\noindent For the SARSA algorithm, for any $\boldsymbol{\theta} \in \mathbb{R}^M$, let us define
\begin{align*}
g_t(\boldsymbol{\theta}) &:= \boldsymbol{\phi}(\boldsymbol{S}_t^{\pi_{t-2}},A_t^{\pi_{t-1}})\left\{R_t + \gamma \boldsymbol{\phi}(\boldsymbol{S}_{t+1}^{\pi_{t-1}},A_{t+1}^{\pi_t})^\top\boldsymbol{\theta} - \boldsymbol{\phi}(\boldsymbol{S}_t^{\pi_{t-2}},A_t^{\pi_{t-1}})^\top\boldsymbol{\theta}\right\}  \\
\bar{g}(\boldsymbol{\theta}) &:= \mathbb{E}^{\mu^{\boldsymbol{\theta}}}\left[\phi(\boldsymbol{S}, A)\left\{r(\boldsymbol{S}, A) + \gamma \phi(\boldsymbol{S}', A')^{\top} \boldsymbol{\theta} - \phi(\boldsymbol{S}, A)^{\top}\boldsymbol{\theta}\right\}\right]. 
\end{align*}
Recall that $\mathbb{E}^{\mu^{\boldsymbol{\theta}}}$ denotes the expectation where the current state $\boldsymbol{S}$ follows the stationary distribution $\mu^{\pi_{\boldsymbol{\theta}}}$, the current action $A$ is chosen according to the behavioral policy $\pi_{\boldsymbol{\theta}}(\cdot | \boldsymbol{S})$, $\boldsymbol{S}'$ denotes the subsequent state transitioned following $\Ptrans(\cdot | \boldsymbol{S}, A)$, and $A'$ is chosen based on the behavioral policy $\pi_{\boldsymbol{\theta}}(\cdot | \boldsymbol{S}')$.
With a slight abuse of notation, we define $\Lambda_t(\boldsymbol{\theta}) := \langle\boldsymbol{\theta} - \boldsymbol{\theta}^{\star}, g_t(\boldsymbol{\theta}) - \bar{g}(\boldsymbol{\theta})\rangle$. Recall that $\bar{g}(\boldsymbol{\theta}^{\star}) = 0$. We now provide the necessary Lemmas to establish finite-time error bounds.

\begin{lemma}[Lemma 2 of \citep{zou2019finite}]\label{lem:C1}
For any $\boldsymbol{\theta}\in \mathbb{R}^M$ such that $\|\boldsymbol{\theta}\|_2 \leq r$, 
$$
\|g_t(\boldsymbol{\theta})\|_2 \leq G:= b + 2r ~~ \text{for all}~~ t \ge 0
$$
\end{lemma}

\begin{lemma}[Lemma 4 of \citep{zou2019finite}]\label{lem:C2}
For any $\boldsymbol{\theta}\in \mathbb{R}^M$ such that $\|\boldsymbol{\theta}\|_2 \leq r$,
$$
\langle\boldsymbol{\theta} - \boldsymbol{\theta}^{\star}, \bar{g}(\boldsymbol{\theta}) - \bar{g}(\boldsymbol{\theta}^{\star})\rangle \leq -w_s\|\boldsymbol{\theta} - \boldsymbol{\theta}^{\star}\|_2^2
$$
where $w_s > 0$ is the smallest eigenvalue of $-\frac{1}{2}\left\{\left(\boldsymbol{A}^{\boldsymbol{\theta}^{\star}} + C\lambda \boldsymbol{I}\right) + \left(\boldsymbol{A}^{\boldsymbol{\theta}^{\star}} + C\lambda \boldsymbol{I}\right)^\top\right\}$. Here, $C > 0$ is the Lipschitz constant of the behavioral policy $\pi_\theta$ and $\lambda := G|\mathcal{A}|\left(2 + \lceil{\log_{\rho} \frac{1}{m}}\rceil + \frac{1}{1-\rho}\right)$.
\end{lemma}

\begin{lemma}[Lemma 5 of \citep{zou2019finite}]\label{lem:C3}
For any $\boldsymbol{\theta}\in \mathbb{R}^M$ such that $\|\boldsymbol{\theta}\|_2 \leq r$, $\Lambda_t(\boldsymbol{\theta}) \leq 2G^2$ for all $t \ge 0$.
\end{lemma}

\begin{lemma}[Lemma 6 of \citep{zou2019finite}]\label{lem:C4} For any $\boldsymbol{\theta}_1, \boldsymbol{\theta}_2 \in \mathbb{R}^M$ such that $\|\boldsymbol{\theta}_1\|_2 \le r,  \|\boldsymbol{\theta}_2\|_2 \le r$,
$$
|\Lambda_t(\boldsymbol{\theta}_1) - \Lambda_t(\boldsymbol{\theta}_2)| \leq (6 + \lambda C)G\|\boldsymbol{\theta}_1 - \boldsymbol{\theta}_2\|_2 ~~\text{for all}~~ t \ge 0.
$$
Here, $C > 0$ is the Lipschitz constant of the behavioral policy $\pi_\theta$ and $\lambda := G|\mathcal{A}|\left(2 + \lceil{\log_{\rho} \frac{1}{m}}\rceil + \frac{1}{1-\rho}\right)$.
\end{lemma}

\begin{lemma}\label{lem:C5}
Given a step-size sequence $(\beta_k)_{k \ge 0}$, 
for any $i \le \tau$,
\begin{equation*}
\mathbb{E}\left\{\Lambda_i(\widehat{\boldsymbol{\theta}}_{i})\right\} \leq 2G^2 + (6 + \lambda C)G^2\sum_{k=0}^{\tau-1} \beta_k.
\end{equation*}
For any $i > \tau > 0$,
\begin{equation*}
\mathbb{E}\left\{\Lambda_i(\widehat{\boldsymbol{\theta}}_{i})\right\} \leq 2C|\mathcal{A}|G^3\tau\sum_{k=i-\tau}^{i-1} \beta_k + 4G^2m\rho^{\tau-1} + (6 + \lambda C)G^2\sum_{k=i-\tau}^{i-1} \beta_k.
\end{equation*}
\end{lemma}
\begin{proof}
We provide a proof for the case $i \le \tau$. For $i > \tau > 0$, the proof follows directly from Lemma 7 of \citep{zou2019finite}. For any $i \geq 0$, we have
\begin{align*}
\left\|\widehat{\boldsymbol{\theta}}_{k+1} - \widehat{\boldsymbol{\theta}}_k\right\|_2 &= \left\|\Pi_{r}\{\widehat{\boldsymbol{\theta}}_k + \beta_k g_k(\widehat{\boldsymbol{\theta}}_k)\} - \Pi_{r}(\widehat{\boldsymbol{\theta}}_k)\right\|_2 \\
&\leq \left\|\beta_k g_k(\widehat{\boldsymbol{\theta}}_k)\right\|_2 && \text{(non-expansiveness of $\Pi_r$)}\\
&\leq \beta_k G. && \text{(by Lemma \ref{lem:C1})}
\end{align*}
Therefore, by the triangle inequality, 
$$
\left\|\widehat{\boldsymbol{\theta}}_{i} - \widehat{\boldsymbol{\theta}}_{0}\right\|_2 \leq \sum_{k=0}^{i-1}\left\|\widehat{\boldsymbol{\theta}}_{i+1} - \widehat{\boldsymbol{\theta}}_i\right\|_2 \leq G\sum_{k=0}^{i-1}\beta_k \leq G\sum_{k=0}^{\tau-1}\beta_k ~~\text{for any} ~i \le \tau.
$$
Combined with the Lipschitzness of $\Lambda_i(\boldsymbol{\theta})$ in Lemma \ref{lem:C4}, we have 
\begin{align*}
  |\Lambda_i(\widehat{\boldsymbol{\theta}}_{i}) - \Lambda_i(\widehat{\boldsymbol{\theta}}_{0})| \leq (6+\lambda C)G^2\sum_{k=0}^{\tau-1}\beta_k \Longrightarrow \Lambda_i(\widehat{\boldsymbol{\theta}}_{i}) \leq \Lambda_i(\widehat{\boldsymbol{\theta}}_{0}) + (6 + \lambda C)G^2\sum_{k=0}^{\tau-1}\beta_k.  
\end{align*}
Applying Lemma \ref{lem:C3} and taking expectation on both sides, we get
$$
\mathbb{E}\left\{\Lambda_i(\widehat{\boldsymbol{\theta}}_{i})\right\} \leq 2G^2 + (6 + \lambda C)G^2\sum_{k=0}^{\tau-1}\beta_k. 
$$
\end{proof}

\begin{lemma}\label{lem:C6}
Given a step-size sequence $(\beta_k)_{k \ge 0}$, for any $i \le \tau$,
\begin{equation*}
\mathbb{E}\left\{\Lambda_i(\widehat{\boldsymbol{\theta}}^{\text{im}}_{i})\right\} \leq 2G^2 + (6 + \lambda C)G^2\sum_{k=0}^{\tau-1} \beta_k.
\end{equation*}
For any $i > \tau > 0$,
\begin{equation*}
\mathbb{E}\left\{\Lambda_i(\widehat{\boldsymbol{\theta}}^{\text{im}}_{i})\right\} \leq 2C|\mathcal{A}|G^3\tau\sum_{k=t-\tau}^{i-1} \beta_k + 4G^2m\rho^{\tau-1} + (6 + \lambda C)G^2\sum_{k=t-\tau}^{i-1} \beta_k.
\end{equation*}
\end{lemma}
\begin{proof}
We proceed in three steps. The proof largely follows the approach of Lemma 7 in \citet{zou2019finite}, with additional care in handling the implicit iterates. \\
\textbf{Step 1.} For any $i \geq 0$, we have
\begin{align*}
\left\|\widehat{\boldsymbol{\theta}}^{\text{im}}_{k+1} - \widehat{\boldsymbol{\theta}}^{\text{im}}_k\right\|_2 &= \left\|\Pi_{r}\{\widehat{\boldsymbol{\theta}}^{\text{im}}_k + \tilde\beta_k g_k(\widehat{\boldsymbol{\theta}}^{\text{im}}_k)\} - \Pi_{r}(\widehat{\boldsymbol{\theta}}^{\text{im}}_k)\right\|_2 \\
&\leq \left\|\tilde\beta_k g_k(\widehat{\boldsymbol{\theta}}^{\text{im}}_k)\right\|_2 && \text{(non-expansiveness of $\Pi_r$)}\\
&\leq \beta_k G. && \text{(by Lemma \ref{lem:C1} and $\tilde\beta_k \le \beta_k$)}
\end{align*}
The proof for the case $i \le \tau$ follows from analogous arguments to those in Lemma \ref{lem:C5}. Therefore, we focus on the case $i > \tau > 0$. By the triangle inequality, 
$$
\left\|\widehat{\boldsymbol{\theta}}^{\text{im}}_{i} - \widehat{\boldsymbol{\theta}}^{\text{im}}_{i-\tau}\right\|_2 \leq \sum_{k=i-\tau}^{i-1}\left\|\widehat{\boldsymbol{\theta}}^{\text{im}}_{k+1} - \widehat{\boldsymbol{\theta}}^{\text{im}}_k\right\|_2 \leq G\sum_{k=i-\tau}^{i-1}\beta_k ~~\text{for any} ~i > \tau \geq 0.
$$
Combined with the Lipschitzness of $\Lambda_i(\boldsymbol{\theta})$ in Lemma \ref{lem:C4}, we have 
\begin{align}\label{LEMC6:STEP1}
  |\Lambda_i(\widehat{\boldsymbol{\theta}}^{\text{im}}_{i}) - \Lambda_i(\widehat{\boldsymbol{\theta}}^{\text{im}}_{i-\tau})| \leq (6+\lambda C)G^2\sum_{k=i-\tau}^{i-1}\beta_k \Longrightarrow\Lambda_t(\widehat{\boldsymbol{\theta}}^{\text{im}}_{t}) \leq \Lambda_i(\widehat{\boldsymbol{\theta}}^{\text{im}}_{i-\tau}) + (6 + \lambda C)G^2\sum_{k=i-\tau}^{i-1}\beta_k.  
\end{align}
In the subsequent steps, we aim to upper bound the quantity $\Lambda_i(\widehat{\boldsymbol{\theta}}^{\text{im}}_{i-\tau})$ in expectation.\\

\noindent \textbf{Step 2.} In this step, we obtain an upper bound on the closely related object induced by an auxiliary Markov chain. To clarify the approach, we use the abbreviated policy notation $\pi_i := \pi_{\widehat{\boldsymbol{\theta}}^{\text{im}}_{i}}$ and data tuple $O_i := (\boldsymbol{S}_i^{\pi_{i-2}},A_i^{\pi_{i-1}}, \boldsymbol{S}_{i+1}^{\pi_{i-1}},A_{i+1}^{\pi_{i}})$.  To highlight the dependence on data, we rewrite $\Lambda_i(\widehat{\boldsymbol{\theta}}^{\text{im}}_{i-\tau}) = \Lambda(\widehat{\boldsymbol{\theta}}^{\text{im}}_{i-\tau}, O_i)$. Additionally, conditioning on $\widehat{\boldsymbol{\theta}}^{\text{im}}_{i-\tau}$ and $\boldsymbol{S}^{\pi_{i-\tau}}_{i-\tau+2}$, we construct the following auxiliary Markov chain
\begin{equation}
\boldsymbol{S}^{\pi_{i-\tau-1}}_{i-\tau+1}\to
\boldsymbol{S}^{\pi_{i-\tau}}_{i-\tau+2} \to \boldsymbol{S}^{\pi_{i-\tau}}_{i-\tau+3} \to \cdots \to \boldsymbol{S}^{\pi_{i-\tau}}_i \to \boldsymbol{S}^{\pi_{i-\tau}}_{i+1},
\end{equation}
whose transition probability is defined as
\begin{align*}
\mathbb{P}(\boldsymbol{S}^{\pi_{i-\tau}}_{i-\tau+2} \in \cdot\mid\widehat{\boldsymbol{\theta}}^{\text{im}}_{i-\tau}, \boldsymbol{S}^{\pi_{i-\tau-1}}_{i-\tau+1}) &= \sum_{a \in A}\pi_{i-\tau}(A^{\pi_{i-\tau}}_{i-\tau+1} = a\mid\boldsymbol{S}^{\pi_{i-\tau-1}}_{i-\tau+1})\Ptrans(\boldsymbol{S}^{\pi_{i-\tau}}_{i-\tau+2} \in \cdot\mid A^{\pi_{i-\tau}}_{i-\tau+1} = a, \boldsymbol{S}^{\pi_{i-\tau-1}}_{i-\tau+1})\\
\mathbb{P}(\boldsymbol{S}^{\pi_{i-\tau}}_k \in \cdot\mid\widehat{\boldsymbol{\theta}}^{\text{im}}_{i-\tau}, \boldsymbol{S}^{\pi_{i-\tau}}_{k-1}) &= \sum_{a \in A}\pi_{i-\tau}(A^{\pi_{i-\tau}}_{k-1} = a\mid\boldsymbol{S}^{\pi_{i-\tau}}_{k-1})\Ptrans(\boldsymbol{S}^{\pi_{i-\tau}}_k \in \cdot\mid A^{\pi_{i-\tau}}_{k-1} = a, \boldsymbol{S}^{\pi_{i-\tau}}_{k-1}),
\end{align*}
for $k = i-\tau+3, \ldots, i+1$. In short, conditioning on $\widehat{\boldsymbol{\theta}}^{\text{im}}_{i-\tau}$ and $\boldsymbol{S}^{\pi_{i-\tau-1}}_{i-\tau+1}$, the constructed Markov chain is induced by repeatedly following the same behavior policy $\pi_{i-\tau}$. From the Assumption \ref{assume:ergodicity_SARSA}, it follows that conditioning on $(\widehat{\boldsymbol{\theta}}^{\text{im}}_{i-\tau}, \boldsymbol{S}^{\pi_{i-\tau-1}}_{i-\tau+1})$, for any $k \geq i-\tau+2$,
\begin{equation*}
\left\|\mathbb{P}(\boldsymbol{S}^{\pi_{i-\tau}}_k \in \cdot\mid\widehat{\boldsymbol{\theta}}^{\text{im}}_{i-\tau}, \boldsymbol{S}^{\pi_{i-\tau-1}}_{i-\tau+1}) - \mu^{\pi_{i-\tau}}\right\|_{TV} \leq m\rho^{k-(i-\tau+1)}.
\end{equation*}
Now, let us define auxiliary data tuple $\tilde{O}_i := (\boldsymbol{S}_i^{\pi_{i-\tau}},A_i^{\pi_{i-\tau}}, \boldsymbol{S}_{i+1}^{\pi_{i-\tau}},A_{i+1}^{\pi_{i-\tau}})$. Applying Lemma \ref{lem:C3}, Lemma 9 in \citep{bhandari2018finite} and the triangle inequality with
\begin{align*}
    \left\|\mathbb{P}(\boldsymbol{S}^{\pi_{i-\tau}}_i \in \cdot\mid\widehat{\boldsymbol{\theta}}^{\text{im}}_{i-\tau}, \boldsymbol{S}^{\pi_{i-\tau-1}}_{i-\tau+1}) - \mu^{\pi_{i-\tau}}\right\|_{TV} &\leq m\rho^{\tau-1}\\
        \left\|\mathbb{P}(\boldsymbol{S}^{\pi_{i-\tau}}_{i+1} \in \cdot\mid\widehat{\boldsymbol{\theta}}^{\text{im}}_{i-\tau}, \boldsymbol{S}^{\pi_{i-\tau-1}}_{i-\tau+1}) - \mu^{\pi_{i-\tau}}\right\|_{TV} &\leq m\rho^{\tau},
\end{align*}
we obtain
\begin{align*}
\mathbb{E}\left\{\Lambda(\widehat{\boldsymbol{\theta}}^{\text{im}}_{i-\tau}, \tilde{O}_i)\mid\widehat{\boldsymbol{\theta}}^{\text{im}}_{i-\tau}, \boldsymbol{S}^{\pi_{i-\tau-1}}_{i-\tau+1}\right\} - \mathbb{E}\left\{\Lambda(\widehat{\boldsymbol{\theta}}^{\text{im}}_{i-\tau}, O'_i) \mid \widehat{\boldsymbol{\theta}}^{\text{im}}_{i-\tau}, \boldsymbol{S}^{\pi_{i-\tau-1}}_{i-\tau+1}\right\} 
\leq 4G^2m\rho^{\tau-1}
\end{align*}
where $O'_i$ are independently generated by the stationary distribution $\mu^{\pi_{i-\tau}}$ and the policy $\pi_{i-\tau}$. Since $\mathbb{E}^{\mu^{\boldsymbol{\theta}}}\left\{\Lambda_i(\boldsymbol{\theta})\right\} = 0$ for any fixed $\boldsymbol{\theta}$, we know 
$$
\mathbb{E}\left\{\Lambda(\widehat{\boldsymbol{\theta}}^{\text{im}}_{i-\tau}, O'_i)\mid\widehat{\boldsymbol{\theta}}^{\text{im}}_{i-\tau}, \boldsymbol{S}^{\pi_{i-\tau-1}}_{i-\tau+1}\right\} = 0,
$$
which yields
\begin{equation}\label{LEMC6:STEP2}
\mathbb{E}\left\{\Lambda(\widehat{\boldsymbol{\theta}}^{\text{im}}_{i-\tau}, \tilde{O}_i)\right\}  \leq 4G^2m\rho^{\tau-1}.
\end{equation}

\noindent \textbf{Step 3.} In this final step, we analyze the difference between the target quantity 
$\mathbb{E}\left\{\Lambda_i(\widehat{\boldsymbol{\theta}}^{\text{im}}_{i-\tau})\right\}$ and the auxiliary quantity $\mathbb{E}\left\{\Lambda(\widehat{\boldsymbol{\theta}}^{\text{im}}_{i-\tau}, \tilde{O}_i)\right\}$. First, note that
\begin{align*}
&\mathbb{E}\left\{\Lambda_i(\widehat{\boldsymbol{\theta}}^{\text{im}}_{i-\tau}, O_i)\mid\widehat{\boldsymbol{\theta}}^{\text{im}}_{i-\tau},  \boldsymbol{S}^{\pi_{i-\tau-1}}_{i-\tau+1}\right\} - \mathbb{E}\left\{\Lambda_i(\widehat{\boldsymbol{\theta}}^{\text{im}}_{i-\tau}, \tilde{O}_i)\mid\widehat{\boldsymbol{\theta}}^{\text{im}}_{i-\tau},  \boldsymbol{S}^{\pi_{i-\tau-1}}_{i-\tau+1}\right\}\\& \leq 2G^2\left\|\mathbb{P}(O_i \in \cdot\mid\widehat{\boldsymbol{\theta}}^{\text{im}}_{i-\tau},  \boldsymbol{S}^{\pi_{i-\tau-1}}_{i-\tau+1}) - \mathbb{P}(\tilde{O}_i \in \cdot\mid\widehat{\boldsymbol{\theta}}^{\text{im}}_{i-\tau},  \boldsymbol{S}^{\pi_{i-\tau-1}}_{i-\tau+1})\right\|_{TV}.
\end{align*}
From the result in section B of \citep{zou2019finite}, we have
\begin{align*}
&\left\|\mathbb{P}(O_i \in \cdot \mid \widehat{\boldsymbol{\theta}}^{\text{im}}_{i-\tau},  \boldsymbol{S}^{\pi_{i-\tau-1}}_{i-\tau+1}) - \mathbb{P}(\tilde{O}_i \in \cdot \mid \widehat{\boldsymbol{\theta}}^{\text{im}}_{i-\tau},  \boldsymbol{S}^{\pi_{i-\tau-1}}_{i-\tau+1})\right\|_{TV} \nonumber \\
&\leq \left\|\mathbb{P}(\boldsymbol{S}^{\pi_{i-2}}_{i} \in \cdot \mid \widehat{\boldsymbol{\theta}}^{\text{im}}_{i-\tau},  \boldsymbol{S}^{\pi_{i-\tau-1}}_{i-\tau+1}) - \mathbb{P}(\boldsymbol{S}^{\pi_{i-\tau}}_{i} \in \cdot \mid \widehat{\boldsymbol{\theta}}^{\text{im}}_{i-\tau},  \boldsymbol{S}^{\pi_{i-\tau-1}}_{i-\tau+1})\right\|_{TV} + C|\mathcal{A}|G\sum_{j=i-2}^{i-1}\sum_{k=i-\tau}^{j}\beta_k. 
\end{align*}
We obtain the upper bound of the first term recursively as follows:
\begin{align}
&\left\|\mathbb{P}(\boldsymbol{S}^{\pi_{i-2}}_{i} \in \cdot \mid \widehat{\boldsymbol{\theta}}^{\text{im}}_{i-\tau},  \boldsymbol{S}^{\pi_{i-\tau-1}}_{i-\tau+1}) - \mathbb{P}(\boldsymbol{S}^{\pi_{i-\tau}}_{i} \in \cdot \mid \widehat{\boldsymbol{\theta}}^{\text{im}}_{i-\tau},  \boldsymbol{S}^{\pi_{i-\tau-1}}_{i-\tau+1})\right\|_{TV} \nonumber \\
&= \frac{1}{2}\int_{s' \in \mathcal{S}} \left| \mathbb{P}(\boldsymbol{S}^{\pi_{i-2}}_{i} = ds'\mid\widehat{\boldsymbol{\theta}}^{\text{im}}_{i-\tau}, \boldsymbol{S}^{\pi_{i-\tau-1}}_{i-\tau+1}) - \mathbb{P}(\boldsymbol{S}^{\pi_{i-\tau}}_{i} = ds'\mid\widehat{\boldsymbol{\theta}}^{\text{im}}_{i-\tau}, \boldsymbol{S}^{\pi_{i-\tau-1}}_{i-\tau+1}) \right| \nonumber \\
&= \frac{1}{2}\int_{s' \in \mathcal{S}} \Big| \int_{s \in \mathcal{S}} \mathbb{P}(\boldsymbol{S}^{\pi_{i-3}}_{i-1} = ds\mid\widehat{\boldsymbol{\theta}}^{\text{im}}_{i-\tau}, \boldsymbol{S}^{\pi_{i-\tau-1}}_{i-\tau+1})\mathbb{P}(\boldsymbol{S}^{\pi_{i-2}}_{i} = ds'\mid\widehat{\boldsymbol{\theta}}^{\text{im}}_{i-\tau}, \boldsymbol{S}^{\pi_{i-\tau-1}}_{i-\tau+1}, \boldsymbol{S}^{\pi_{i-3}}_{i-1} = s) \nonumber \\
&\quad - \int_{s \in \mathcal{S}} \mathbb{P}(\boldsymbol{S}^{\pi_{i-\tau}}_{i-1} = ds\mid\widehat{\boldsymbol{\theta}}^{\text{im}}_{i-\tau}, \boldsymbol{S}^{\pi_{i-\tau-1}}_{i-\tau+1})\mathbb{P}(\boldsymbol{S}^{\pi_{i-\tau}}_{i} = ds'\mid\widehat{\boldsymbol{\theta}}^{\text{im}}_{i-\tau}, \boldsymbol{S}^{\pi_{i-\tau-1}}_{i-\tau+1}, \boldsymbol{S}^{\pi_{i-\tau}}_{i-1} = s) \Big| \label{LEMD6_STEP3_LAST}
\end{align}
Pushing absolute value operation inside the integral, we upper bound \eqref{LEMD6_STEP3_LAST} as
\begin{align}
\eqref{LEMD6_STEP3_LAST}&\leq \frac{1}{2}\int_{s' \in \mathcal{S}} \int_{s \in \mathcal{S}} \Big| \mathbb{P}(\boldsymbol{S}^{\pi_{i-3}}_{i-1} = ds\mid\widehat{\boldsymbol{\theta}}^{\text{im}}_{i-\tau}, \boldsymbol{S}^{\pi_{i-\tau-1}}_{i-\tau+1})\mathbb{P}(\boldsymbol{S}^{\pi_{i-2}}_{i} = ds'\mid\widehat{\boldsymbol{\theta}}^{\text{im}}_{i-\tau}, \boldsymbol{S}^{\pi_{i-\tau-1}}_{i-\tau+1}, \boldsymbol{S}^{\pi_{i-3}}_{i-1} = s) \nonumber \\
&\quad - \mathbb{P}(\boldsymbol{S}^{\pi_{i-\tau}}_{i-1} = ds\mid\widehat{\boldsymbol{\theta}}^{\text{im}}_{i-\tau}, \boldsymbol{S}^{\pi_{i-\tau-1}}_{i-\tau+1})\mathbb{P}(\boldsymbol{S}^{\pi_{i-\tau}}_{i} = ds'\mid\widehat{\boldsymbol{\theta}}^{\text{im}}_{i-\tau}, \boldsymbol{S}^{\pi_{i-\tau-1}}_{i-\tau+1}, \boldsymbol{S}^{\pi_{i-\tau}}_{i-1} = s) \Big|. \nonumber 
\end{align}
Using the triangle inequality, we know that 
\begin{align}
\eqref{LEMD6_STEP3_LAST}&\leq \frac{1}{2}\int_{s' \in \mathcal{S}} \int_{s \in \mathcal{S}} \Bigg\{ \bigg| \mathbb{P}(\boldsymbol{S}^{\pi_{i-3}}_{i-1} = ds\mid\widehat{\boldsymbol{\theta}}^{\text{im}}_{i-\tau}, \boldsymbol{S}^{\pi_{i-\tau-1}}_{i-\tau+1})\mathbb{P}(\boldsymbol{S}^{\pi_{i-2}}_{i} = ds'\mid\widehat{\boldsymbol{\theta}}^{\text{im}}_{i-\tau}, \boldsymbol{S}^{\pi_{i-\tau-1}}_{i-\tau+1}, \boldsymbol{S}^{\pi_{i-3}}_{i-1} = s) \nonumber \\
&\quad - \mathbb{P}(\boldsymbol{S}^{\pi_{i-\tau}}_{i-1} = ds\mid\widehat{\boldsymbol{\theta}}^{\text{im}}_{i-\tau}, \boldsymbol{S}^{\pi_{i-\tau-1}}_{i-\tau+1})\mathbb{P}(\boldsymbol{S}^{\pi_{i-2}}_{i} = ds'\mid\widehat{\boldsymbol{\theta}}^{\text{im}}_{i-\tau}, \boldsymbol{S}^{\pi_{i-\tau-1}}_{i-\tau+1}, \boldsymbol{S}^{\pi_{i-3}}_{i-1} = s) \bigg| \nonumber \\
&\quad + \bigg| \mathbb{P}(\boldsymbol{S}^{\pi_{i-\tau}}_{i-1} = ds\mid\widehat{\boldsymbol{\theta}}^{\text{im}}_{i-\tau}, \boldsymbol{S}^{\pi_{i-\tau-1}}_{i-\tau+1})\mathbb{P}(\boldsymbol{S}^{\pi_{i-2}}_{i} = ds'\mid\widehat{\boldsymbol{\theta}}^{\text{im}}_{i-\tau}, \boldsymbol{S}^{\pi_{i-\tau-1}}_{i-\tau+1}, \boldsymbol{S}^{\pi_{i-3}}_{i-1} = s) \nonumber \\
&\quad - \mathbb{P}(\boldsymbol{S}^{\pi_{i-\tau}}_{i-1} = ds\mid\widehat{\boldsymbol{\theta}}^{\text{im}}_{i-\tau}, \boldsymbol{S}^{\pi_{i-\tau-1}}_{i-\tau+1})\mathbb{P}(\boldsymbol{S}^{\pi_{i-\tau}}_{i} = ds'\mid\widehat{\boldsymbol{\theta}}^{\text{im}}_{i-\tau}, \boldsymbol{S}^{\pi_{i-\tau-1}}_{i-\tau+1}, \boldsymbol{S}^{\pi_{i-\tau}}_{i-1} = s) \bigg| \Bigg\}, \nonumber 
\end{align}
which leads to
\begin{align*}
\eqref{LEMD6_STEP3_LAST}&\leq \left\|\mathbb{P}(\boldsymbol{S}^{\pi_{i-3}}_{i-1} \in \cdot\mid\widehat{\boldsymbol{\theta}}^{\text{im}}_{i-\tau}, \boldsymbol{S}^{\pi_{i-\tau-1}}_{i-\tau+1}) - \mathbb{P}(\boldsymbol{S}^{\pi_{i-\tau}}_{i-1} \in \cdot\mid\widehat{\boldsymbol{\theta}}^{\text{im}}_{i-\tau}, \boldsymbol{S}^{\pi_{i-\tau-1}}_{i-\tau+1})\right\|_{TV} \nonumber \\
&\quad + \sup_{s \in \mathcal{S}} \left\|\mathbb{P}(\boldsymbol{S}^{\pi_{i-2}}_{i} \in \cdot\mid\widehat{\boldsymbol{\theta}}^{\text{im}}_{i-\tau}, \boldsymbol{S}^{\pi_{i-\tau-1}}_{i-\tau+1}, \boldsymbol{S}^{\pi_{i-3}}_{i-1} = s) - \mathbb{P}(\boldsymbol{S}^{\pi_{i-\tau}}_{i} \in \cdot\mid\widehat{\boldsymbol{\theta}}^{\text{im}}_{i-\tau}, \boldsymbol{S}^{\pi_{i-\tau-1}}_{i-\tau+1}, \boldsymbol{S}^{\pi_{i-\tau}}_{i-1} = s)\right\|_{TV}.
\end{align*}
We re-express the term inside the supremum operator as follows: 
\begin{align}
&\left\|\mathbb{P}(\boldsymbol{S}^{\pi_{i-2}}_{i} \in \cdot\mid\widehat{\boldsymbol{\theta}}^{\text{im}}_{i-\tau}, \boldsymbol{S}^{\pi_{i-\tau-1}}_{i-\tau+1}, \boldsymbol{S}^{\pi_{i-3}}_{i-1} = s) - \mathbb{P}(\boldsymbol{S}^{\pi_{i-\tau}}_{i} \in \cdot\mid\widehat{\boldsymbol{\theta}}^{\text{im}}_{i-\tau}, \boldsymbol{S}^{\pi_{i-\tau-1}}_{i-\tau+1}, \boldsymbol{S}^{\pi_{i-\tau}}_{i-1} = s)\right\|_{TV} \nonumber \\
&= \frac{1}{2}\int_{s' \in \mathcal{S}} \Bigg| \sum_{a \in \mathcal{A}} \Ptrans(ds'\mid s, a)\mathbb{E}\left\{\pi_{i-2}(a\mid s)\mid\widehat{\boldsymbol{\theta}}^{\text{im}}_{i-\tau}, \boldsymbol{S}^{\pi_{i-\tau-1}}_{i-\tau+1}, \boldsymbol{S}^{\pi_{i-3}}_{i-1} = s\right\} - \Ptrans(ds'\mid s, a)\pi_{i-\tau}(a\mid s) \Bigg| \label{LEMD6_STEP3_EQUAL}
\end{align}
where the $\mathbb{E}$ is with respect to $\widehat{\boldsymbol{\theta}}^{\text{im}}_{i-2}$ used in the policy $\pi_{i-2}$. We provide the justification of \eqref{LEMD6_STEP3_EQUAL} in the following subsection. Then we have
\begin{align*}
    \eqref{LEMD6_STEP3_EQUAL}&\leq \frac{1}{2}\int_{s' \in \mathcal{S}} \sum_{a \in \mathcal{A}} \Ptrans(ds'\mid s, a) \left| \mathbb{E}\left\{\pi_{i-2}(a\mid s) - \pi_{i-\tau}(a\mid s)\mid \widehat{\boldsymbol{\theta}}^{\text{im}}_{i-\tau}, \boldsymbol{S}^{\pi_{i-\tau-1}}_{i-\tau+1}, \boldsymbol{S}^{\pi_{i-3}}_{i-1} = s\right\} \right| \nonumber \\
    &\leq C\int_{s' \in \mathcal{S}} \sum_{a \in \mathcal{A}} \Ptrans(ds'\mid s, a) \left\| \widehat{\boldsymbol{\theta}}^{\text{im}}_{i-2}-\widehat{\boldsymbol{\theta}}^{\text{im}}_{i-\tau}\right\|_2 \leq C|{\mathcal{A}}|G \sum_{k=i-\tau}^{i-3} \beta_k,
\end{align*}
where the first inequality is due to the triangle inequality and the second inequality is the consequence of the Lipschitzness of the policy. In the third inequality, we used the fact that
\begin{equation*}
\left\|\widehat{\boldsymbol{\theta}}^{\text{im}}_{i-2} - \widehat{\boldsymbol{\theta}}^{\text{im}}_{i-\tau}\right\|_2 \leq \sum_{k=i-\tau}^{i-3} \left\|\widehat{\boldsymbol{\theta}}^{\text{im}}_{k+1} - \widehat{\boldsymbol{\theta}}^{\text{im}}_{k}\right\|_2 \leq G \sum_{k=i-\tau}^{i-3} \tilde\beta_k\leq G \sum_{k=i-\tau}^{i-3} \beta_k,
\end{equation*}
Thus, we obtain a recursive bound of \eqref{LEMD6_STEP3_LAST} as follows:
\begin{equation*}
\eqref{LEMD6_STEP3_LAST}  \leq \left\|\mathbb{P}(\boldsymbol{S}^{\pi_{i-3}}_{i-1} \in \cdot\mid\widehat{\boldsymbol{\theta}}^{\text{im}}_{i-\tau}, \boldsymbol{S}^{\pi_{i-\tau-1}}_{i-\tau+1}) - \mathbb{P}(\boldsymbol{S}^{\pi_{i-\tau}}_{i-1} \in \cdot\mid\widehat{\boldsymbol{\theta}}^{\text{im}}_{i-\tau}, \boldsymbol{S}^{\pi_{i-\tau-1}}_{i-\tau+1})\right\|_{TV} + C|{\mathcal{A}}|G \sum_{k=i-\tau}^{i-3} \beta_k.
\end{equation*}
Applying this recursion repeatedly, we obtain
\begin{align*}
&\left\|\mathbb{P}(O_i \in \cdot\mid\widehat{\boldsymbol{\theta}}^{\text{im}}_{i-\tau}, \boldsymbol{S}^{\pi_{i-\tau-1}}_{i-\tau+1}) - \mathbb{P}(\tilde{O}_i \in \cdot\mid\widehat{\boldsymbol{\theta}}^{\text{im}}_{i-\tau},  \boldsymbol{S}^{\pi_{i-\tau-1}}_{i-\tau+1})\right\|_{TV} \\
&\leq \left\|\mathbb{P}(\boldsymbol{S}^{\pi_{i-\tau}}_{i-\tau+2} \in \cdot \mid \widehat{\boldsymbol{\theta}}^{\text{im}}_{i-\tau},  \boldsymbol{S}^{\pi_{i-\tau-1}}_{i-\tau+1}) - \mathbb{P}(\boldsymbol{S}^{\pi_{i-\tau}}_{i-\tau+2} \in \cdot \mid \widehat{\boldsymbol{\theta}}^{\text{im}}_{i-\tau},  \boldsymbol{S}^{\pi_{i-\tau-1}}_{i-\tau+1})\right\|_{TV} + C|\mathcal{A}|G\sum_{j=i-\tau}^{i-1}\sum_{k=i-\tau}^{j}\beta_k \\
&\le C|\mathcal{A}|G \tau \sum_{k=i-\tau}^{i-1}\beta_k
\end{align*}
Combining all three steps, we obtain the desired result.
\end{proof}

\subsubsection{Justification of \eqref{LEMD6_STEP3_EQUAL}}
\noindent Notice that
\begin{align}
&\mathbb{P}(\boldsymbol{S}^{\pi_{i-2}}_{i} \in \cdot\mid\widehat{\boldsymbol{\theta}}^{\text{im}}_{i-\tau},  \boldsymbol{S}^{\pi_{i-\tau-1}}_{i-\tau+1}, \boldsymbol{S}^{\pi_{i-3}}_{i-1} = s) \nonumber \\
&= \sum_{a \in \mathcal{A}} \mathbb{P}(\boldsymbol{S}^{\pi_{i-2}}_{i} \in \cdot, A^{\pi_{i-2}}_{i-1} = a \mid\widehat{\boldsymbol{\theta}}^{\text{im}}_{i-\tau}, \boldsymbol{S}^{\pi_{i-\tau-1}}_{i-\tau+1}, \boldsymbol{S}^{\pi_{i-3}}_{i-1} = s) \nonumber \\
&= \sum_{a \in \mathcal{A}} \Ptrans(\boldsymbol{S}^{\pi_{i-2}}_{i} \in \cdot\mid \boldsymbol{S}^{\pi_{i-3}}_{i-1} = s, A^{\pi_{i-2}}_{i-1} = a)\mathbb{P}(A^{\pi_{i-2}}_{i-1} = a\mid\widehat{\boldsymbol{\theta}}^{\text{im}}_{i-\tau}, \boldsymbol{S}^{\pi_{i-\tau-1}}_{i-\tau+1}, \boldsymbol{S}^{\pi_{i-3}}_{i-1} = s).
\end{align}
Since 
\begin{align*}
&\mathbb{P}(A^{\pi_{i-2}}_{i-1} = a\mid\widehat{\boldsymbol{\theta}}^{\text{im}}_{i-\tau}, \boldsymbol{S}^{\pi_{i-\tau-1}}_{i-\tau+1}, \boldsymbol{S}^{\pi_{i-3}}_{i-1} = s)\\ &= \mathbb{E}\left\{\mathbb{P}(A^{\pi_{i-2}}_{i-1} = a \mid \widehat{\boldsymbol{\theta}}^{\text{im}}_{i-2}, \widehat{\boldsymbol{\theta}}^{\text{im}}_{i-\tau}, \boldsymbol{S}^{\pi_{i-\tau}}_{i-\tau+2}, \boldsymbol{S}^{\pi_{i-3}}_{i-1} = s) \mid \widehat{\boldsymbol{\theta}}^{\text{im}}_{i-\tau}, \boldsymbol{S}^{\pi_{i-\tau-1}}_{i-\tau+1}, \boldsymbol{S}^{\pi_{i-3}}_{i-1} = s\right\} \\
&= \mathbb{E}\left\{\pi_{i-2}(a\mid s)\mid\widehat{\boldsymbol{\theta}}^{\text{im}}_{i-\tau}, \boldsymbol{S}^{\pi_{i-\tau-1}}_{i-\tau+1}, \boldsymbol{S}^{\pi_{i-3}}_{i-1} = s\right\}
\end{align*}
where the expectation is with respect to the randomness of $\widehat{\boldsymbol{\theta}}^{\text{im}}_{i-2}$, we have
\begin{align*}
&\mathbb{P}(\boldsymbol{S}^{\pi_{i-2}}_{i} \in \cdot \mid \widehat{\boldsymbol{\theta}}^{\text{im}}_{i-\tau},  \boldsymbol{S}^{\pi_{i-\tau-1}}_{i-\tau+1},  \boldsymbol{S}^{\pi_{i-3}}_{i-1} = s)\\ &= \sum_{a \in \mathcal{A}} \Ptrans(\boldsymbol{S}^{\pi_{i-2}}_{i} \in \cdot\mid \boldsymbol{S}^{\pi_{i-3}}_{i-1} = s, A^{\pi_{i-2}}_{i-1} = a)\mathbb{E}\left\{\pi_{i-2}(a\mid s)\mid\widehat{\boldsymbol{\theta}}^{\text{im}}_{i-\tau}, \boldsymbol{S}^{\pi_{i-\tau-1}}_{i-\tau+1}, \boldsymbol{S}^{\pi_{i-3}}_{i-1} = s\right\}.
\end{align*}
Finally, one can also observe that
\begin{align}
\mathbb{P}(\boldsymbol{S}^{\pi_{i-\tau}}_{i} \in \cdot\mid \widehat{\boldsymbol{\theta}}^{\text{im}}_{i-\tau},  \boldsymbol{S}^{\pi_{i-\tau-1}}_{i-\tau+1},  \boldsymbol{S}^{\pi_{i-\tau}}_{i-1} = s)
= \sum_{a \in \mathcal{A}} \Ptrans(\boldsymbol{S}^{\pi_{i-\tau}}_{i} \in \cdot\mid \boldsymbol{S}^{\pi_{i-\tau}}_{i-1} = s, \boldsymbol{A}^{\pi_{i-\tau}}_{i-1} = a)\pi_{i-\tau}(a|s).
\end{align}

\subsection{Proof of Theorem \ref{thm:constant_SARSA}}
\noindent In this section, we obtain finite-time error bounds of the projected implicit SARSA iterates with a constant step-size sequence, i.e., $\beta_i = \beta, ~\forall t \in \mathbb{N}$. Recall that $\tilde{\beta}_t := \beta/(1+\beta\|\boldsymbol{\phi}(\boldsymbol{S}_t^{\pi_{t-2}},A_t^{\pi_{t-1}})\|^2_2)$. Consider
\begin{align*}
\mathbb{E}\left\|\widehat{\boldsymbol{\theta}}^{\text{im}}_{t+1} - \boldsymbol{\theta}^{\star}\right\|_2^2 &= \mathbb{E}\left\|\Pi_r\{\widehat{\boldsymbol{\theta}}^{\text{im}}_{t} + \tilde\beta_t g_t(\widehat{\boldsymbol{\theta}}_{t})\} - \Pi_r\boldsymbol{\theta}^{\star}\right\|_2^2 \\
&\leq \mathbb{E}\left\|\widehat{\boldsymbol{\theta}}^{\text{im}}_{t} + \tilde\beta_t g_t(\widehat{\boldsymbol{\theta}}^{\text{im}}_{t}) - \boldsymbol{\theta}^{\star}\right\|_2^2\\
&\leq \mathbb{E}\left\{\left\|\widehat{\boldsymbol{\theta}}^{\text{im}}_{t} - \boldsymbol{\theta}^{\star}\right\|_2^2 + \beta^2\left\|g_t(\widehat{\boldsymbol{\theta}}^{\text{im}}_{t})\right\|_2^2 + 2\tilde\beta_t\langle\widehat{\boldsymbol{\theta}}^{\text{im}}_{t} - \boldsymbol{\theta}^{\star}, g_t(\widehat{\boldsymbol{\theta}}^{\text{im}}_{t})\rangle\right\} \\
&= \mathbb{E}\left\{\left\|\widehat{\boldsymbol{\theta}}^{\text{im}}_{t} - \boldsymbol{\theta}^{\star}\right\|_2^2 + \beta^2\left\|g_t(\widehat{\boldsymbol{\theta}}^{\text{im}}_{t})\right\|_2^2 + 2\tilde\beta_t\langle\widehat{\boldsymbol{\theta}}^{\text{im}}_{t} - \boldsymbol{\theta}^{\star}, \bar{g}(\widehat{\boldsymbol{\theta}}^{\text{im}}_{t}) - \bar{g}(\boldsymbol{\theta}^{\star})\rangle + 2\tilde\beta_t\Lambda_t(\widehat{\boldsymbol{\theta}}^{\text{im}}_{t})\right\}
\end{align*}
where the first inequality is thanks to the non-expansiveness of $\Pi_r$, and the second inequality is due to the bound $\tilde \beta_t \le \beta$, which holds almost surely. In the last equality, we used $\bar{g}(\boldsymbol{\theta}^{\star}) = 0$. From Lemma \ref{lem:C1} and Lemma \ref{lem:C2}, we obtain
\begin{align*}
\mathbb{E}\left\|\widehat{\boldsymbol{\theta}}^{\text{im}}_{t+1} - \boldsymbol{\theta}^{\star}\right\|_2^2 &\leq \mathbb{E}\left\{(1 - 2\tilde \beta_t w_s)\left\|\widehat{\boldsymbol{\theta}}^{\text{im}}_{t} - \boldsymbol{\theta}^{\star}\right\|_2^2 \right\}+ \beta^2 G^2 + 2\beta\mathbb{E}\left\{\Lambda_t(\widehat{\boldsymbol{\theta}}^{\text{im}}_{t})\right\} + 2\mathbb{E}\left\{(\tilde\beta_t-\beta)\Lambda_t(\widehat{\boldsymbol{\theta}}^{\text{im}}_{t})\right\}\\
&\leq \left(1 - \frac{2\beta w_s}{1+\beta}\right)\mathbb{E}\left\|\widehat{\boldsymbol{\theta}}^{\text{im}}_{t} - \boldsymbol{\theta}^{\star}\right\|_2^2 + \beta^2 G^2 + 2\beta\mathbb{E}\left\{\Lambda_t(\widehat{\boldsymbol{\theta}}^{\text{im}}_{t})\right\} + 4\beta^2 G^2 \\
&= \left(1 - \frac{2\beta w_s}{1+\beta}\right)\mathbb{E}\left\|\widehat{\boldsymbol{\theta}}^{\text{im}}_{t} - \boldsymbol{\theta}^{\star}\right\|_2^2 + 5\beta^2 G^2 + 2\beta\mathbb{E}\left\{\Lambda_t(\widehat{\boldsymbol{\theta}}^{\text{im}}_{t})\right\}
\end{align*}
where in the second inequality, we used Lemma \ref{lem:C3} and the bound $\tilde\beta_t \ge \frac{\beta}{1+\beta}$, which holds almost surely.  Recursively applying the inequality with the assumption $\frac{2\beta w_s}{1+\beta} \in (0, 1)$, we have
\begin{align*}
\mathbb{E}^\pi\left\|\widehat{\boldsymbol{\theta}}^{\text{im}}_{t+1} - \boldsymbol{\theta}^{\star}\right\|_2^2 &\le \left(1 - \frac{2\beta w_s}{1+\beta}\right)^{t+1}\left\|\widehat{\boldsymbol{\theta}}^{\text{im}}_0- \boldsymbol{\theta}^{\star}\right\|_2^2 + 5\beta^2 G^2 \sum_{i=0}^\infty \left(1 - \frac{2\beta w_s}{1+\beta}\right)^i \\
&\quad + 2\beta \sum_{i=0}^{t} \left(1 - \frac{2\beta w_s}{1+\beta}\right)^{t-i} \mathbb{E}\left\{\Lambda_i(\widehat{\boldsymbol{\theta}}^{\text{im}}_{i})\right\}  \\
&\le \left(1 - \frac{2\beta w_s}{1+\beta}\right)^{t+1}\left\|\widehat{\boldsymbol{\theta}}^{\text{im}}_0- \boldsymbol{\theta}^{\star}\right\|_2^2 + 5\beta^2 G^2 \sum_{i=0}^\infty \left(1 - \frac{2\beta w_s}{1+\beta}\right)^i \\
&\quad + 2\beta \sum_{i=0}^{\tau_\beta} \left(1 - \frac{2\beta w_s}{1+\beta}\right)^{t-i} \left\{2G^2 + (6 + \lambda C)G^2\tau_\beta \beta\right\} \\
&\quad + 2\beta \sum_{i=\tau_\beta + 1}^{t} \left(1 - \frac{2\beta w_s}{1+\beta}\right)^{t-i} \left\{2C|\mathcal{A}|G^3\tau_\beta^2\beta + 4G^2\beta/\rho + (6 + \lambda C)G^2\tau_\beta \beta\right\} \\
&\le \left(1 - \frac{2\beta w_s}{1+\beta}\right)^{t+1}\left\|\widehat{\boldsymbol{\theta}}^{\text{im}}_0- \boldsymbol{\theta}^{\star}\right\|_2^2 + 5\beta^2 G^2 \sum_{i=0}^\infty \left(1 - \frac{2\beta w_s}{1+\beta}\right)^i \\
&\quad + 2\beta \left\{2C|\mathcal{A}|G^3\tau_\beta^2\beta + 4G^2\beta/\rho + 2G^2 + 2(6 + \lambda C)G^2\tau_\beta \beta\right\} \sum_{i=0}^{\infty} \left(1 - \frac{2\beta w_s}{1+\beta}\right)^{i}  \\
&= \left(1 - \frac{2\beta w_s}{1+\beta}\right)^{t+1}\left\|\widehat{\boldsymbol{\theta}}^{\text{im}}_0- \boldsymbol{\theta}^{\star}\right\|_2^2 + \mathcal{O}\left\{(1+\beta)(\beta + \beta \tau_\beta + \beta \tau_\beta^2)\right\}
\end{align*}
where in the second inequality, we have used Lemma \ref{lem:C6}.

\subsection{Proof of Theorem \ref{thm:decreasing_SARSA}}
\noindent \textbf{Standard SARSA:} we first obtain the finite-time error bound of the projected SARSA iterates with a step-size sequence of the form $\beta_i = \frac{\beta_0}{(i+1)^s}, ~s \in (0,1), ~i \in \mathbb{N}$. Consider
\begin{align*}
\mathbb{E}\left\|\widehat{\boldsymbol{\theta}}_{t+1} - \boldsymbol{\theta}^{\star}\right\|_2^2 &= \mathbb{E}\left\|\Pi_r\{\widehat{\boldsymbol{\theta}}_{t} + \beta_t g_t(\widehat{\boldsymbol{\theta}}_{t})\} - \Pi_r\boldsymbol{\theta}^{\star}\right\|_2^2 \\
&\leq \mathbb{E}\left\|\widehat{\boldsymbol{\theta}}_{t} + \beta_t g_t(\widehat{\boldsymbol{\theta}}_{t}) - \boldsymbol{\theta}^{\star}\right\|_2^2\\
&= \mathbb{E}\left\{\left\|\widehat{\boldsymbol{\theta}}_{t} - \boldsymbol{\theta}^{\star}\right\|_2^2 + \beta_t^2\left\|g_t(\widehat{\boldsymbol{\theta}}_{t})\right\|_2^2 + 2\beta_t\langle\widehat{\boldsymbol{\theta}}_{t} - \boldsymbol{\theta}^{\star}, g_t(\widehat{\boldsymbol{\theta}}_{t})\rangle\right\} \\
&= \mathbb{E}\left\{\left\|\widehat{\boldsymbol{\theta}}_{t} - \boldsymbol{\theta}^{\star}\right\|_2^2 + \beta_t^2\left\|g_t(\widehat{\boldsymbol{\theta}}_{t})\right\|_2^2 + 2\beta_t\langle\widehat{\boldsymbol{\theta}}_{t} - \boldsymbol{\theta}^{\star}, \bar{g}(\widehat{\boldsymbol{\theta}}_{t}) - \bar{g}(\boldsymbol{\theta}^{\star})\rangle + 2\beta_t\Lambda_t(\widehat{\boldsymbol{\theta}}_{t})\right\}
\end{align*}
where the inequality is thanks to the non-expansiveness of $\Pi_r$, and the last equality is due to $\bar{g}(\boldsymbol{\theta}^{\star}) = 0$. From Lemma \ref{lem:C1} and Lemma \ref{lem:C2}, we obtain
\begin{align*}
\mathbb{E}\left\|\widehat{\boldsymbol{\theta}}_{t+1} - \boldsymbol{\theta}^{\star}\right\|_2^2 &\leq (1 - 2\beta_t w_s)\mathbb{E}\left\|\widehat{\boldsymbol{\theta}}_{t} - \boldsymbol{\theta}^{\star}\right\|_2^2 + \beta_t^2 G^2 + 2\beta_t\mathbb{E}\Lambda_t(\widehat{\boldsymbol{\theta}}_{t}).
\end{align*}
Recursively applying the inequality with the assumption $\beta_0 \in (0, 1/2w_s)$, we have
\begin{align*}
\mathbb{E} \left\|\widehat{\boldsymbol{\theta}}_{t+1} - \boldsymbol{\theta}^{\star}\right\|_2^2 
&\leq \left\{\prod_{i=0}^{t} (1 - 2\beta_i w_s)\right\} \left\|\widehat{\boldsymbol{\theta}}_{0} - \boldsymbol{\theta}^{\star}\right\|_2^2 + G^2 \sum_{i=0}^{t} \left\{\prod_{k=i+1}^{t} (1 - 2\beta_k w_s)\right\} \beta_i^2 \\ &\quad + 2\sum_{i=0}^{t} \left\{\prod_{k=i+1}^{t} (1 - 2\beta_k w_s)\right\} \mathbb{E}\{\Lambda_i(\widehat{\boldsymbol{\theta}}_i)\} \beta_i \\
&\leq \left\{\prod_{i=0}^{t} (1 - 2\beta_iw_s) \right\}\left\|\widehat{\boldsymbol{\theta}}_{0} - \boldsymbol{\theta}^{\star}\right\|_2^2 + G^2 \sum_{i=0}^{t} \left\{\prod_{k=i+1}^{t} (1 - 2\beta_k w_s)\right\} \beta_i^2 \\
&\quad + \sum_{i=0}^{\tau_{\beta_t}} \left\{\prod_{k=i+1}^{t} (1 - 2\beta_k w_s)\right\}  \left\{4G^2 + (12 + 2\lambda C)G^2\sum_{k=0}^{\tau_{\beta_t}-1} \beta_k\right\}\beta_i \\
&\quad + \sum_{i=\tau_{\beta_t}+1}^{t} \left\{\prod_{k=i+1}^{t} (1 - 2\beta_k w_s)\right\}  \left[\left\{4C|\mathcal{A}|G^3\tau_{\beta_t} + (12 + 2\lambda C)G^2 \right\}\sum_{k=i-\tau_{\beta_t}}^{i-1} \beta_k + 8G^2\beta_t/\rho \right] \beta_i \\
&\leq \left\{\prod_{i=0}^{t} (1 - 2\beta_iw_s) \right\}\left\|\widehat{\boldsymbol{\theta}}_{0} - \boldsymbol{\theta}^{\star}\right\|_2^2 + G^2 \sum_{i=0}^{t} \left\{\prod_{k=i+1}^{t} (1 - 2\beta_k w_s)\right\} \beta_i^2 \\
&\quad + \left\{4G^2 + (12 + 2\lambda C)G^2\tau_{\beta_t} \beta_0\right\}\sum_{i=0}^{\tau_{\beta_t}} \left\{\prod_{k=i+1}^{t} (1 - 2\beta_k w_s)\right\} \beta_i \\
&\quad + \left\{4C|\mathcal{A}|G^3\tau_{\beta_t}^2\beta_{i-\tau_{\beta_t}}   + (12 + 2\lambda C)G^2\tau_{\beta_t}\beta_{i-\tau_{\beta_t}} + 8G^2\beta_t/\rho \right\}\sum_{i=\tau_{\beta_t}+1}^{t} \left\{\prod_{k=i+1}^{t} (1 - 2\beta_k w_s)\right\} \beta_i
\end{align*}
where the second inequality follows from Lemma \ref{lem:C5}. Thanks to $1-2\beta_i w_s \le \exp(-2\beta_i w_s)$, we have
\begin{align}
\mathbb{E} \left\|\widehat{\boldsymbol{\theta}}_{t+1} - \boldsymbol{\theta}^{\star}\right\|_2^2 &\leq \exp\left(-2w_s \sum_{i=0}^{t} \beta_i\right) \left\|\widehat{\boldsymbol{\theta}}_{0} - \boldsymbol{\theta}^{\star}\right\|_2^2 \quad \label{SARSA_decr_first} \\
&\quad + G^2 \sum_{i=0}^{t} \exp\left(-2w_s \sum_{k=i+1}^{t} \beta_k\right) \beta_i^2 \quad  \label{SARSA_decr_second}\\
&\quad + \left\{4G^2 + (12 + 2\lambda C)G^2\tau_{\beta_t} \beta_0\right\}  \sum_{i=0}^{\tau_{\beta_t}} \exp\left(-2w_s \sum_{k=i+1}^{t} \beta_k\right) \beta_i \quad  \label{SARSA_decr_third} \\
&\quad + \left\{4C|\mathcal{A}|G^3\tau_{\beta_t}^2\beta_{i-\tau_{\beta_t}} + (12 + 2\lambda C)G^2\tau_{\beta_t}\beta_{i-\tau_{\beta_t}} + 8G^2\beta_t/\rho \right\} \sum_{i=\tau_{\beta_t}+1}^{t} \exp\left(-2w_s \sum_{k=i+1}^{t} \beta_k\right) \beta_i.  \label{SARSA_decr_fourth}
\end{align}
We now provide a bound for each of the term above. For the first term, we have
\begin{align}
\eqref{SARSA_decr_first} &= \exp\left\{-2w_s\beta_0 \sum_{i=0}^{t} \frac{1}{(i+1)^s}\right\} \left\|\widehat{\boldsymbol{\theta}}_{0} - \boldsymbol{\theta}^{\star}\right\|_2^2 \nonumber\\
&\leq \exp\left[\frac{-2w_s\beta_0}{(1-s)} \left\{(t+1)^{1-s} - 1\right\}\right] \left\|\widehat{\boldsymbol{\theta}}_{0} - \boldsymbol{\theta}^{\star}\right\|_2^2. \label{SARSA_const_first_bound}
\end{align}
For the second term, we have
\begin{align}
\eqref{SARSA_decr_second}  &\le 2G^2 (K_b e^{-w_s \sum_{k=0}^{t} \beta_k} + \beta_t) \frac{\exp\left(w_s\beta_0\right)}{2w_s} && \text{(by Lemma \ref{lem:exp_bound1})} \nonumber\\
&\lesssim \exp\left[\frac{-w_s\beta_0}{(1-s)} \{(t+1)^{1-s} - 1\}\right] + \beta_t.\label{SARSA_const_second_bound}
\end{align}
For the third term, by Lemma \ref{lem:exp_bound2}, we get
\begin{align}
\eqref{SARSA_decr_third} 
&\leq \left\{4G^2 + (12 + 2\lambda C)G^2\tau_{\beta_t} \beta_0\right\} \frac{\exp(2w_s\beta_0)}{2w_s} \exp\left[\frac{-2w_s\beta_0}{(1-s)} \{(t+1)^{1-s} - (\tau_{\beta_t}+1)^{1-s}\}\right] \nonumber\\
&\lesssim \tau_{\beta_t}  \exp\left[\frac{-2w_s\beta_0}{(1-s)} \{(t+1)^{1-s} - (\tau_{\beta_t}+1)^{1-s}\}\right].\label{SARSA_const_third_bound}
\end{align}
For the last term, by Lemma \ref{lem:exp_bound2}, we yield
\begin{align}
\eqref{SARSA_decr_fourth}
&\leq \left\{4C|\mathcal{A}|G^3\tau_{\beta_t}^2  + (12 + 2\lambda C)G^2\tau_{\beta_t} \right\} \left[\exp\left\{\frac{-2w_s\beta_0}{(1-s)} \{(t+1)^{1-s} - 1\}\right\} D_{\beta_t} \mathbb{I}_{\{\tau_{\beta_t}+1 \leq 1\}} + \beta_{t-\tau_{\beta_t}}\right] \frac{\exp(w_s\beta_0)}{w_s} \nonumber\\
&\quad + 8G^2\beta_t/\rho  \left\{\sum_{i=\tau_{\beta_t}+1}^{t} \exp\left(-2w_s \sum_{k=i+1}^{t} \beta_k\right) \beta_i\right\}\nonumber\\
&\lesssim  \left(\tau_{\beta_t}^2 + \tau_{\beta_t} \right) \beta_{t-\tau_{\beta_t}}.\label{SARSA_const_fourth_bound}
\end{align}
Combining \eqref{SARSA_const_first_bound}-\eqref{SARSA_const_fourth_bound} we get,
\begin{align*}
\mathbb{E} \left\|\widehat{\boldsymbol{\theta}}_{t+1} - \boldsymbol{\theta}^{\star}\right\|_2^2 &\leq \exp\left[\frac{-2w_s\beta_0}{(1-s)} \{(t+1)^{1-s} - 1\}\right] \left\|\widehat{\boldsymbol{\theta}}_{0} - \boldsymbol{\theta}^{\star}\right\|_2^2 + \mathcal{O}\left\{\beta_{t-\tau_{\beta_t}}\left(\tau_{\beta_t}^2 + \tau_{\beta_t} \right) \right\} 
\end{align*}
completing the proof for the standard SARSA finite-time bound.\\

\noindent \textbf{Implicit SARSA:} we next obtain the finite-time error bound of the projected implicit SARSA iterates with a step-size sequence of the form $\beta_i = \frac{\beta_0}{(i+1)^s}, ~s \in (0,1), ~i \in \mathbb{N}$. With a slight abuse of notation, let $\tilde \beta_i := \frac{\beta_i}{1+\beta_i \|\boldsymbol{\phi}(\boldsymbol{S}_i^{\pi},A_i^{\pi})\|^2_2}$. Consider

\begin{align*}
\mathbb{E}\left\|\widehat{\boldsymbol{\theta}}^{\text{im}}_{t+1} - \boldsymbol{\theta}^{\star}\right\|_2^2 &= \mathbb{E}\left\|\Pi_r\{\widehat{\boldsymbol{\theta}}^{\text{im}}_{t} + \tilde\beta_t g_t(\widehat{\boldsymbol{\theta}}^{\text{im}}_{t})\} - \Pi_r\boldsymbol{\theta}^{\star}\right\|_2^2 \\
&\leq \mathbb{E}\left\|\widehat{\boldsymbol{\theta}}^{\text{im}}_{t} + \tilde\beta_t g_t(\widehat{\boldsymbol{\theta}}^{\text{im}}_{t}) - \boldsymbol{\theta}^{\star}\right\|_2^2\\
&\le \mathbb{E}\left\{\left\|\widehat{\boldsymbol{\theta}}^{\text{im}}_{t} - \boldsymbol{\theta}^{\star}\right\|_2^2 + \beta_t^2\left\|g_t(\widehat{\boldsymbol{\theta}}^{\text{im}}_{t})\right\|_2^2 + 2\tilde\beta_t\langle\widehat{\boldsymbol{\theta}}^{\text{im}}_{t} - \boldsymbol{\theta}^{\star}, g_t(\widehat{\boldsymbol{\theta}}^{\text{im}}_{t})\rangle\right\} \\
&= \mathbb{E}\left\{\left\|\widehat{\boldsymbol{\theta}}^{\text{im}}_{t} - \boldsymbol{\theta}^{\star}\right\|_2^2 + \beta_t^2\left\|g_t(\widehat{\boldsymbol{\theta}}^{\text{im}}_{t})\right\|_2^2 + 2\tilde\beta_t\langle\widehat{\boldsymbol{\theta}}^{\text{im}}_{t} - \boldsymbol{\theta}^{\star}, \bar{g}(\widehat{\boldsymbol{\theta}}^{\text{im}}_{t}) - \bar{g}(\boldsymbol{\theta}^{\star})\rangle + 2\tilde\beta_t\Lambda_t(\widehat{\boldsymbol{\theta}}^{\text{im}}_{t})\right\}
\end{align*}
where the first inequality is thanks to the non-expansiveness of $\Pi_r$, and the second inequality is due to $\tilde \beta_t \le \beta_t$. In the last inequality, we used $\bar{g}(\boldsymbol{\theta}^{\star}) = 0$. From Lemma \ref{lem:C1}  and Lemma \ref{lem:C2}, we obtain
\begin{align*}
\mathbb{E}\left\|\widehat{\boldsymbol{\theta}}^{\text{im}}_{t+1} - \boldsymbol{\theta}^{\star}\right\|_2^2 &\leq \mathbb{E}\left\{(1 - 2\tilde \beta_t w_s)\left\|\widehat{\boldsymbol{\theta}}^{\text{im}}_{t} - \boldsymbol{\theta}^{\star}\right\|_2^2 \right\}+ \beta_t^2 G^2 + 2\beta_t\mathbb{E}\Lambda_t(\widehat{\boldsymbol{\theta}}^{\text{im}}_{t}) + 2\mathbb{E}\left\{(\tilde\beta_t-\beta_t)\Lambda_t(\widehat{\boldsymbol{\theta}}^{\text{im}}_{t})\right\}\\
&\leq \left(1 - \frac{2\beta_t w_s}{1+\beta_0}\right)\mathbb{E}\left\|\widehat{\boldsymbol{\theta}}^{\text{im}}_{t} - \boldsymbol{\theta}^{\star}\right\|_2^2 + 5\beta_t^2 G^2 + 2\beta_t\mathbb{E}\Lambda_t(\widehat{\boldsymbol{\theta}}^{\text{im}}_{t}).
\end{align*}
where in the second inequality, we used the bound $\tilde\beta_t \ge \frac{\beta_t}{1+\beta_0}$, which holds almost surely as well as the fact $|\tilde\beta_t-\beta_t| \le \beta_t^2$ with Lemma \ref{lem:C3}. 
Recursively applying the inequality with the assumption $2\beta_0 w_s/(1+\beta_0)\in (0, 1)$, we have
\begin{align*}
\mathbb{E} \left\|\widehat{\boldsymbol{\theta}}^{\text{im}}_{t+1} - \boldsymbol{\theta}^{\star}\right\|_2^2 
&\leq \left\{\prod_{i=0}^{t} \left(1 - \frac{2\beta_i w_s}{1+\beta_0}\right)\right\} \left\|\widehat{\boldsymbol{\theta}}_{0}- \boldsymbol{\theta}^{\star}\right\|_2^2 + 5G^2 \sum_{i=0}^{t} \left\{\prod_{k=i+1}^{t} \left(1 - \frac{2\beta_k w_s}{1+\beta_0}\right)\right\} \beta_i^2 \\ &\quad + 2\sum_{i=0}^{t} \left\{\prod_{k=i+1}^{t} \left(1 - \frac{2\beta_k w_s}{1+\beta_0}\right)\right\} \mathbb{E}\{\Lambda_i(\widehat{\boldsymbol{\theta}}^{\text{im}}_i)\} \beta_i 
\end{align*}
Applying Lemma \ref{lem:C6}, we have
\begin{align*}
\mathbb{E} \left\|\widehat{\boldsymbol{\theta}}^{\text{im}}_{t+1} - \boldsymbol{\theta}^{\star}\right\|_2^2 &\leq \left\{\prod_{i=0}^{t} \left(1 - \frac{2\beta_i w_s}{1+\beta_0}\right) \right\}\left\|\widehat{\boldsymbol{\theta}}_{0}- \boldsymbol{\theta}^{\star}\right\|_2^2 + 5G^2 \sum_{i=0}^{t} \left\{\prod_{k=i+1}^{t} \left(1 - \frac{2\beta_k w_s}{1+\beta_0}\right)\right\} \beta_i^2 \\
&\quad + \sum_{i=0}^{\tau_{\beta_t}} \left\{\prod_{k=i+1}^{t} \left(1 - \frac{2\beta_k w_s}{1+\beta_0}\right)\right\}  \left\{4G^2 + (12 + 2\lambda C)G^2\sum_{k=0}^{\tau_{\beta_t}-1} \beta_k\right\}\beta_i \\
&\quad + \sum_{i=\tau_{\beta_t}+1}^{t} \left\{\prod_{k=i+1}^{t} \left(1 - \frac{2\beta_k w_s}{1+\beta_0}\right)\right\}  \left[\left\{4C|\mathcal{A}|G^3\tau_{\beta_t} + (12 + 2\lambda C)G^2 \right\}\sum_{k=i-\tau_{\beta_t}}^{i-1} \beta_k + 8G^2\beta_t/\rho \right] \beta_i \\
&\leq \left\{\prod_{i=0}^{t} \left(1 - \frac{2\beta_i w_s}{1+\beta_0}\right) \right\}\left\|\widehat{\boldsymbol{\theta}}_{0}- \boldsymbol{\theta}^{\star}\right\|_2^2 + 5G^2 \sum_{i=0}^{t} \left\{\prod_{k=i+1}^{t} \left(1 - \frac{2\beta_k w_s}{1+\beta_0}\right)\right\} \beta_i^2 \\
&\quad + \left\{4G^2 + (12 + 2\lambda C)G^2\tau_{\beta_t} \beta_0\right\}\sum_{i=0}^{\tau_{\beta_t}} \left\{\prod_{k=i+1}^{t} \left(1 - \frac{2\beta_k w_s}{1+\beta_0}\right)\right\} \beta_i \\
&\quad + \left\{4C|\mathcal{A}|G^3\tau_{\beta_t}^2\beta_{i-\tau_{\beta_t}}   + (12 + 2\lambda C)G^2\tau_{\beta_t}\beta_{i-\tau_{\beta_t}} + 8G^2\beta_t/\rho \right\}\sum_{i=\tau_{\beta_t}+1}^{t} \left\{\prod_{k=i+1}^{t} \left(1 - \frac{2\beta_k w_s}{1+\beta_0}\right)\right\} \beta_i
\end{align*}
Thanks to $1-\frac{2\beta_i w_s}{1+\beta_0} \le \exp\left(-\frac{2\beta_i w_s}{1+\beta_0}\right)$, we have
\begin{align}
\mathbb{E} \left\|\widehat{\boldsymbol{\theta}}^{\text{im}}_{t+1} - \boldsymbol{\theta}^{\star}\right\|_2^2 &\leq \exp\left(-\frac{2w_s}{1+\beta_0} \sum_{i=0}^{t} \beta_i\right) \left\|\widehat{\boldsymbol{\theta}}_{0}- \boldsymbol{\theta}^{\star}\right\|_2^2 \quad \label{IMP_SARSA_decr_first} \\
&\quad + 5G^2 \sum_{i=0}^{t} \exp\left(-\frac{2w_s}{1+\beta_0} \sum_{k=i+1}^{t} \beta_k\right) \beta_i^2 \quad  \label{IMP_SARSA_decr_second}\\
&\quad + \left\{4G^2 + (12 + 2\lambda C)G^2\tau_{\beta_t} \beta_0\right\}  \sum_{i=0}^{\tau_{\beta_t}} \exp\left(-\frac{2w_s}{1+\beta_0} \sum_{k=i+1}^{t} \beta_k\right) \beta_i \quad  \label{IMP_SARSA_decr_third} \\
&\quad + \left\{4C|\mathcal{A}|G^3\tau_{\beta_t}^2\beta_{i-\tau_{\beta_t}} + (12 + 2\lambda C)G^2\tau_{\beta_t}\beta_{i-\tau_{\beta_t}} + 8G^2\beta_t/\rho \right\} \sum_{i=\tau_{\beta_t}+1}^{t} \exp\left(-\frac{2w_s}{1+\beta_0} \sum_{k=i+1}^{t} \beta_k\right) \beta_i.  \label{IMP_SARSA_decr_fourth}
\end{align}
We now provide a bound for each of the term above. For the first term, we have
\begin{align}
\eqref{IMP_SARSA_decr_first} &= \exp\left\{-\frac{2w_s}{1+\beta_0}\beta_0 \sum_{i=0}^{t} \frac{1}{(i+1)^s}\right\} \left\|\widehat{\boldsymbol{\theta}}_{0}- \boldsymbol{\theta}^{\star}\right\|_2^2 \nonumber\\
&\leq \exp\left[\frac{-2w_s\beta_0}{(1+\beta_0)(1-s)} \left\{(t+1)^{1-s} - 1\right\}\right] \left\|\widehat{\boldsymbol{\theta}}_{0}- \boldsymbol{\theta}^{\star}\right\|_2^2. \label{IMP_SARSA_const_first_bound}
\end{align}
For the second term, we have
\begin{align}
\eqref{IMP_SARSA_decr_second}  &\le 10G^2 (K_b e^{-\frac{w_s}{1+\beta_0} \sum_{k=0}^{t} \beta_k} + \beta_t) \frac{\exp\left(w_s\beta_0\right)(1+\beta_0)}{w_s} && \text{(by Lemma \ref{lem:exp_bound1})} \nonumber\\
&\lesssim \exp\left[\frac{-w_s\beta_0}{(1+\beta_0)(1-s)} \{(t+1)^{1-s} - 1\}\right] + \beta_t.\label{IMP_SARSA_const_second_bound}
\end{align}
For the third term, by Lemma \ref{lem:exp_bound2}, we get
\begin{align}
\eqref{IMP_SARSA_decr_third} 
&\lesssim \left\{4G^2 + (12 + 2\lambda C)G^2\tau_{\beta_t} \beta_0\right\} \exp\left[\frac{-2w_s\beta_0}{(1+\beta_0)(1-s)} \{(t+1)^{1-s} - (\tau_{\beta_t}+1)^{1-s}\}\right] \nonumber\\
&\lesssim \tau_{\beta_t}  \exp\left[\frac{-2w_s\beta_0}{(1+\beta_0)(1-s)} \{(t+1)^{1-s} - (\tau_{\beta_t}+1)^{1-s}\}\right].\label{IMP_SARSA_const_third_bound}
\end{align}
For the last term, by Lemma \ref{lem:exp_bound2}, we yield
\begin{align}
\eqref{IMP_SARSA_decr_fourth}
&\lesssim \left\{4C|\mathcal{A}|G^3\tau_{\beta_t}^2  + (12 + 2\lambda C)G^2\tau_{\beta_t} \right\} \left[\exp\left\{\frac{-2w_s\beta_0}{(1+\beta_0)(1-s)} \{(t+1)^{1-s} - 1\}\right\} D_{\beta_t} \mathbb{I}_{\{\tau_{\beta_t}+1 \leq 1\}} + \beta_{t-\tau_{\beta_t}}\right] \nonumber\\
&\quad + 8G^2\beta_t/\rho  \left\{\sum_{i=\tau_{\beta_t}+1}^{t} \exp\left(-\frac{2w_s}{(1+\beta_0)} \sum_{k=i+1}^{t} \beta_k\right) \beta_i\right\}\nonumber\\
&\lesssim  \left(\tau_{\beta_t}^2 + \tau_{\beta_t} \right) \beta_{t-\tau_{\beta_t}}.\label{IMP_SARSA_const_fourth_bound}
\end{align}
Combining \eqref{IMP_SARSA_const_first_bound}-\eqref{IMP_SARSA_const_fourth_bound} we get,
\begin{align*}
\mathbb{E} \left\|\widehat{\boldsymbol{\theta}}^{\text{im}}_{t+1} - \boldsymbol{\theta}^{\star}\right\|_2^2 &\leq \exp\left[\frac{-2w_s\beta_0}{(1+\beta_0)(1-s)} \{(t+1)^{1-s} - 1\}\right] \left\|\widehat{\boldsymbol{\theta}}_{0}- \boldsymbol{\theta}^{\star}\right\|_2^2 + \mathcal{O}\left\{\beta_{t-\tau_{\beta_t}}\left(\tau_{\beta_t}^2 + \tau_{\beta_t} \right) \right\} 
\end{align*}
completing the proof.

\section{Auxiliary Lemma}
\begin{lemma}\label{lem:exp_bound1} For $t \in \mathbb{N}$, let $\beta_t = \frac{\beta_0}{(t+1)^s}$ and $s \in (0,1)$. With $\gamma > 0$,
\begin{align*}
    \sum_{i=0}^t \left(e^{-\gamma\sum_{k=i+1}^t \beta_k} \right)\beta^2_i \le 2\left(K_b e^{-\frac{\gamma}{2}\sum_{k=0}^t \beta_k} + \beta_t\right)\frac{e^{\frac{\gamma \beta_0}{2}} }{\gamma}, 
\end{align*}
where $K_b = \beta_0 e^{\frac{\gamma}{2}\sum_{k=0}^{i_0} \beta_k}$ for some $i_0 \in \mathbb{N}$. 
\end{lemma}
\begin{proof}
Let $T_t = \sum_{i=0}^{t-1} \beta_i$ and use the convention $\sum_{k=t+1}^t \beta_k = 0$ and $\sum_{k=t+1}^t \beta^2_k = 0$. Notice that
\begin{align}
\sum_{i=0}^t \left(e^{-\frac{\gamma}{2}\sum_{k=i+1}^t \beta_k} \right)\beta_i &\le \left(\sup_{i \ge 0} e^{\frac{\gamma}{2}\beta_i}\right) \left\{\sum_{i=0}^t \left(e^{-\frac{\gamma}{2}\sum_{k=i}^t \beta_k} \right)\beta_i \right\} 
\nonumber \\ &= \left(\sup_{i \ge 0} e^{\frac{\gamma}{2}\beta_i}\right) \left\{\sum_{i=0}^t \left(e^{-\frac{\gamma}{2}(T_{t+1}-T_i)} \right)\beta_i \right\} \nonumber \\
&\le \left(\sup_{i \ge 0} e^{\frac{\gamma}{2}\beta_i}\right)  \int_{0}^{T_{t+1}} e^{-\frac{\gamma}{2}(T_{t+1}-s)} ds \nonumber \\
&\le \left(\sup_{i \ge 0} e^{\frac{\gamma}{2}\beta_i}\right) \frac{2}{\gamma} \le \frac{2e^{\frac{\gamma \beta_0}{2}}}{\gamma}, \label{exp_order1_bound}
\end{align}
where we have used the definition of the left-Riemann sum in the first inequality. The last inequality is due to the fact that $\{\beta_t\}$ is a non-increasing sequence. Now consider 
\begin{align}
\sum_{i=0}^t \left(e^{-\gamma\sum_{k=i+1}^t \beta_k} \right)\beta^2_i &\le \sup_{0 \le i \le t} \left(\beta_i e^{-\frac{\gamma}{2}\sum_{k=i+1}^t \beta_k}\right) \left\{\sum_{i=0}^t \left(e^{-\frac{\gamma}{2}\sum_{k=i+1}^t \beta_k} \right)\beta_i \right\} \nonumber \\
&\le \sup_{0 \le i \le t} \left(\beta_i e^{-\frac{\gamma}{2}\sum_{k=i+1}^t \beta_k}\right)\frac{2e^{\frac{\gamma \beta_0}{2}}}{\gamma}\label{intermediate_exp_square_bdd}
\end{align}
where the last inequality follows from \eqref{exp_order1_bound}. Note that $\beta_i e^{-\frac{\gamma}{2}\sum_{k=i+1}^t \beta_k}$ is eventually increasing, i.e., after some time $i_0 \in \mathbb{N}$,  for all $t \ge i_0$, we have
$$
\sup_{i_0 \le i \le t} \left\{\beta_i \exp\left(-\frac{\gamma}{2}\sum_{k=i+1}^t \beta_k \right)\right\} \le \beta_t,
$$
where we used the convention $\sum_{k=t+1}^t \beta_k = 0$. Therefore, we have
\begin{align*}
    \eqref{intermediate_exp_square_bdd} &\le \left\{\sup_{0 \le i \le i_0} \left(\beta_i e^{-\frac{\gamma}{2}\sum_{k=i+1}^t \beta_k}\right) + \beta_t\right\}\frac{2e^{\frac{\gamma \beta_0}{2}} }{\gamma} \\
    &\le \left\{e^{-\frac{\gamma}{2}\sum_{k=0}^t \beta_k}\sup_{0 \le i \le i_0} \left(\beta_i e^{\frac{\gamma}{2}\sum_{k=0}^i \beta_k}\right) + \beta_t\right\}\frac{2e^{\frac{\gamma \beta_0}{2}} }{\gamma} \\
    &\le \left(K_b e^{-\frac{\gamma}{2}\sum_{k=0}^t \beta_k} + \beta_t\right)\frac{2e^{\frac{\gamma \beta_0}{2}} }{\gamma},
\end{align*}
where $K_b = \beta_0 e^{\frac{\gamma}{2}\sum_{k=0}^{i_0} \beta_k}$.
\end{proof}

\begin{lemma}\label{lem:exp_bound2} For $t \in \mathbb{N}$, let $\beta_t = \frac{\beta_0}{(t+1)^s}$ and $s\in (0,1)$. With $\gamma > 0$ and $\tau \in \{0, \cdots, t\}$,
\begin{enumerate}
    \item $\sum_{i=0}^{\tau} e^{-\gamma \sum_{k=i+1}^{t} \beta_{k}} \beta_{i}  \le \frac{e^{\gamma \beta_0}}{\gamma} e^{-\frac{\gamma \beta_0}{(1-s)}\left\lbrace(t+1)^{1-s}-\left(1+\tau\right)^{1-s}\right\rbrace}$. 
    \item $\sum_{i=2\tau_{\beta_t}+1}^{t} \left(e^{-\gamma \sum_{k=i+1}^{t} \beta_{k}} \beta_{i-2\tau_{\beta_t}} \beta_{i} \right)  \le \left[e^{\frac{-\gamma \beta_0}{2(1-s)}\left\lbrace(t+1)^{1-s}-1\right\rbrace} D_{\beta} \mathbb{I}_{\left(2\tau_{\beta_t}+1 < i_{\beta} \right)}+\beta_{t-2\tau_{\beta_t}} \right] \frac{2 e^{\gamma \beta_0 / 2}}{\gamma}$
\end{enumerate}
where 
$D_\beta = \exp\lbrace\left(\gamma / 2\right) \sum_{k=0}^{i_{\beta}} \beta_{k}\rbrace \beta_0$ for some $i_{\beta} \in \mathbb{N}$.
\end{lemma}
\begin{proof}
Let $T_t = \sum_{i=0}^{t-1} \beta_i$ and use the convention $\sum_{k=t+1}^t \beta_k = 0$. For the first statement, 
\begin{align*}
\sum_{i=0}^{\tau} e^{-\gamma \sum_{k=i+1}^{t} \beta_{k}} \beta_{i} & \leq \max _{i \geq 0}\left(e^{\gamma \beta_{i}}\right) \sum_{i=0}^{\tau} e^{-\gamma \sum_{k=i}^{t} \beta_{k}} \beta_{i} = e^{\gamma \beta_0} \sum_{i=0}^{\tau} e^{-\gamma\left(T_{t+1}-T_{i}\right)} \beta_{i} \\
& \leq e^{\gamma \beta_0} \int_{0}^{T_{\tau+1}} e^{-\gamma\left(T_{t+1}-s\right)} d s  \leq \frac{e^{\gamma \beta_0}}{\gamma} e^{-\gamma\left(T_{t+1}-T_{\tau+1}\right)} \\&= \frac{e^{\gamma \beta_0}}{\gamma} e^{-\gamma \beta_0 \sum_{k=\tau+1}^{t} 1 /(1+k)^{s}} \le \frac{e^{\gamma \beta_0}}{\gamma} e^{-\frac{\gamma \beta_0}{(1-s)}\left\lbrace(t+1)^{1-s}-\left(1+\tau\right)^{1-s}\right\rbrace}.
\end{align*}
For the second statement, first notice that 
\begin{align}
\sum_{i=2\tau_{\beta_t}+1}^{t} e^{-\gamma \sum_{k=i+1}^{t} \beta_{k}} \beta_{i} & \leq \max _{i \geq 0}\left(e^{\gamma \beta_{i}}\right)  \sum_{i=2\tau_{\beta_t}+1}^{t} e^{-\gamma \sum_{k=i}^{t} \beta_{k}} \beta_{i} = e^{\gamma \beta_0}  \sum_{i=2\tau_{\beta_t}+1}^{t} e^{-\gamma\left(T_{t+1}-T_{i}\right)} \beta_{i} \nonumber \\
& \leq e^{\gamma \beta_0}\int_{T_{2\tau_{\beta_t}+1}}^{T_{t+1}} e^{-\gamma\left(T_{t+1}-s\right)} ds =\frac{e^{\gamma \beta_0}}{\gamma} \left\lbrace1-e^{-\gamma \left(T_{t+1}-T_{2\tau_{\beta_t}+1}\right)}\right\rbrace \leq \frac{e^{\gamma \beta_0}}{\gamma}. \label{prelim_second_pre}
\end{align}
Then, we have
\begin{align}
\sum_{i=2\tau_{\beta_t}+1}^{t} e^{-\gamma \sum_{k=i+1}^{t} \beta_{k}} \beta_{i-2\tau_{\beta_t}} \beta_{i} & \leq \max _{i \in\left[2\tau_{\beta_t}+1, t\right]}\left\{e^{\left(-\gamma / 2\right) \sum_{k=i+1}^{t} \beta_{k}} \beta_{i-2\tau_{\beta_t}}\right\} \sum_{i=2\tau_{\beta_t}+1}^{t} e^{\left(-\gamma / 2\right) \sum_{k=i+1}^{t} \beta_{k}} \beta_{i} \nonumber \\
& \leq \max _{i \in\left[2\tau_{\beta_t}+1, t\right]}\left\{e^{\left(-\gamma / 2\right) \sum_{k=i+1}^{t} \beta_{k}} \beta_{i-2\tau_{\beta_t}}\right\} \frac{2 e^{\gamma \beta_0 / 2}}{\gamma} \label{part_c_exp_bound}
\end{align}
where the second inequality follows from \eqref{prelim_second_pre}. To bound the first term in \eqref{part_c_exp_bound}, note that the sequence $\left\{e^{\left(-\gamma / 2\right) \sum_{k=i+1}^{t} \beta_{k}} \beta_{i-2\tau_{\beta_t}}\right\}_{i\in \mathbb{N}}$ is eventually increasing. In other words, there exists $i_{\beta} \in \mathbb{N}$ such that, 
\begin{align*}
\max _{i \in\left[2\tau_{\beta_t}+1, t\right]}\left\{e^{\left(-\gamma / 2\right) \sum_{k=i+1}^{t} \beta_{k}} \beta_{i-2\tau_{\beta_t}}\right\}&=\beta_{t-2\tau_{\beta_t}} \quad \text{if} \quad 2\tau_{\beta_t}+1\ge i_{\beta}.
\end{align*}
If $2\tau_{\beta_t}+1<i_{\beta}$, then
\begin{align*}
&\max _{i \in\left[2\tau_{\beta_t}+1, t\right]}\left\{e^{\left(-\gamma / 2\right) \sum_{k=i+1}^{t} \beta_{k}} \beta_{i-2\tau_{\beta_t}}\right\}\\ &\leq \max _{i \in\left[2\tau_{\beta}+1, i_{\beta}\right]}\left\{e^{\left(-\gamma / 2\right) \sum_{k=i+1}^{t} \beta_{k}} \beta_{i-2\tau_{\beta_t}}\right\}+\max _{i \in\left[i_{\beta}+1, t\right]}\left\{e^{\left(-\gamma / 2\right) \sum_{k=i+1}^{t} \beta_{k}} \beta_{i-2\tau_{\beta_t}}\right\} \\
& \leq e^{\left(-\gamma / 2\right) \sum_{k=0}^{t} \beta_{k}} \max _{i \in\left[2\tau_{\beta_t}+1, i_{\beta}\right]}\left\{e^{\left(\gamma / 2\right) \sum_{k=0}^{i} \beta_{k}} \beta_{i-2\tau_{\beta_t}}\right\}+\beta_{t-2\tau_{\beta_t}} \\
& \leq e^{-\left(\gamma / 2\right) \sum_{k=0}^{t} \beta_{k}} e^{\left(\gamma / 2\right) \sum_{k=0}^{i_{\beta}} \beta_{k}} \beta_{0}+\beta_{t-2\tau_{\beta_t}} \\
& \leq e^{\frac{-\gamma \beta_0}{2(1-s)}\left\lbrace(t+1)^{1-s}-1\right\rbrace} D_{\beta}+\beta_{t-2\tau_{\beta_t}}
\end{align*}
where $D_\beta = e^{\left(\gamma / 2\right) \sum_{k=0}^{i_{\beta}} \beta_{k}} \beta_0$. Combining everything, we get
\begin{align*}
\sum_{i=2\tau_{\beta_t}+1}^{t} e^{-\gamma \sum_{k=i+1}^{t} \beta_{k}} \beta_{i-2\tau_{\beta_t}} \beta_{i} \leq \left[e^{\frac{-\gamma \beta_0}{2(1-s)}\left\lbrace(t+1)^{1-s}-1\right\rbrace} D_{\beta} \mathbb{I}_{\left(2\tau_{\beta_t}+1 < i_{\beta} \right)}+\beta_{t-2\tau_{\beta_t}} \right] \frac{2 e^{\gamma \beta_0 / 2}}{\gamma}. 
\end{align*}
\end{proof}

\end{document}